\documentclass{article}
\usepackage{comment}
\usepackage{fullpage}
\usepackage{subcaption}
\usepackage{graphicx}
\usepackage{enumitem}
\usepackage{xcolor}
\usepackage[sort&compress,numbers]{natbib}
\usepackage[colorlinks=true, linkcolor=blue, citecolor=blue, urlcolor=blue]{hyperref}
\usepackage{xfrac}
\usepackage{rotating}

\usepackage{amsthm}
\usepackage{amsmath}
\usepackage{amsfonts}
\usepackage{amssymb}
\usepackage{euscript}
\usepackage{bm}
\usepackage{pdflscape} 
\usepackage{framed}
\usepackage{graphicx}
\usepackage{tikz}
\usetikzlibrary{arrows.meta, positioning}
\usepackage{mathrsfs}  
\usepackage[affil-it]{authblk}
\usepackage{dsfont}

\makeatletter
\def\namedlabel#1#2{\begingroup
  #2%
  \def\@currentlabel{#2}%
  \phantomsection\label{#1}\endgroup
}
\makeatother

\renewcommand{\S}{\mathbf{S}}
\newcommand{\C}{\mathbf{C}}

\newcommand{\bA}{\mathbf{A}}

\newcommand{\bB}{\mathbf{B}}

\newcommand{\bU}{\mathbf{U}}

\newcommand{\bSigma}{\mathbf{\Sigma}}

\newcommand{\bx}{\mathbf{x}}
\newcommand{\bu}{\mathbf{u}}
\newcommand{\bv}{\mathbf{v}}
\renewcommand{\bf}{\mathbf{f}}
\renewcommand{\bm}{\mathbf{m}}

\newcommand{\bbP}{\mathbb{P}}
\newcommand{\bbQ}{\mathbb{Q}}
\newcommand{\bbN}{\mathbb{N}}
\newcommand{\bbR}{\mathbb{R}}

\newcommand{\bzero}{\mathbf{0}}

\newcommand{\bH}{\mathbf{H}}
\newcommand{\bC}{\mathbf{C}}

\newcommand{\cF}{\mathcal{F}}
\newcommand{\cP}{\mathcal{P}}

\newcommand{\cH}{\mathcal{H}}

\newcommand{\cX}{\mathcal{X}}
\newcommand{\cA}{\mathcal{A}}

\newcommand{\cN}{\mathcal{N}}

\newcommand{\cC}{\mathcal{C}}
\newcommand{\cK}{\mathcal{K}}

\newcommand{\bI}{\mathbf{I}}
\newcommand{\m}{\mathbf{m}}

\newcommand{\bbE}{\mathbb{E}}
\renewcommand{\bbP}{\mathbb{P}}

\theoremstyle{plain}
\newtheorem{definition}{Definition}[section]
\newtheorem{lemma}[definition]{Lemma}
\newtheorem{theorem}[definition]{Theorem}
\newtheorem{proposition}[definition]{Proposition}
\newtheorem{corollary}[definition]{Corollary}
\newtheorem{remark}[definition]{Remark}
\newtheorem{example}[definition]{Example}

\usepackage{multirow}
\usepackage{multicol}

\usepackage{caption} 
\definecolor{kartik}{rgb}{0.8,0, 0.2}

\newcommand{\trace}{\operatorname{trace}}
\newcommand{\range}{\operatorname{range}}

\usepackage{xfrac}  

\newcommand{\test}{\mathfrak{T}}
\newcommand{\Id}{\mathbf{I}}

\newcommand{\dettwo}{\operatorname{det}_2}

\begin{document}

\title{\bfseries  
\Large Likelihood Ratio Tests by Kernel Gaussian Embedding\footnote{Research supported by a Swiss National Science Foundation (SNF) grant.}}

\author{
  Leonardo V.\ Santoro
  \qquad 
  Victor M.\ Panaretos   \\ {\footnotesize{\texttt{leonardo.santoro@epfl.ch}  \qquad\,\,\texttt{victor.panaretos@epfl.ch}}}}
  
  \affil{Institut de Math\'ematiques, École Polytechnique Fédérale de Lausanne}

\maketitle

\begin{abstract}
 We propose a novel kernel-based nonparametric two-sample test, employing the combined use of kernel mean and kernel covariance embedding.  Our test builds on recent results showing how such combined embeddings map distinct probability measures to mutually singular Gaussian measures on the kernel's RKHS. Leveraging this ``separation of measure phenomenon", we construct a test statistic based on the relative entropy between the Gaussian embeddings, in effect the likelihood ratio. The likelihood ratio is specifically tailored to detect equality versus singularity of two Gaussians, and satisfies a ``$0/\infty$" law, in that it vanishes under the null and diverges under the alternative. To implement the test in finite samples, we introduce a regularised version, calibrated by way of permutation. We prove consistency, establish uniform power guarantees under mild conditions, and discuss how our framework unifies and extends prior approaches based on spectrally regularized MMD. Empirical results on synthetic and real data demonstrate remarkable gains in power compared to state-of-the-art methods, particularly in high-dimensional and weak-signal regimes.
 
\medskip

\noindent \textbf{MSC2020 classes:} 62G10, 62G20, 62H15, 62H20, 60G15, 46E22

\smallskip

\end{abstract}


\section{Introduction}

Given two unknown probability distributions $\bbP$ and $\bbQ$ on a measurable space \( \cX \), the nonparametric two-sample problem consists in testing
\[
H_0: \bbP = \bbQ \quad \textnormal{versus} \quad H_1: \bbP \neq \bbQ,
\]
based on i.i.d.\ random samples $\{X_1, \dots, X_n\} \sim \bbP$ and $\{Y_1, \dots, Y_m\} \sim \bbQ$, for $n, m \in \bbN$, without making parametric assumptions on either $\bbP$ or $\bbQ$. Nonparametric two-sample testing is a fundamental problem in statistical inference, dating back to the foundational work of Pearson \cite{pearson1900} on the $\chi^2$-test. Early work on one-dimensional distributions made use of the order structure of the real line, using the empirical distribution function (as in the Kolmogorov–Smirnov test \cite{smirnov1939estimation}), or rank information (Wilcoxon \cite{wilcoxon1945} and Mann–Whitney \cite{mannwhitney1947, lehmann1954testing}). In higher dimensions, the problem is considerably more challenging: the vague specification of the alternative allows for manifold and subtle differences --possibly exacerbated by complexity/dimensionality-- and does not inform on how to design efficient criteria. And yet, nonparametric tests are especially well-suited to high-dimensional and complex settings, where the researcher is either unable or unwilling to commit to a strong specification of the hypothesis pair. 

Modern approaches to higher-dimensional nonparametric testing include permutation-based and matching-based methods \cite{friedman1979multivariate, rosenbaum2005exact}, graph-based methods \cite{Henze88, chen2017new}, rank-based approaches \cite{jureckova2012nonparametric, clemenccon2025bipartite}, and depth-based techniques \cite{zuo2000general, Wilcox01042003, Shojaeddin12}. Other popular frameworks include the energy distance \cite{szekely2004testing}, and, more recently, methods based on optimal transport \cite{ramdas2017wasserstein, ghosal2022multivariate, masarotto2024transportation}. In the last few years, kernel-based methods have emerged as a powerful and flexible approach in such settings. Among these, the Maximum Mean Discrepancy (MMD) \cite{gretton2012kernel} has become a standard tool. It operates by embedding probability measures into a reproducing kernel Hilbert space (RKHS) \cite{aronszajn1950theory, berlinet2011reproducing} and computing the norm of the difference between their kernel mean embeddings. MMD enjoys appealing theoretical guarantees and is easily computable, making it widely used in practice. Consequently, a variety of refinements of MMD have been developed to boost the power and robustness of kernel-based two-sample tests. These include procedures for selecting kernels in a data-dependent manner \cite{gretton2012optimal, liu2020learning}, spectral-regularization techniques  \cite{hagrass2024spectral, eric2007testing}, calibration schemes \cite{shekhar2022permutation}, and ensemble or boosting-based strategies \cite{chatterjee2024boosting}. Yet, despite this empirical progress, a comprehensive understanding of the geometric and information-theoretic principles underlying the success of these methods had arguably remained elusive.

Recent work by \citet{santoro2025from} has provided a new theoretical perspective by eliciting a ``separation of measure phenomenon". They show that when distinct probability measures are embedded into an RKHS via their \emph{combined} kernel mean \textit{and} covariance embeddings, they can be identified with Gaussian measures $\cN_{\bbP}$ and $\cN_{\bbQ}$ that are \emph{mutually singular}.  That is, the combined embedding translates $\{H_0:\bbP=\bbQ\mbox{ vs }H_1:\bbP\neq\bbQ\}$ to the considerably sharper pair $\{H_0:\cN_\bbP=\cN_\bbQ\mbox{ vs }H_1:\cN_\bbP\perp\cN_\bbQ\}$. From an information-theoretic perspective, singular measures are infinitely different, showing that even subtle differences between $\bbP$ and $\bbQ$ are greatly amplified in the RKHS: while the original hypothesis pair may be arbitrarily close, in the kernel-embedded space, the the hypotheses are maximally separated. 
 This embedding perspective not only sheds light on the reasons for the remarkable empirical performance of kernel methods, but also suggests that the performance can be enhanced considerably further, as existing strategies do not explicitly target the separation of measure phenomenon (indeed,  \citet{santoro2025from} show that existing methods may well fail to harness it depending on the scenario). The purpose of this paper is to precisely make full use of the potential offered by this phenomenon, by specifically employing information-theoretic quantities targeting it. In particular, the Kullback–Leibler (KL) divergence between the embedded Gaussian measures is either zero (under $H_0$) or infinite (under $H_1$). This provides a natural basis for constructing a powerful test statistic tailored to the very sharp theoretical separation between the null and alternative. 

\paragraph{Contributions.}
We propose a new kernel-based two-sample test built on the embedded Gaussians' likelihood ratio (relative entropy, to be precise). This  idea is implemented by comparing the kernel embeddings of the empirical distributions $\bbP_n$ and $\bbQ_n$ through the suitably regularised relative entropy of their induced empirical Gaussian embeddings $\cN_{\bbP_n}$ and $\cN_{\bbQ_n}$. Regularisation is called for to induce stability at the finite sample level. The test statistic is simple to compute, yet grounded in concrete geometric and probabilistic insights. Unlike previous approaches where the criterion depends smoothly on the magnitude of discrepancies between $\bbP$ and $\bbQ$, our criterion enjoys a much sharper dichotomy: at the population level, any departure from the null leads to asymptotically maximal separation. Concretely, we show that the regularised test statistic is finite for all values of the regularisation parameter, vanishes under the null, and diverges under the alternative as the regularisation parameter vanishes (Theorem~\ref{thm:main}), a dichotomy crystallising in the 0--1 law given by Corollary~\ref{cor:01law}. We establish theoretical guarantees for the proposed test, including consistency against fixed alternatives (Theorem~\ref{thm:consistent}) and uniform consistency under a suitable notion of separation between $\bbP$ and $\bbQ$ (Theorem~\ref{thm:sepa_bound}), for translation-invariant kernels. These results hold under mild conditions and provide a precise characterization of the regularization rate required for consistency. In addition, our framework offers a unifying perspective on previously proposed kernel-based tests. In particular, it shows that spectral regularization methods \cite{eric2007testing, hagrass2024spectral} are a restricted version of a broader information-theoretic criterion, shedding light on the mechanisms that govern their effectiveness. 
 We illustrate the finite-sample performance of our test on a range of synthetic and real datasets. We empirically find that the method offers substantial power improvements compared to existing approaches, excelling in challenging scenarios involving subtle and/or high-dimensional differences between distributions.

\section{Background}

We begin by summarising basic notions in functional analysis and measure theory that will be key to develop our main results. Let $\cH$ be a separable Hilbert space with inner product $\langle\cdot,\cdot\rangle_{\cH} \::\: \cH\times\cH\rightarrow\bbR$ and induced norm $\|\cdot \|_{\cH}\::\: \cH \rightarrow \bbR_{+}$, with $\mathrm{dim}\in \bbN\cup\{\infty\}.$
Given $f,g\in\cH$, their \textit{tensor product} $f\otimes g \::\: \cH\rightarrow\cH$ is the linear operator defined by:
$$
(f\otimes g)u = \langle g,u\rangle_{\cH} f,\qquad u\in\cH.
$$ 
Given Hilbert spaces $\cH_1,\cH_2$ and a linear operator $\bA:\cH_1\rightarrow\cH_2$, we define its adjoint as the unique operator $\bA^*:\cH_2\rightarrow\cH_1$ such that $\langle \bA u, v \rangle_{\cH_2} = \langle u, \bA^* v \rangle_{\cH_1}$ for all $u\in \cH_1, v\in \cH_2$. We say that an operator $\bA:\cH\to\cH$ is \textit{self-adjoint} if $\bA=\bA^*$. We say that $\bA$ is \textit{non-negative definite} (or \textit{positive semidefinite}), and write $\bA\succeq 0$ if it is self-adjoint, and satisfies $\langle\bA h, h\rangle_{\cH} \geq 0$ for all $u\in\cH$. When the inequality is strict for all $x\in\cH\setminus\{0\}$ we call the $\bA$ \textit{positive definite} and write $\bA\succ 0$. 
We say that $\bA$ is compact if for any bounded sequence $\{h_n\}\subset  \cH$, $\{\bA h_n\}\subset  \cH$ contains a convergent sub-sequence.
If $\bA$ is a non-negative, \textit{compact} operator, then there exists a unique non-negative operator denoted by $\bA^{\sfrac{1}{2}}$ that satisfies $( \bA^{\sfrac{1}{2}} )^2 = \bA$.
The \textit{kernel} of $\bA$ is denoted by $\ker(\bA) = \{h\in\cH\::\: \bA h= 0\}$, and its \textit{range} by $\range(\bA)=\{\bA h \::\: h \in \cH\}$. 
We denote the \textit{trace} of an operator $\bA$, when defined, by $\trace(\bA)= \sum_{i\geq 1}\langle \bA e_i,e_i\rangle_{\cH} $, where $\{e_i\}_{i\geq 1}$ is an (arbitrary) Complete Orthonormal System (CONS) of $\cH$. We write
$$
 \|\bA\|_{\text{op}}:=\sup_{\|h\|=1}\|\bA h\|_{\cH}, 
 \qquad\qquad
 \|\bA\|_{\textnormal{HS}}:= \sqrt{\trace(\bA^*\bA)}, 
 \qquad \qquad\|\bA\|_{\textnormal{Tr}}:= \trace(\sqrt{\bA^*\bA})
$$
 for the \emph{operator norm}, \emph{Hilbert-Schmidt norm}, and \emph{trace norm},
 respectively. An operator $\bA$ is said to be \emph{Hilbert-Schmidt} if $\| \bA \|_{\textnormal{HS}} < \infty$, and \emph{trace-class} if $\| \bA \|_{\textnormal{Tr}} < \infty$. 
One always has
 $
  \|\bA\|_{\text{op}} \leq
 \|\bA\|_{\textnormal{HS}} \leq \|\bA\|_{\textnormal{Tr}}
 $. We write $\bI$ for the identity operator on $\cH$. We define the Carleman-Fredholm determinant \cite{simon1977notes,fredholm1903classe} of a self-adjoint operator $\mathbf{H}$ as
$$
\dettwo(\mathbf{I}+\mathbf{H})= \exp\left[ \trace(\log(\mathbf{I} + \mathbf{H}) - \mathbf{H})\right]
$$
It can be shown that the right-hand side converges, and thus that the Carleman-Fredholm determinant is well-defined, for all Hilbert-Schmidt operators with eigenvalues exceeeding $-1$. It is also known that the $\operatorname{map} \mathbf{H} \mapsto \dettwo(\mathbf{I}+\mathbf{H})$ is strictly log-concave, continuous everywhere in $\|\cdot\|_2$ norm and Gateaux differentiable on the subset $\{\mathbf{H}:-1 \notin \sigma(\mathbf{H})\} $ of Hilbert-Schmidt operators. Finally, we say that $\bU: \cH \to \cH$ is a \emph{partial isometry} if $\bU^{\ast}\bU$ (or equivalently, $\bU\bU^{\ast}$) is a projection operator.

\subsection{Reproducing Kernel Hilbert Spaces}
Let $\cX$ be a compact and separable metric space, and consider a Mercer  (positive semidefinite) kernel $k: \cX \times \cX \to \bbR$. The Reproducing Kernel Hilbert Space (RKHS) associated with  $k$, denoted $\cH$, is the Hilbert space of
$f: \cX\to \bbR$ that can be approximated by linear combinations of  kernel functions. That is:
$$
\cH = \overline{
    \left\lbrace \sum_{j=1}^n a_j k\left(\cdot, x_j\right) \::\:  a_j \in \bbR, x_j\in\cX, j=1,\dots,n,\, n \in \mathbb{N} \right\rbrace}
$$
where the closure is taken with respect to the inner-product given by
$$
\langle f, g\rangle_{\cH}:=\sum_{i=1}^n \sum_{j=1}^m a_i b_j k\left(x_i, y_j\right).
$$
$$
f:=\sum_{i=1}^n a_i k\left(\cdot, x_i\right)\quad \textnormal{ and }  \quad g:=\sum_{j=1}^m b_j k\left(\cdot, y_j\right), \quad a_i,b_j \in \bbR, \quad x_i,y_j \in \cX \quad  m,n\in\bbN.
$$ 
In an RKHS the evaluation functional is continuous, and satisfies the reproducing property:
$$
f(x) = \langle f, k_x \rangle \quad \textnormal{for all} \quad x \in \cX \quad \textnormal{and} \quad f\in \cH.
$$
where $k_x = k(x,\cdot)$.   In other words, the value of the function $f$ at any point $x \in \cX$ can be recovered by taking the inner product of $f$ with the kernel function $k_x$.

\medskip

An important property of kernel is $L^2-$\textit{universality}, which refers to the ability to approximate arbitrary continuous functions. Specifically, a kernel $k$ is universal if the RKHS $\cH$ is dense in the space of continuous functions on $\cX$.

\paragraph{Kernel Embeddings of Distributions.}
For a compact and separable metric space $\cX$  let $\cP(\cX)$ denote the set of probability measures (or distributions) on $\cX$.
We write $\m_\bbP$ and $\S_{\bbP}$ for the first and second order embeddings of the measure $\bbP$ in the RKHS $\cH$, respectively defined as:
\begin{equation}\label{eq:mean&covemb}
\m_{\bbP} := \int k_\bx \, d\bbP(\bx),
\quad\textnormal{and} \quad
\S_{\bbP} := \int k_\bx\otimes k_\bx \, d\bbP(\bx).
\end{equation}
These are often referred to as \textit{kernel mean} and \textit{kernel (non-central) covariance} embeddings, respectively. 
The kernel mean embedding $\m_\bbP$ is an element in the RKHS $\cH$, while the kernel covariance embedding $\S_\bbP$ is a linear operator acting on elements of $\cH$. The kernel covariance operator $\S_\bbP$ is positive semi-definite and trace class \cite{bach2022information}. The kernel mean and covariance are respectively characterised by
\begin{equation*}
    \langle \m_\bbP, f\rangle_{\cH} = \bbE_{X\sim\bbP} f(X) \quad \&\quad 
        \langle f, \S_\bbP g\rangle_{\cH}
     = \bbE_{X\sim \bbP}[f(X)g(X)], \qquad\forall f,g\in\cH.
\end{equation*}

\paragraph{Kernel Two Sample Tests.}
Maximum Mean Discrepancy (MMD), first proposed in \cite{gretton2012kernel}, is a popular two-sample test statistic that provides a simple and effective way to compare two distributions, by mapping them into an RKHS and measuring the distance between their mean embeddings. Specifically, MMD
compares the {kernel mean embeddings} of two distributions $\bbP$ and $\bbQ$ by way of the RKHS norm of their difference:
$$
\mathrm{MMD}(\bbP, \bbQ)  := \|\m_\bbP - \m_\bbQ\|_{\cH}^2 = \sup_{f\in\cH}\frac{\langle f, \m_\bbP - \m_\bbQ\rangle_{\cH}}{\langle f,f\rangle_{\cH}}
$$  
This quantity can be estimated directly from samples, uniformly in the dimension $d$, and has been shown to characterize distributional equality when $k$ is characteristic \cite{gretton2012kernel}. It yields a consistent, powerful and cheap test statistic for the two sample problem, but is suboptimal in terms of the  Hellinger distance induced  separation boundary \cite{hagrass2024spectral}. The Kernel Fisher Discrepancy (KFD)  \cite{eric2007testing} builds on such an approach by additionally incorporating the second order structure of the probability distributions into the test statistics.
Specifically, rather than comparing the raw mean embeddings, their approach measures the discrepancy between \( \m_{\bbP} \) and \( \m_{\bbQ} \) through a covariance-dependent norm, defined via spectral regularization of pooled covariance:
$$
\operatorname{KFD}_{\gamma}(\bbP,\bbQ) := \sup_{f \in \cH} \frac{\langle f, \m_\bbP - \m_\bbQ \rangle}{\langle f, (\bSigma_\bbP + \bSigma_\bbQ + \gamma \Id) f \rangle}
$$
where $ \gamma > 0 $ is a regularization parameter ensuring well-posedness of the inverse, and $\overline{\bSigma}$ is the pooled, centred  covariance: 
\begin{equation}
    \overline{\bSigma} := \frac{1}{2}\int (k_\bx - \m_\bbP)\otimes (k_\bx - \m_\bbP) \, d\bbP(\bx) + \frac{1}{2}\int (k_\bx - \m_\bbQ)\otimes (k_\bx - \m_\bbQ) \, d\bbQ(\bx)
\end{equation}
The KFD thus  seeks the direction that maximizes the mean difference \( \mu_{\bbP} - \mu_{\bbQ} \)  taking into account the average variability under \( \bbP \) and \( \bbQ \), and can be seen as a refinement of MMD that adapts the testing direction according to the geometry of the two distributions. In a similar spirit, \cite{hagrass2024spectral} generalise this approach and propose a family of two-sample tests based on \emph{spectrally regularized differences} between kernel mean embeddings, including but not limited to 
$$
 \| (\overline{\bSigma} + \gamma\Id)^{-\sfrac{1}{2}}(\m_{\bbP} - \m_{\bbQ})\|_{\cH}^2 
$$ which incorporates the spectral decay of the covariance into the mean difference in a variety of possible ways. \cite{hagrass2024spectral} further show that with appropriate choice of the regularization, the spectral regularised statistic is consistent and can achieve minimax optimality over suitably separated alternatives.

\section{Regularised Kernel Likelihood Testing}
In this section we propose and study a test statistic based on a notion of regularised likelihood ratio between kernel Gaussian embeddings of probability distributions. We will assume that:
\begin{itemize}
    \item[\textbf{(A0)}] $\cX$ is a compact, separable and locally convex metric space.

    \item[\textbf{(A1)}] $k\::\: \cX\times\cX\to \bbR$ is a continuous positive definite kernel on $\cX$. 
\end{itemize}
To obtain suitable rates of convergence, we will further assume that the kernel is bounded:
\begin{itemize}
    \item[\textbf{(A2)}] $\sup_{s,t}|k(s,t)| \leq K$ for some positive $K>0$.
\end{itemize}

We begin by introducing the concept of kernel Gaussian embedding of probability distribution, 
which associates a distribution $\bbP$ on $\cX$ with \emph{Gaussian} measure $\cN_\bbP$ on the RKHS $\cH$,
with mean and covariance given by the corresponding kernel mean and covariance embeddings \eqref{eq:gaussianemb}.

\begin{definition}\label{def:kge}
    Let $\bbP\in \cP(\cX)$ be a probability measure on a compact and separable metric space $\cX$. Its central and non-central \textit{kernel Gaussian embedding} are respectively given by the Gaussian measures on the RKHS $\cH$:
\begin{equation}
    \label{eq:gaussianemb}
    \bbP\mapsto \cN(\mathbf{0}, \S_{\bbP})
    \qquad\textnormal{and}\qquad
    \bbP\mapsto \cN(\m_{\bbP}, \S_{\bbP})
\end{equation}
where $\m_\bbP$ and $\S_{\bbP}$ are the mean and covariance embeddings of $\bbP$ in the RKHS $\cH$, defined in \eqref{eq:mean&covemb}.
\end{definition}
The kernel Gaussian embedding
 provides a tractable way to express the distribution $ \bbP $ using  (the first \emph{and}) second moment in the kernel space. Given probability distributions $\mathbb{P}$ and $\mathbb{Q}$ on $\cX$,  \citet{santoro2025from} showed that Gaussian embeddings into RKHS yield \textit{mutually singular} measures.
 \begin{theorem}[Santoro, Waghmare and Panaretos \cite{santoro2025from}]\label{thm:singular}
    Let $\bbP,\bbQ$ be probability measures on a compact and separable metric space $\cX$, and $k: \cX \times \cX \to \bbR$ be a universal reproducing kernel.
    Then:
    \begin{equation}
        \label{eq:equivalence}
    \bbP \neq \bbQ 
    \quad\Longleftrightarrow \quad
    \cN(0, \S_\bbP) \perp \cN(0,\S_\bbQ)
    \quad\Longleftrightarrow \quad
    \cN(\m_\bbP, \S_\bbP) \perp \cN(\m_\bbQ,\S_\bbQ)
    \end{equation}
    where $\perp$ denotes the \textit{mutual singularity} of measures.
\end{theorem}

 This result establishes that the alternative $\{H_1:\bbP\neq\bbQ\}$, is equivalent to $\{H_1':\cN_{\bbP}\perp\cN_{\bbQ}\}$. That is, two arbitrary distributions coincide or differ according to whether their Gaussian embeddings have same or (almost everywhere) disjoint supports (see Figure \ref{fig:gaussian_embedding}). 

\begin{figure}[htbp]\label{fig:cartoon-embedding}
    \centering
    \begin{tikzpicture}[
        box/.style={rectangle, draw=gray, fill=gray!10, minimum width=2.5cm, minimum height=1.2cm, font=\small, align=center},
        arrow/.style={-Stealth, thick},
    ]

        \node[draw, thick, rectangle, inner sep=0pt] (img1)
        {\includegraphics[width=0.25\textwidth]{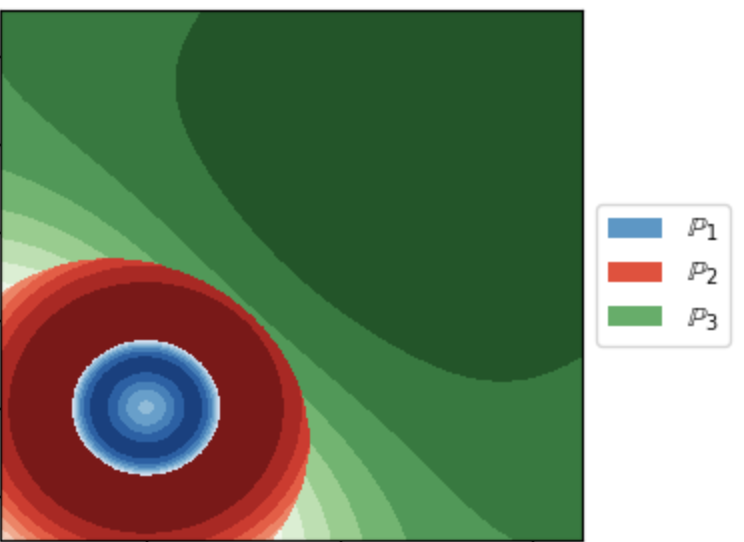}};
        \node[above=2mm of img1.north] {\textbf{Measures on $\cX$}};

        \node[box, right=1.5cm of img1] (kernel)
            {\(\displaystyle \mathbb{P} \mapsto \cN\big(\m_{\mathbb{P}}, \S_\bbP\big)\)};
        \node[font=\bfseries, above=6mm of kernel.north] (title) {Gaussian RKHS Embedding};

        \node[draw, thick, rectangle, inner sep=0pt, right=1.5cm of kernel] (img2)
            {\includegraphics[width=0.25\textwidth]{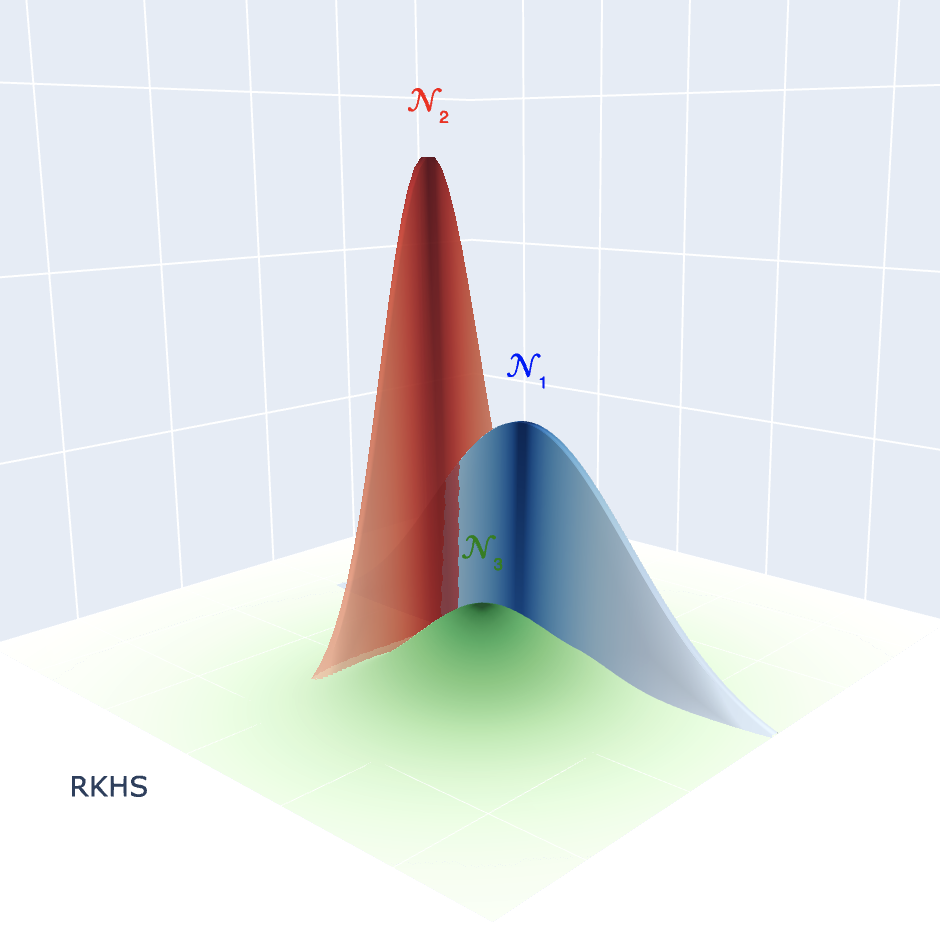}};
        \node[above=2mm of img2.north, text=green!70!black] {\textbf{Gaussian measures on $\cH$}};

        \draw[arrow] ([xshift=5pt]img1.east) -- ([xshift=-5pt]kernel.west);
        \draw[arrow] ([xshift=5pt]kernel.east) -- ([xshift=-5pt]img2.west);

    \end{tikzpicture}
    \caption{Cartoon visualising  embedding of arbitrary distributions on $\cX$ to Gaussian measures on the RKHS $\cH$  via \eqref{eq:gaussianemb}: \textit{distinct} measures on $\cX$ are mapped to \textit{mutually singular} Gaussian measures on $\cH$.}
    \label{fig:gaussian_embedding}
\end{figure}

\begin{center}
\begin{minipage}{.95\textwidth}    
\begin{framed}
        \begin{center}
        \textsc{Equivalence and singularity of Gaussian measures}
    \end{center}
    \textit{
    Recall that two measures $\mu,\nu$ are said to be mutually singular -- written $\mu \perp \nu$ -- if there exists a measurable set $\cA\subset \cX$ such that $\mu(\cA) = 1$ while $\nu(\cA^c) =1$.
    Similarly, $\mu, \nu$ are said to be equivalent -- written $\mu\sim\nu$ -- if for all measurable set $\cA\subset \cX, \mu(\cA) >0 \Leftrightarrow \nu(\cA) > 0$.
    Specifically in the Gaussian case,
 the Feldman–Hájek theorem \cite{feldman1958equivalence,hajek1958property} is a fundamental result stating that that two Gaussian measures 
on a (locally convex) space $\cX$
are either equivalent measures or else mutually singular, with no intermediate situation possible. Furthermore, the Feldman–Hájek theorem gives an explicit, verifiable description of the circumstances under which Gaussian measures $\cN(\bm_1,\bC_1)$ and $ \cN(\bm_2,\bC_2)$ are equivalent, or mutually singular:}
\begin{equation*}
    \cN(\bm_1, \bC_1) \sim \cN(\bm_2, \bC_2)
    \quad \Longleftrightarrow \quad
    \begin{cases}
         \bm_2 - \bm_1 \in \operatorname{Range}(\bar{\bC}^{\sfrac{1}{2}}), \quad \text{where } \bar{\bC} = \tfrac{1}{2}(\bC_1 + \bC_2), \quad \text{and} \\
         \bC_1 = \bC_2^{\sfrac{1}{2}} (\bH-\bI) \bC_2^{\sfrac{1}{2}}, \quad \text{for some  Hilbert-Schmidt } \bH \succ -\bI.
    \end{cases}
\end{equation*}
    \end{framed}
\end{minipage}
\end{center}

\begin{center}
\begin{minipage}{.95\textwidth}
\begin{framed}
    \begin{center}
        \textsc{Absolute continuity, relative entropy, Gaussian measures}
    \end{center}
\textit{ 
Consider probability distributions $\mu,\nu$ on $\cH$, and suppose \ $\mu\ll\nu$, so that the Radon-Nikodym derivative $d\mu/d\nu$ exists $\nu$-almost everywhere, in that
$
\mu(\cA) = \int_{\cA}\left({d\mu}/{d\nu}\right)(x) d\nu(x)
$
for any measurable set $\cA\subset\cX$.
With this notation, when $\mu\ll\nu$, the Kullback-Leibler divergence of $\mu$ wrt to $\nu$, or relative entropy, is given by
$
D_{\mathrm{KL}}(\mu,\nu) = \int_{\cX}\log\left({d\mu}/{d\nu}\right) d\mu.
$
Focusing on the case of (equivalent) \emph{Gaussian} measures $\mu=\cN(\m_1,\S_1)$ and $\nu=\cN(\m_2,\S_2)$, the Kullback-Leibler divergence admits a particularly tractable form:
$$
D_{\mathrm{KL}} \big( \cN(\m_1,\S_1)\:||\:  \cN(\m_2, \S_2) \big) =
 \frac{1}{2}\|{\S_2^{-\sfrac{1}{2}}}(\m_1 - \m_2) \|^2 -  \frac{1}{2}\log\dettwo\left( \Id - \S_2^{-\sfrac{1}{2}}(\S_1 - \S_2)\S_2^{-\sfrac{1}{2}} \right)
$$
where $\dettwo$ denotes the Fredholm-Carleman determinant \cite{simon1977notes,fredholm1903classe}, generalising the determinant to Hilbert-Schmidt operators. 
Note that the Gaussian relative entropy captures both the mean shift (through the Mahalanobis distance between \(\m_1\) and \(\m_2\)) and the covariance discrepancy (via the Fredholm-Carleman determinant involving \(\S_1\) and \(\S_2\)). 
}
\end{framed}
\end{minipage}
\end{center}

To \emph{operationalise} the result in Theorem~\ref{thm:singular}, we characterise the discrepancy between Gaussian embeddings in information-theoretic terms, specifically through the notion of regularized relative-entropy in the next section.

\paragraph{Kernel Regularised Relative Entropy.} Given two probability distributions \( \bbP \) and \( \bbQ \) on $\cX$, consider their corresponding Gaussian embeddings \( \cN(\m_{\bbP}, \S_{\bbP}) \) and \( \cN(\m_{\bbQ}, \S_{\bbQ}) \), which are \textit{Gaussian} measures on the RKHS $\cH$. 
Given  $\gamma>0$, we consider the following discrepancy:
\begin{equation}\label{eq:KRKL}
\begin{split}
D_{\gamma, \mathrm{KL}} \big( \cN(\m_{\bbP}, \S_{\bbP}) \:||\:  \cN(\m_{\bbQ}, \S_{\bbQ}) \big) := &\frac{1}{2}\| (\S_{\bbQ} + \gamma \Id)^{-\sfrac{1}{2}} (\m_{\bbP} - \m_{\bbQ}) \|_{\cH}^2 
   \\&\:  -  \frac{1}{2}
      \log\dettwo \left(\Id + (\S_{\bbQ} + \gamma \Id)^{-\sfrac{1}{2}}(\S_{\bbP}  - \S_{\bbQ})(\S_{\bbQ} + \gamma \Id)^{-\sfrac{1}{2}} \right).
\end{split}
\end{equation}
This can be interpreted as a \textit{regularised Kullback-Leibler} divergence between the two Gaussian embeddings associated with $\bbP,\bbQ$.
In fact, algebraically, the regularized KL divergence  in \eqref{eq:ker_regularized_KL} can be formally be identified with the relative entropy between the  measures \(\cN(\m_\bbP, \S_\bbP + \gamma \Id)\) and \(\cN(\m_\bbQ, \S_\bbQ + \gamma \Id)\) -- albeit these are not well defined, as the identity is not trace-class in infinite dimensions; nevertheless, one can make sense of this employing a sequence of projections converging to the identity. See also Figure~\ref{fig:regkl}.

 We now have the ingredients to introduce the criterion that will form the basis for our test statistic.
    Given two probability distributions \( \bbP \) and \( \bbQ \) on $\cX$ and $\gamma>0$, we define
\begin{equation}\label{eq:teststat}
    \test_{\gamma}(\bbP,\bbQ) :=  D_{\gamma, \mathrm{KL}} \big( \cN(\m_{\bbP}, \S_{\bbP}) \:||\:  \cN(\m_{\bbQ}, \S_{\bbQ}) \big) 
\end{equation}
to be the $\gamma$-\emph{regularlised kernel relative entropy} between the corresponding Gaussian embeddings \( \cN(\m_{\bbP}, \S_{\bbP}) \) and \( \cN(\m_{\bbQ}, \S_{\bbQ}) \).
It is easy to see that, for all $\gamma>0$, the mapping $(\bbP, \bbQ)\to \test_{\gamma}(\bbP,\bbQ) $ is a probability divergence over measures $\bbP, \bbQ\in \cP(\cX)$, since $\test_{\gamma}(\bbP,\bbQ) \geq 0$, with equality whenever (in fact, if and only if) $\bbP,\bbQ$ are equal.
Furthermore, it was shown in \cite[][Theorem 2]{minh2021regularized} seen that if the measures  \( \cN(\m_{\bbP}, \S_{\bbP}) \) and \( \cN(\m_{\bbQ}, \S_{\bbQ}) \) are equivalent, then the regularized divergence converges, for vanishing regularisation, to the relative entropy. Contrarily, in the absence of equivalence (i.e. in the singular case), we show that that the regularized divergence \textit{diverges to infinity} as regularization vanishes. In particular, since the (measure theoretic) equivalence of the Gaussian embeddings $ \cN(\m_{\bbP}, \S_{\bbP})$ and $\cN(\m_{\bbQ}, \S_{\bbQ})$  is (logically) equivalent to the equality of $\bbP$ and $\bbQ$ by the main result in \cite{santoro2025from}, we have the following result:

\begin{theorem}\label{thm:main}
    Let $\bbP,\bbQ$ be probability measures on $\cX$.
    Let $k: \cX \times \cX \to \bbR$ be a universal reproducing kernel. 
    Then $\test_{\gamma}(\bbP,\bbQ)$ is finite for any positive $\gamma>0$, while:
    \begin{equation}
        \label{eq:ker_regularized_KL}
        \lim_{\gamma \to 0}  
        \test_{\gamma}(\bbP,\bbQ) 
        = \begin{cases}
            0, \: &\textnormal{if} \quad \bbP=\bbQ,\\
            \infty, \: &\textnormal{if} \quad \bbP\neq \bbQ.
        \end{cases}
    \end{equation}

\end{theorem}

Observe that, since relative entropy can be interpreted as the expected log-likelihood ratio, the left hand side of \eqref{eq:ker_regularized_KL} can be understood as a regularized likelihood ratio between the two Gaussian embeddings \(\cN(\m_\bbP, \S_\bbP)\) and \(\cN(\m_\bbQ, \S_\bbQ)\), which explains the title of the paper.
Plugging-in the empirical versions of $\bbP$ and $\bbQ$ yields a test statistic that is well-defined, interpretable, and computable, and which promises nearly perfect discrimination (as sample size grows and regularization strength decays).

\begin{remark}[Connection to kernel Fisher discriminant and spectral MMD]
The first term in the regularized relative entropy $D_{\gamma, \mathrm{KL}}$, i.e.\ the Mahalanobis term $\frac{1}{2}\| (\S_{\bbQ} + \gamma \Id)^{-\sfrac{1}{2}} (\m_{\bbP} - \m_{\bbQ}) \|_{\cH}^2,$
reconnects to well-known statistics which appeared in the literature. In particular, replacing the non-central covariance operators $\S_{\bbP}, \S_{\bbQ}$ with their central counterparts $\bSigma_{\bbP}, \bSigma_{\bbQ}$ yields expressions that are closely related to the kernel Fisher discriminant (KFD), and recovers the spectral-regularized Maximum Mean Discrepancy (SR-MMD) \cite{eric2007testing, hagrass2024spectral}. In this sense, our formulation extends these existing methods but also provides a new conceptual justification for their empirical effectiveness: these statistics can be viewed as approximations of the KL divergence between embedded Gaussian surrogates of $\bbP$ and $\bbQ$. This connection offers a unifying information-theoretic perspective on kernel-based two-sample testing. In fact, as remarked in \cite{santoro2025from}, criteria based on mean embeddings alone provide a weaker measure of discrimination, and are not guaranteed to harness the singularity effect achieved by covariance embeddings: see \cite[][Proposition 4.1]{santoro2025from}.
\end{remark}

\begin{remark}
    
    The proof of the Theorem shows in fact that \textit{centered} Gaussian embeddings $\cN(0,\S_\bbP), \cN(0,\S_\bbQ)$ obey the same dichotomic behavior. That is because the driving factor in the divergence is a regularised Hilbert-Schmidt discrepancy, coming from the Feldman-Hajek criterion. That is, we in fact prove that:
        \begin{equation}
     \label{eq:ker_regularized_HS}
        \lim_{\gamma \to 0}\|[\gamma \Id + \S_{\bbP}]^{-\sfrac{1}{2}}[\S_\bbQ - \S_\bbP]  [\gamma \Id + \S_{\bbP}]^{-\sfrac{1}{2}}\|_{\textnormal{HS}} =
            \begin{cases}
                0, \: &\textnormal{if} \quad \bbP=\bbQ,\\
                \infty, \: &\textnormal{if} \quad \bbP\neq \bbQ
            \end{cases}
        \end{equation}
       and the limit lower bounds the kernel relative entropy of the \textit{centered} (mean zero) Gaussian embeddings
       , which in turn lower bounds the relative entropy of the uncentred Gaussian embeddings. 
    Thus, central or non-central Gaussian embeddings obey a similar paradigm, which sets basis for two-sample testing.
\end{remark}

\begin{remark} 
The divergence in \eqref{eq:KRKL} was considered in \citet[][Definition~6] {minh2021regularized} and \cite{minh2020infinite} in the context of regularised divergences between Gaussian measures on Hilbert spaces. Indeed, it can be seen  \cite[][Proof of Theorem 3, p.371]{minh2020infinite} that the ridge-regularised Carleman-Fredholm determinant can be expressed as:
$$
 \log\dettwo \left(\Id + (\S_{\bbQ} + \gamma \Id)^{-\sfrac{1}{2}}(\S_{\bbP}  - \S_{\bbQ})(\S_{\bbQ} + \gamma \Id)^{-\sfrac{1}{2}} \right) =
 d_{\textrm{logdet}}^1(\S_{\bbP} + \gamma \Id,\S_{\bbQ} + \gamma \Id) 
$$
where $d_{\textrm{logdet}}^1$ for the $\alpha$-log-determinant divergence with $\alpha=1$,
\end{remark}

\begin{figure}[]
  \centering
  \begin{subfigure}[b]{0.5\textwidth}
    \includegraphics[width=.8\linewidth]{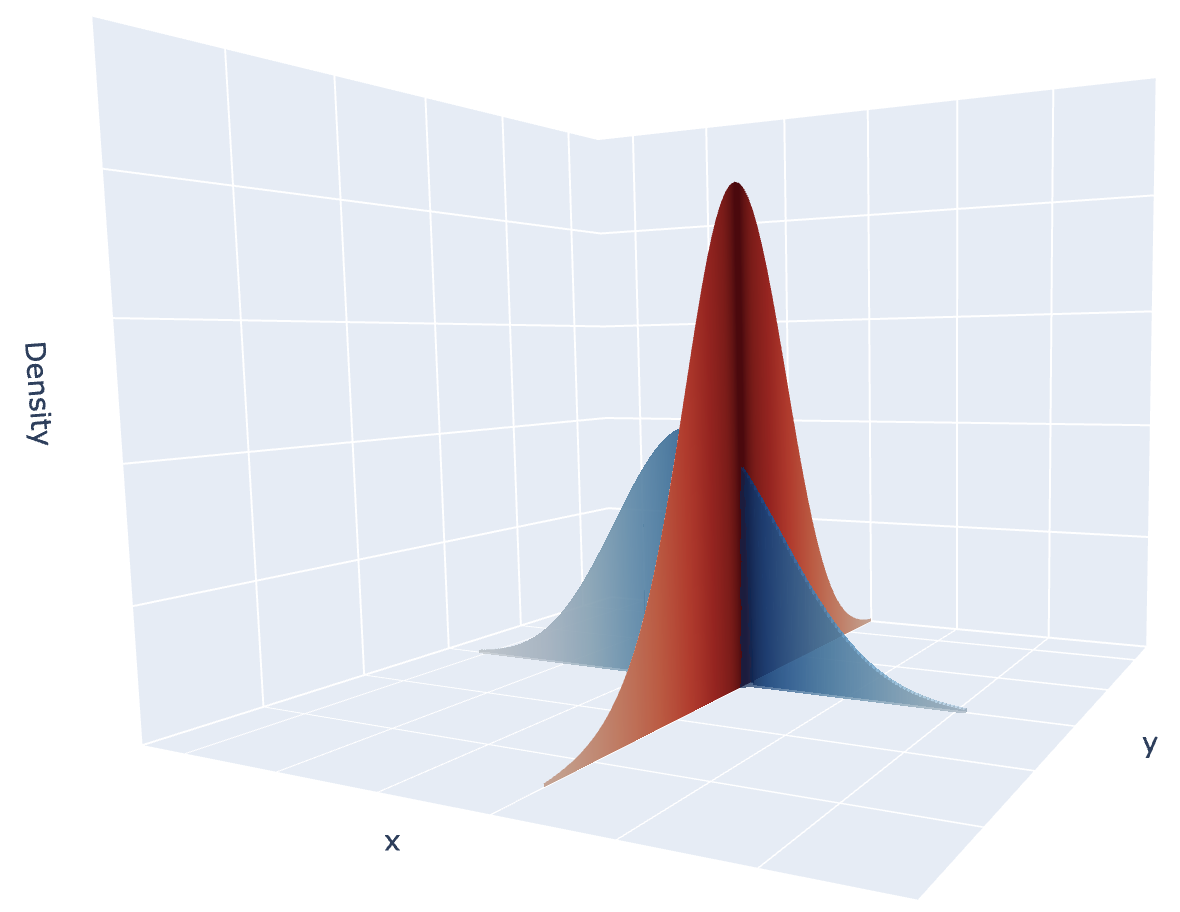}
\caption{Two mutually singular Gaussian measures in $\bbR^2$, $\cN(\mu_1,\Sigma_1)$ and $\cN(\mu_2,\Sigma_2)$.
}

  \end{subfigure}
  \hfill
  \begin{subfigure}[b]{0.4\textwidth}
    \includegraphics[width=\linewidth]{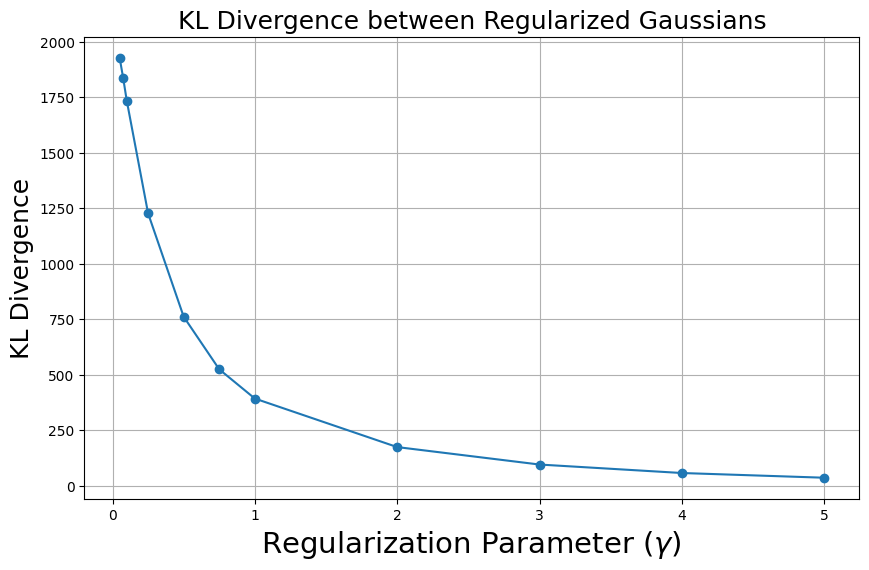}
    \caption{The Kullback-Leibler divergence obtained adding a $\gamma$-ridge to the respective covariances: $\operatorname{KL}\left( \cN(\mu_1,\Sigma_1 + \gamma I) || \cN(\mu_2,\Sigma_2 + \gamma I)\right)$, for decreasing values of $\gamma$.}
  \end{subfigure}

  \vspace{1em}

  \begin{subfigure}[b]{1\textwidth}
    \centering
    \includegraphics[width=\linewidth]{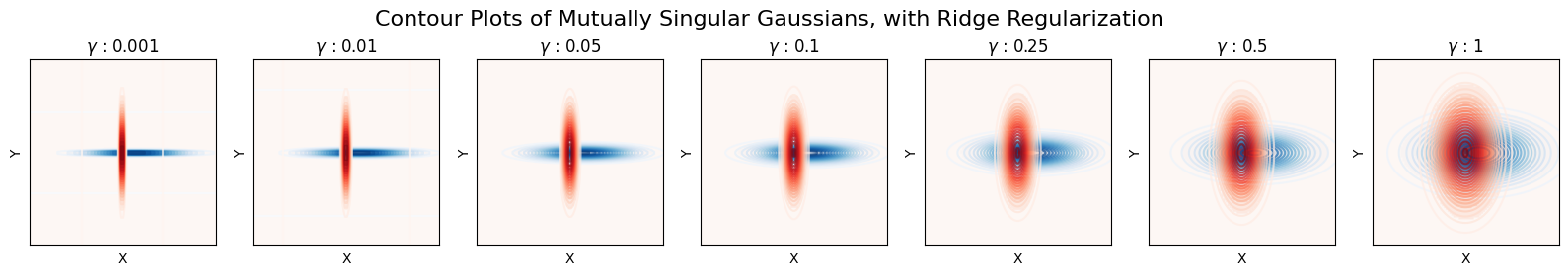}
    \caption{Contour lines of the regularised densities of the two Gaussians, for increasing regularisation.}
  \end{subfigure}
  \caption{Regularised KL divergence between mutually singular Gaussian measures on $\bbR^2$.}
\label{fig:regkl}
\end{figure}

\begin{example}[Regularised likelihood ratio of scaled Brownian Motion]
Let \( k_v(s,t) = v \cdot \min(s,t) \) denote the covariance kernel of standard Brownian motion on the interval \([0,1]\) with variance parameter \( v > 0 \). Let \( C_v : L^2([0,1]) \to L^2([0,1]) \) be the corresponding integral operator:
\[
(\bC_v f)(t) = \int_0^1 k_v(s,t) f(s) \, ds = v \int_0^1 \min(s,t) f(s) \, ds.
\]
 The kernel \( k(s,t) = \min(s,t) \) has the orthonormal eigenfunctions and eigenvalues respectively given by
\[
\phi_n(t) = \sqrt{2} \sin\left( \left(n - \tfrac{1}{2} \right)\pi t \right),
\quad \lambda_n = \frac{1}{((n - \tfrac{1}{2})\pi)^2},
\qquad n = 1, 2, 3, \dots.
\]
It is trivial to see that, for \( \bC_1 = \bC_{v_1} \) and \( \bC_2 = \bC_{v_2} \) for two variances \( v_1, v_2 > 0 \), the two operators do not satisfy the Feldman-Hajek conditions, and therefore that the Gaussian measures $\cN(\mathbf{0},\C_1)$ and $\cN(\mathbf{0}, \C_2) \big)$ are mutually singular.
We wish to study the regularised kernel likelihood ratio $D_{\gamma, \mathrm{KL}} \big( \cN(\mathbf{0},\C_1)\:||\:  \cN(\mathbf{0}, \C_2) \big)$
for regularisation $\gamma>0$. In the Appendix (Lemma~\ref{lem:lowerbound}) we see that this is lower bounded by:

\[
 \|\bH_\gamma \|_{\mathrm{HS}}^2 = \sum_{n=1}^\infty \left\| \bH_\gamma  \phi_n \right\|_{L^2}^2
= \sum_{n=1}^\infty \left( \frac{(v_2 - v_1) \lambda_n}{v_1 \lambda_n + \gamma} \right)^2
= (v_2 - v_1)^2 \sum_{n=1}^\infty \left( \frac{1}{((n - \tfrac{1}{2})\pi)^2 v_1 + \gamma} \right)^2,
\]
where $\bH_\gamma = (\bC_1 + \gamma \Id)^{-1/2} (\bC_2 - \bC_1)(\bC_1 + \gamma \Id)^{-1/2}$. Note that the sum converges, since the denominator grows quadratically.
As $\gamma\to 0$, however, the sum  diverges. To see this, we approximate it by an integral:
\[
\sum_{n=1}^\infty \frac{1}{(((n - \tfrac{1}{2})\pi)^2 v_1 + \gamma)^2}
\approx \int_0^\infty \frac{1}{(\pi^2 v_1 x^2 + \gamma)^2} \, dx = \frac{1}{4 \gamma^{3/2} \sqrt{v_1}}.
\]
Therefore, the $\gamma$-regularised likelihood ratio of $\cN(0,\bC_1), \cN(0,\bC_2)$ diverges as
$ \frac{(v_2 - v_1)^2}{4 \sqrt{v_1}} \cdot \gamma^{-3/2}$ as $ \gamma \to 0^+.$
\end{example}

\subsection{The Two-Sample Test}\label{sec:the_test}
Given two samples $\{X_1, \dots, X_n\} \overset{\textsc{iid}}{\sim} \mathbb{P}$ and $\{Y_1, \dots, Y_m\} \overset{\textsc{iid}}{\sim} \mathbb{Q}$ inducing empirical distributions $
\bbP_n =\frac{1}{n}\sum_{i=1}^n \delta_{X_i},
$ and $
\bbQ_m =\frac{1}{n}\sum_{i=1}^m \delta_{Y_i}
$, the discussion in the previous paragraph motivates the use of the test statistic
\begin{equation}
\label{test_emp_KL}
    \test_{\gamma_{nm}}^{\textnormal{KL}} = \test^{\textnormal{KL}}_{\gamma_{nm}}(\bbP_n,\bbQ_n):= D_{\gamma_{nm}, \mathrm{KL}} \big( \cN(\m_{\bbP_n}, \S_{\bbP_n}) \:||\:  \cN(\m_{\bbQ_n}, \S_{\bbQ_n}) \big),
\end{equation}
for $\gamma =\gamma_{nm}$, vanishing at a suitable rate as $n,m\to\infty$, in order to ensure that the level is maintained while targeting maximal power. 
  Specfically, we reject when the statistic's value exceeds a threshold $\tau_{nm}(\alpha)$, chosen to ensure adherence to a prespecified level $\alpha\in(0,1)$: 
 \begin{equation}\label{eq:decision_function}
        \psi_{nm} = \begin{cases}
            1\quad \textnormal{if} \quad \test^{\textnormal{KL}}_{\gamma_{nm}}(\bbP_n,\bbQ_m) \geq  \tau_{nm}(\alpha)\\
            0\quad \textnormal{otherwise.}
        \end{cases}
    \end{equation}
The critical value $ \tau_{nm}$ will be chosen by way of a permutation method. For given values \( n, m \geq 0 \), let \( \{Z_j\}_{j=1}^{n+m} \) denote the pooled sample \( \{X_1, \dots, X_n, Y_1, \dots, Y_m\} \). Let \( \sigma  \) be a uniformly random permutation of \( \{1, \dots, n+m\} \), denoted $\sigma\in \mathfrak{S}_{n+m}$. The \( (1 - \alpha) \)-quantile of the permutation distribution of the regularized test statistic is defined as:
\begin{equation}\label{decision:kl}
    \hat q^{\gamma}_{1-\alpha} := 
    \inf\left\{t \geq 0 \::\: P_\sigma\left( \test_{\gamma}\left( \frac{1}{n} \sum_{j=1}^n \delta_{Z_{\sigma(j)}}, \frac{1}{m} \sum_{j=1}^m \delta_{Z_{\sigma(n+j)}} \right) \leq t \right) \geq 1 - \alpha \right\},
\end{equation}
where \( P_\sigma \) denotes the probability with respect to the random permutation \( \sigma \). For a test at level $\alpha\in(0,1)$, then set the critical value to be the corresponding $1-\alpha$ quantile:
$$\tau_{nm}(\alpha)\equiv  \hat{q}^{\gamma_{nm}}_{1-\alpha}.$$

In the next section, we establish guarantees for this decision rule. We first establish approximation guarantees of population-level quantities by their empirical counterparts, and determine the choice of a suitable regularisation rate $\gamma = \gamma_{n,m}$ to ensure consistency of the permutation-calibrated test under fixed alternatives. Then, we establish consistency under local alternatives and investigate the corresponding separation boundary.

\subsection{Theoretical Guarantees}
Our first result provides a probabilistic bound on the deviation between the empirical and population versions of the regularized test statistic.

\begin{proposition}[Estimation]\label{prop:prob_bound:KL}
    Let $\{X_j\}_{j\geq 1},\{Y_j\}_{j\geq 1}$ be i.i.d.\ sequences drawn from $\bbP,\bbQ \in \cP(\cX)$, respectively.
 For any $\varepsilon>0$:
    \begin{equation*}\label{eq:probtailKL}
    \begin{split}
        P\left( \left\lvert\test^{\textnormal{KL}}_{\gamma}(\bbP_n,\bbQ_m) -  \test^{\textnormal{KL}}_{\gamma}(\bbP,\bbQ) \right\rvert  > \frac{12K^{3/2}}{\gamma^2}\left( \frac{1}{\sqrt{n}} + \frac{1}{\sqrt{m}}\right) + \varepsilon  \right)
       \leq  \exp\left( - \frac{\gamma^4 n m }{18K^4(n + m)}\varepsilon^2\right)
    \end{split}
    \end{equation*}
In particular, if $n=m$ and $\gamma = Cn^{\sfrac{( \beta - 1)}{4}}$ for some $0<\beta<1$, then:
 \begin{equation*}\label{eq:probtailKL:n=m}
    \begin{split}
       P\left( \left\lvert\test^{\textnormal{KL}}_{\gamma}(\bbP_n,\bbQ_n) -  \test^{\textnormal{KL}}_{\gamma}(\bbP,\bbQ) \right\rvert  > C_1n^{-\beta/2} + \varepsilon  \right)
       \leq  \exp\left( - C_2n^{\beta}\varepsilon^2\right)
    \end{split}
    \end{equation*}
for all $\gamma>0$ and measures $\bbP,\bbQ\in\cP(\cX)$, and universal constants $C_1,C_2$.
\end{proposition}

Proposition~\ref{prop:prob_bound:KL} yields a key property of the statistic \eqref{test_emp_KL}, reflecting the same dichotomy as at the population level (Theorem~\ref{thm:main}), and comprising a key step to establishing consistency of the proposed test.

\begin{corollary}[0-1 Law]\label{cor:01law}
    Let $\{X_j\}_{j\geq 1},\{Y_j\}_{j\geq 1}$ be i.i.d.\ sequences drawn from $\bbP,\bbQ \in \cP(\cX)$, respectively. and consider a sequence of regularization parameters $\gamma_{nm}$ indexed by the sample sizes.
        If $\gamma_{nm}^4\min\{n,m\}\to\infty$ as $(n,m)\to\infty$, then, for any $\varepsilon>0$: 
        $$
        P(\test^{\textnormal{KL}}_{\gamma_{nm}}(\bbP_n,\bbQ_m) > \varepsilon) = \begin{cases}
            o(1) &\text{under }H_0,\\
            1 - o(1) &\text{under }H_1,
        \end{cases}
        \quad \text{as } (n,m)\to\infty.
        $$
\end{corollary}

\medskip

We can next prove that our testing procedure is consistent against fixed alternatives when calibrated by permutation tests: that is, with probability approaching $1,$ as the sample size grows, the test \eqref{decision:kl} rejects the null hypothesis when the null hypothesis is false, for a fixed alternative $H_1$.

\begin{theorem}[Consistency against fixed alternatives]\label{thm:consistent}
    Let $X = \{X_j\}_{j\geq 1}, Y=\{Y_j\}_{j\geq 1}$ be i.i.d.\ sequences drawn from $\bbP,\bbQ \in \cP(\cX)$, respectively, and let $\gamma_{nm}$ be a sequence of regularisation parameters.
    Set $q^{\gamma_{nm}}_{1-\alpha}$ to be the $1-\alpha$ quantile of the permutation distribution of the test statistic, based on the pooled sample.
    Consider the decision function:
    \begin{equation*}
        \psi_{nm} = \begin{cases}
            1\quad \textnormal{if} \quad \test^{\textnormal{KL}}_{\gamma_{nm}}(\bbP_n,\bbQ_m) \geq  q^{\gamma_{nm}}_{1-\alpha}\\
            0\quad \textnormal{otherwise.}
        \end{cases}
    \end{equation*}
    Then, under the null, the Type I error is controlled at $\alpha$,
    whereas, under the alternative, we correctly reject the null with probability converging to $1$ as sample sizes diverge and $\gamma$ vanishes at a suitable rate. That is :
    \begin{align*}
    E_{H_0}[\psi_{n}(\bbP_n,\bbQ_n)]\leq \alpha,
    \qquad\text{and}\qquad  E_{H_1}[\psi_{n}(\bbP_n,\bbQ_n)]\geq   1 - o(1),
\end{align*}
if $\gamma^4 \cdot\min\{n,m\}\to \infty$, as $(n,m)\to\infty$.
\end{theorem}

Next we investigate the separation boundary (or contiguity radius) between the null and alternative, which quantifies the minimal detectable discrepancy between distributions in terms of a given probability metric. Specifically, for a pseudometric $\rho$ on (a class of) distributions, we define the set of $\Delta$-separated alternatives as
$$
\cP_\Delta := \left\{(\bbP, \bbQ) \in \cC^2 : \rho^2(P, Q) \geq \Delta \right\}.
$$
Assume, for simplicity that $n=m$. For a given separation $\Delta>0$, level $\alpha$, and power $1-\delta$, we wish to quantify the sample size $n$ that is required in order to guarantee level at most $\alpha$ and power at least $1-\delta$, \emph{uniformly} over $\cP_\Delta$. This trade-off defines the separation boundary of the test and captures its uniform sensitivity to local alternatives.

The next Theorem delineates our test's separation boundary, when the discrepancy between $\bbP$ and $\bbQ$ is measured via a suitable $L^2$ norm of their characteristic functions. To obtain this result, we will additionally require the following conditions:
\begin{itemize}
    \item[\textbf{(A3)}] $\cX\subset \bbR^p$, for $p\geq 1$ is a compact euclidean subset.
    \item[\textbf{(A4)}] the kernel is translation inveriant, i.e., $k(\cdot,\cdot) = \Psi( \cdot - \cdot)$, for $\Psi:\cX\to \bbR$   bounded, continuous and positive definite.
\end{itemize}

In this restricted setting, let us 
define the following class of $\Delta$-separated alternatives:
$$
\cP_{\Delta}^{\Lambda}:= \left\lbrace (\bbP,\bbQ) \in\cP(\cX) \::\:  \|\phi(\bbP) - \phi(\bbQ)\|_{L^2(\cX,\Lambda)} >  \Delta \right\rbrace
$$
where $\Lambda = \cF(\Psi^2)$ denotes the Fourier transform of the squared kernel generator $\Psi^2$, and $\phi(\bbP),\phi(\bbQ)$  the characteristic functions of probability measures $\bbP,\bbQ\in\cP(\cX)$.

\begin{theorem}\label{thm:sepa_bound}
Let $\{X_j\}_{j\geq 1},\{Y_j\}_{j\geq 1}$ be i.i.d.\ sequences drawn from $\bbP, \bbQ \in \cP(\cX)$, respectively. 
Take $\gamma_n \geq C^2n^{-\sfrac{(1-\beta)}{4}}$ for some $C>0$ and $\beta \in(0,1)$.
Fix $\alpha\in(0,1)$ and $\delta\in(0,1)$.
Consider the decision function:
\begin{equation}\label{eq:stu-test}
        \psi_n = \begin{cases}
            1\quad \textnormal{if} \quad \test^{\textnormal{KL}}_{\gamma_{n}}(\bbP_n,\bbQ_n) \geq 
            u_{\alpha, n}\\
            0\quad \textnormal{otherwise.}
        \end{cases}
\end{equation}
with threshold:
\begin{equation*}
    u_{\alpha, n} := n^{-\sfrac{\beta}{2}} 6K\left(  {4K^{1/2}} +  KC\sqrt{\log \alpha^{-1} } \right).
\end{equation*}
Then
\begin{align*}
    \sup_{(\bbP,\bbQ)\in\cP_{\Delta}^{\Lambda} } E_{H_0}[\psi_{n}(\bbP_n,\bbQ_n)]\leq \alpha,
    \qquad\text{and}\qquad \inf_{(\bbP,\bbQ)\in\cP_{\Delta}^{\Lambda} }    E_{H_1}[\psi_{n}(\bbP_n,\bbQ_n)]\geq   1 - \delta,
\end{align*}
for sample sizes
$$
n \geq  \max \left\lbrace\left( \frac{72K^3\left(4K^{1/2} + KC\sqrt{\log\frac{1}{\alpha}} \right)}{{\Delta^2} } \right)^{2/\beta}, \:
 \frac{144K^4}{\Delta^4}\log (4/\delta) \right\rbrace.
$$
\end{theorem}

In more interpretable terms, the previous result states that, if regularisation decays sufficiently slowly, then, for sufficiently large sample sizes, the test \eqref{eq:stu-test} achieves power arbitrarily close to $1$, uniformly over separated alternatives, while maintaining the level controlled.
More precisely, if regularisation decays
\textit{not faster then } $n^{-\sfrac{(1-\beta)}{4}}$
for $\beta \in(0,1)$, then, there exist universal constants $C_1, C_2>0$ such that for all  $\alpha\in(0,1)$ and $\delta\in(0,1)$, for all sufficiently large sample sizes
$$
n \geq  C_1\left(\Delta^{-2}(1 + \log\alpha^{-1} + \log\delta^{-1})\right)^{2/\beta},
$$
the regularised kernel likelihood ratio test achieves power ($1 -$ Type II error) greater then $1-\delta$, whilst maintaining the level (Type I error) controlled at $\alpha$, \textit{uniformly} over $\Delta$-separated alternatives $\bbP,\bbQ\in $, with decision threshold decaying like $
C_2 \sqrt{\log\frac{1}{\alpha}}n^{-\sfrac{\beta}{2}}
$.

\medskip

The proof relies on two technical results. The first is an upper bound on the choice of valid threshold for a test at level $\alpha\in(0,1)$, controlling the tail of the distribution of $\test_\gamma(\bbP,\bbQ)$ uniformly over measures $\bbP,\bbQ$ (Lemma~\ref{lem:boundonthresh}, in the appendix). While this bound ensures type I error control uniformly over all distributions, it may be suboptimal in practice. For this reason, we typically calibrate the decision rule via permutation (or bootstrap) procedures in simulations and applications. Nevertheless, the bound is a crucial building block in the theoretical analysis of the separation boundary. The second technical result is a \textit{lower bound} for the test statistic $\test_\gamma(\bbP,\bbQ)$, when $\bbP\neq\bbQ$, independently of regularisation $\gamma$ (Lemma~\ref{lem:lowerbound} in the Appendix).

\begin{remark}
    The rates in Theorem~\ref{thm:sepa_bound} suggest that to maximise power, one should  take $\beta$ as large as possible, $\beta = 1$, in which case the decision threshold decays as $n^{-1/2}$, and power grows proportional to $\Delta^{-4}$, relative to sample size. However, such a choice implies that the regularisation does not vanish asymptotically. This is not a contradiction, but rather a reflection of the fact that our consistency proof does not rely on divergence of the statistic under the alternative. Instead, we employ a uniform lower bound under $H_1$, and a vanishing upper bound under $H_0$.
From the perspective of uniform consistency, this is indeed sufficient, and is an artifact of the proof technique than of the actual behaviour of the test. Although $\test_\gamma(\bbP,\bbQ)$ diverges as $\gamma \to 0$ when $\bbP \ne \bbQ$, capturing this blow-up quantitatively in terms of $\gamma$ and a semimetric between $\bbP$ and $\bbQ$ appears challenging. Ideally, one would show that $\test_{\gamma}(\bbP,\bbQ) \geq L(\gamma) \rho(\bbP,\bbQ)$ with $L(\gamma)\to \infty$ as $\gamma\to 0$, but such a bound remains elusive. Lemma~\ref{lem:lowerbound} establishes a weaker bound, which fortunately suffices to establish consistency and the separation boundary, and manifests in the rates obtained
\end{remark}

In closing this section, we highlight that all the theoretical guarantees remain valid, mutatis mutandis, when the test statistic is constructed using the central (zero-mean) empirical Gaussian embeddings. In fact, as noted by \cite{santoro2025from}, it is the $\log\dettwo$ term in \eqref{eq:KRKL}—rather than the Mahalanobis distance—that is chiefly responsible for eliciting the singularity effect, which ultimately drives the power of our test. From a theoretical standpoint, both the central and uncentered variants lead to the same separation boundary under the asymptotic regime we study, though this does not imply identical power functions at finite sample sizes. In practice, as illustrated in our simulations, the two versions exhibit broadly similar performance, although differences may arise depending on the sample size and the nature of the underlying alternatives.

\subsection{Implementation}
Let $k(\cdot,\cdot)$ be a universal kernel on $\cX\times\cX$.
Given independent samples of $\cX$-valued random variables:
\begin{align*}
\mathsf{X} = \{X_{i}\}_{i=1}^{n} \sim \bbP \quad \textnormal{and}\quad \mathsf{Y} = \{Y_{j}\}_{j=1}^{m} \sim \bbQ,
\end{align*}
inducing empirical distributions $\bbP_n$ and $\bbQ_m$, we define \emph{the empirical kernel matrices}:
$$
K_{xx} = \big\{k(X_i, X_j)\big\}_{i, j = 1}^{n, m}, \quad
K_{xy} = \big\{k(X_i, Y_j)\big\}_{i, j = 1}^{n, m}
$$
$$
K_{yx} = \big\{k(Y_i, X_j)\big\}_{i, j = 1}^{n, m}, \quad
K_{xy} = \big\{k(X_i, Y_j)\big\}_{i, j = 1}^{n, m}
$$
Writing $Z_k = X_k$ for $k = 1, \dots, n$, and $Y_{k-n}$ for $k = n+1, \dots, n+m$, for the elements of the pooled sample, one can straightforwardly verify that:
\begin{align*}
    \langle \m_{\bbP_n}, k_{Z_i} \rangle 
            = \int_{\mathcal{X}} \langle k_{Z_i}, k_u\rangle \, d\bbP_n(u) 
         = \frac{1}{n}\sum_{\ell = 1} ^n  \langle k_{Z_i}, k_{X_\ell}\rangle
         = \frac{1}{n}\sum_{\ell = 1} ^n  k(Z_i, X_\ell).
\end{align*}
Therefore, a coordinate representation of the mean embeddings $\m_{\hat \bbP_n}, \m_{\hat \bbQ_m}$ can be expressed with respect to the system $\mathfrak{B} = \{k_{X_1}, \dots, k_{X_n},k_{Y_1},\dots, k_{Y_m}\}$ in terms of row-wise sums of the kernel matrices
\begin{align*}
    m_{\mathsf{X}} = \frac{1}{n}\begin{bmatrix}
        K_{xx}\mathbf{1}_{n} \\
        K_{xy}\mathbf{1}_{n}
    \end{bmatrix} \qquad \textnormal{and} \qquad 
    m_{\mathsf{Y}} = \frac{1}{m}\begin{bmatrix}
        K_{yx}\mathbf{1}_{m} \\
        K_{yy}\mathbf{1}_{m},
    \end{bmatrix}
\end{align*}
respectively. Here,  $\mathbf{1}_{m} = \left[1, 1,\dots,1\right]^\top\in \bbR^m$, $\mathbf{1}_n = \left[ 1,1,\dots,1\right]^\top\in\bbR^n$ is the vector of ones. 
Similarly, 
\begin{align*}
    \langle k_{Z_i}, \S_{\bbP_n} k_{Z_j}\rangle 
        = \int_{\mathcal{X}} \langle k_{Z_i}, k_u\rangle \langle k_{Z_j}, k_u\rangle \, d \bbP_n(u) 
         = \frac{1}{n}\sum_{\ell = 1} ^n  \langle k_{Z_i}, k_{X_\ell}\rangle \langle k_{Z_j}, k_{X_\ell}\rangle
         = \frac{1}{n}\sum_{\ell = 1} ^n  k(Z_i, X_\ell)k({X_\ell}, {Z_j}),
\end{align*}
with an analogous expression for $\S_{\bbQ_m}$.
Hence, we see that the actions of $\S_{\bbP_n} and \S_{\bbQ_m}$ on the system given by $\mathfrak{B} = \{k_{X_1}, \dots, k_{X_n},k_{Y_1},\dots, k_{Y_m}\}$ can be respectively expressed in terms of the {Gram matrices}
\begin{equation}
\label{eq:gram}
\begin{split}
&S_{\mathsf{X}} = 
\frac{1}{n}\begin{bmatrix}
K_{xx}^{\phantom{\top}}K_{xx}^\top & K_{xx}^{\phantom{\top}}K_{yx}^\top \\
K_{yx}^{\phantom{\top}}K_{xx}^\top & K_{yx}^{\phantom{\top}}K_{yx}^\top
\end{bmatrix} \qquad \textnormal{and} \qquad 
S_{\mathsf{Y}} = \,
\frac{1}{m} \, \begin{bmatrix}
K_{xy}^{\phantom{\top}}K_{xy}^\top & K_{xy}^{\phantom{\top}}K_{yy}^\top \\
K_{yy}^{\phantom{\top}}K_{xy}^\top & K_{yy}^{\phantom{\top}}K_{yy}^\top
\end{bmatrix}.
\end{split}
\end{equation}
For $\gamma>0$ we write
$
S_{\mathsf{X},\gamma}  = S_{\mathsf{X}} + \gamma I
$ and $ 
S_{\mathsf{Y},\gamma}  = S_{\mathsf{Y}} + \gamma I,
$
where $I$ denotes the $(n+m)$ identity matrix. 

In summary, given a Mercer kernel $k(\cdot,\cdot)$ and a ridge parameter $\gamma$, the regularised likelihood-ratio test statistic \eqref{eq:ker_regularized_KL} involving $\bbP_n,\bbQ_n$, can be expressed in closed form as:
\begin{align*}
T_{\gamma, k} = T_{\gamma, k}(\{X_{i}\}_{i=1}^{n}, \{Y_{j}\}_{j=1}^{m}) = 
        \left\lVert S_{\mathsf{X},\gamma} ^{-\sfrac{1}{2}}(m_{\mathsf{Y}} - m_{\mathsf{X}}) \right\rVert^2  + \trace\left(\log\left(S_{\mathsf{X},\gamma} ^{-\sfrac{1}{2}}S_{\mathsf{Y},\gamma} S_{\mathsf{X},\gamma} ^{-\sfrac{1}{2}}\right) - S_{\mathsf{Y},\gamma} S_{\mathsf{X},\gamma} ^{-1} + I\right)
\end{align*}
where we have simplified the expression employing orthogonality and the cyclic property of the trace.

\paragraph{Permutation test.}

As discussed in Section \ref{sec:the_test}, the test statistic will be calibrated by way of permutation: we consider the distribution of the test statistic computed over all label permutations of the pooled sample,
$$
\{T(\{Z_{\sigma(i)}\}_{i=1}^{n}, \{Z_{\sigma(j)}\}_{j=n+1}^{n+m}): \sigma \in \mathfrak{S}_{n+m}\},
$$
where $\mathfrak{S}_{n+m}$ is the symmetric group of permutations of $[n+m] = \{1, \dots, n+m\}$. In practice, this distribution is approximated by sampling $B$ permutations $\sigma_1, \dots, \sigma_B$ uniformly at random from $\mathfrak{S}_{n+m}$, and computing the empirical $(1 - \alpha)$-quantile of the resulting values. This sample yields an approximation $\hat{q}_{1-\alpha}^{B,\gamma}$  to the quantile defined in \eqref{decision:kl}, namely
$$
\hat{q}_{1-\alpha}^{B,\gamma} := \text{empirical } (1-\alpha)\text{-quantile of } \left\{ T(\{Z_{\sigma_b(i)}\}_{i=1}^{n}, \{Z_{\sigma_b(j)}\}_{j=n+1}^{n+m}) \right\}_{b=1}^{B}.
$$
\begin{remark}
    It is known 
    that $\hat{q}_{1-\alpha}^{B,\gamma}$ concentrates around the true permutation quantile $q^{\gamma}_{1-\alpha}$, and as the sample size grows, this converges to the $(1-\alpha)$-quantile of the limiting null distribution—typically a certain mixture distribution reflecting the combined sampling variability under the null. That is, if $\hat{q}^{B,\gamma}_{1-\alpha}$ is obtained drawing $B$ permutations at random, for $\alpha>0, \tilde{\alpha}>0, \delta>0$, if $B \geq \frac{1}{2 \tilde{\alpha}^2} \log 2 \delta^{-1}$: 
    $$
    P_\pi\left(\hat{q}_{1-\alpha}^{B, \gamma} \geq q^{\gamma}_{1-\alpha-\tilde{\alpha}}\right) \geq 1-\delta
    \qquad\textnormal{and}\qquad
    P_\pi\left(\hat{q}_{1-\alpha}^{B, \gamma} \leq q^{\gamma}_{1-\alpha+\tilde{\alpha}}\right) \geq 1-\delta.
    $$
    For a proof, see \cite[][Lemma~15]{hagrass2024spectral}.
\end{remark}
Our decision rule at approximate level $\alpha$ is then defined as:
\begin{equation}\label{eq:decision}
\psi_{k, \gamma, B}(\{X_{i}\}_{i=1}^{n}, \{Y_{j}\}_{j=1}^{m}) = 
\begin{cases} 
1 & \textnormal{if } T(\{X_{i}\}_{i=1}^{n}, \{Y_{j}\}_{j=1}^{m}) > \hat{q}^{B,\gamma}_{1-\alpha}, \\ 
0 & \textnormal{otherwise}.
\end{cases}
\end{equation}
which in effect rejects  the null hypothesis whenever the observed test statistic exceeds at least a proportion of $1-\alpha$ of the permuted statistics, obtained drawing $B$ samples uniformly at random from $\mathfrak{S}_{n+m}$.

\medskip

\paragraph{Computational cost.}
The computational cost of the testing procedure stems from: matrix inversion, the computation of the test statistic, and the calibration through permutations. For a pooled sample of size $N = m + n$, the construction of the Gram matrices $S_{\mathsf{X}}$ and $S_{\mathsf{Y}}$ requires $O(N^2)$ kernel evaluations. Computing the test statistic involves matrix multiplications and the spectral norm, with a computational complexity of approximately $O(N^3)$ due to the inversion of the Gram matrices. Since the test threshold is determined via permutation testing, $B$ permutations require repeating the entire computation $B$ times, resulting in a total cost of $O(BN^3)$.  The computational burden increases significantly for large sample sizes $N$ and/or a high number of permutations $B$, making parallelization critical for large-scale applications.

 \begin{remark}
     The cubic dependence on total sample size is a common feature of kernel based tests involving inversion \cite{hagrass2024spectral,eric2007testing}. However, \cite{hagrass2024spectral} observed that computational load can be reduced by way of sample splitting, and separate estimation of the embedded covariance operator and the embedded mean element estimators. That is, based on  samples $\{X_{i}\}_{i=1}^{n}, \{Y_{j}\}_{j=1}^{m}$, one can split the samples into $\left(X_i\right)_{i=1}^{N-s}$ and $\left(X_i^1\right)_{i=1}^s:=\left(X_i\right)_{i=N-s+1}^N$, and $\left(Y_j\right)_{j=1}^M$ to $\left(Y_j\right)_{j=1}^{M-s}$ and $\left(Y_j^1\right)_{j=1}^s:=\left(Y_j\right)_{j=M-s+1}^M$. Then, the samples $\left(X_i^1\right)_{i=1}^s$ and $\left(Y_j^1\right)_{j=1}^s$ are  used to estimate the covariance embeddings $\S_{\bbP}$ and $\S_{\bbQ}$, respectively,  while $\left(X_i\right)_{i=1}^{N-s}$ and $\left(Y_i\right)_{i=1}^{M-s}$ are used to estimate the mean elements $\m_{\bbP}$ and $\m_{\bbQ}$, respectively. Effectively, this reduces the complexity to from $O(BN^3)$ to $O(BNs^2)$, so that suitably  picking $s =\Omega(\sqrt{N})$
yields to comparable complexity to
that of the MMD test.
 \end{remark}

\paragraph{Adaptive bandwidth and regularisation}
   Implementing the test requires the choice of a kernel function $k(\cdot,\cdot)$ and a regularisation parameter $\gamma>0$. In theory, the kernel can be any universal kernel, whereas the regularisation parameter needs to satisfy $\gamma^4 \min\{n,m\}\to\infty$ as $n,m\to\infty$. In practice, the precise choice of $k$ and $\gamma$ can influence finite sample performance. In particular:

   \medskip
   \begin{enumerate}
       \item The choice of the kernel function defines the features through which differences between $\mathbb{P}$ and $\mathbb{Q}$ are measured. Popular kernels include, polynomial, Gaussian, and Laplacian kernels. We employ the latter two in our simulations. Often, a kernel comes with the choice of an additional \emph{bandwidth} or \emph{concentration} parameter $\sigma$: for instance for kernels of the form $k_{\rho, \sigma}(x, y) = \exp\left(-\sigma^{-2}\rho(x, y)/2\right),$
    for some semi-metric $\rho$. Suitable tuning of the kernel, i.e.\ the choice of $\sigma$, is crucial to optimising power in finite sample sizes. Too large a bandwidth can smooth over meaningful differences, while too small a bandwidth can amplify noise.  A common rule for adaptively choosing $\sigma$, is set it in proportion to the median interpoint distance in the pooled sample. 
     This is the rule we employ in our simulations and analyses, and is seen to perform well, though we remark that this is one of several other possibilities \cite{sutherland2016generative, chatterjee2024boosting, liu2020learning}.

    \item
    The choice of regularization strength $\gamma$ regulates a tradeoff between finite-sample stability and making use of the separation of measure phenomenon. To ensure maximal gains in power, our proposal is to consider testing across a range of values $\gamma$, and adjusting for multiple testing. This is reminiscent of the idea of adapting to a family of kernels, that has been explored in regression and classification settings under the name \textit{multiple-kernel learning} \cite{gonen2011multiple}. This approach is seen to yield very good performance -- the gains in power to be made by adaptive choice seem to clearly outweigh the cost of the multiple testing correction.
   \end{enumerate}

    In summary, we implement an adaptive selection procedure as in \cite{schrab2023mmd, hagrass2024spectral}
    and consider jointly aggregating over bandwidths and ridge regularization parameters, with Bonferroni correction to control Type I error. That is, fix a parametric family of kernels $k_h(\cdot,\cdot)$ indexed by bandwidth parameter $h>0$, and set a finite sequences $\cK = \{h_1,\dots,h_{|\cK|}\}$ and $\Lambda = \{\lambda_1, \dots, \lambda_{|\Lambda|}\}$ of kernel bandwidths and regularisation parameters. Denote by
    ${T}_{\lambda, h}$ the test statistic based on kernel $k_h$ and regularization parameter $\lambda$. We reject $H_0$ if ${T}_{\lambda, h} \geq  \hat{q}^{B,\gamma}_{1-\frac{\alpha}{|\Lambda||K|}}$ for any $(\lambda, h) \in \Lambda \times \cK$. 

\section{Data analyses}

In this section, we investigate the empirical performance of our proposed two-sample test. 
 We explore the performance of the likelihood ratio statistic based both on the \textit{central} (mean-zero)  and the \textit{non-central}  empirical Gaussian embeddings, 
which we denote by $\texttt{KLR0}$ and $\texttt{KLR}$, respectively, in our summaries. For benchmarking purposes, we compare with several existing approaches: the adaptive Maximum Mean Discrepancy test (\texttt{AggMMD})  \cite{schrab2023mmd}, the spectral regularised Maximum Mean Discrepancy test (\texttt{SpecRegMMD}) \cite{hagrass2024spectral}, the K-nearest neighbours (\texttt{KNN}) statistic \cite{schilling1986multivariate},  the Friedman and Rafsky (\texttt{FR}) test \cite{friedman1979multivariate}, and the Hall and Tajvidi (\texttt{HT}) test \cite{HallTajvidi02}.

The adaptive MMD test \cite{schrab2023mmd} requires using a translation invariant kernel and considers multiple bandwidth parameters $h$, correcting for multiple testing: multiple such tests are constructed over $h$, and are subsequently aggregated to achieve adaptivity.  The resulting test is referred to as \texttt{AggMMD}. The spectrally regularised MMD test\cite{hagrass2024spectral}, similarly to our method, involves adaptively selecting both a bandwidth parameter and a regularisation parameter $\gamma$. Multiple tests are constructed over $h,\gamma$, which are aggregated to achieve adaptivity, and the resultant test is referred to as \texttt{SpecRegMMD}. 

\medskip

We will consider different experimental setups using either a Laplacian kernel, $k(x, y)=\exp \left(-\frac{\|x-y\|_2}{ h}\right)$ or a Gaussian kernel $k(x, y)=\exp \left(-\frac{\|x-y\|_2^2}{2 h}\right)$, with $h$ being the bandwidth. The significance level is always fixed to $\alpha=0.05$. For our test, we construct an adaptive test by taking the union of tests jointly over regularisation parameters $\gamma \in \Lambda$ and bandwidth parameters $h \in \cK$, and correct for multiple testing with the Bonferroni correction to preserve the level. We consider:
$$\Lambda:=\left\{10^{-7}, 10^{-6}, \hdots, 10^{-1}\right\}
\quad \text{and}\quad 
\cK:=\left\{ \frac{h_m}{50}, \frac{h_m}{10}, \frac{h_m}{5}, h_m, 5\cdot h_m, 10\cdot h_m\right\}
$$
where $h_m:= \operatorname{median} \left\{\left\|q-q^{\prime}\right\|_2: q, q^{\prime} \in X \cup Y, \,{q\neq q'}\right\}$.  

\medskip

All relevant code can be accessed at:
\begin{center}
    \fbox{\texttt{https://github.com/leonardoVsantoro/Kernel-Likelihood-Ratio-Two-Sample-Test}.}
    \end{center}
As for the test cases, we consider both synthetic and real data sets, in the next two sections. Overall, it is found that the proposed method's empirical performance exhibiting substantial power improvements compared to existing approaches, and excels in challenging scenarios involving subtle and/or high-dimensional differences between distributions.

\subsection{Synthetic Data}

To evaluate the performance of the statistical tests across a range of sample sizes and dimensionalities,  we shall consider eight distinct models and modes of perturbation, each designed to probe different forms of distributional differences between the null and alternative hypotheses:
\begin{enumerate}[label=\textbf{Model \arabic*.}, leftmargin=5em]
    \item Isotropic Gaussian distributions with a sparse shift in the mean vector.
    \item Product Laplace distributions with a sparse shift in the location parameter.
    \item A symmetric mixture of two Gaussians, each shifted in a sparse subset of coordinates, contrasted with a standard isotropic Gaussian.
    \item Isotropic Gaussian distributions with a sparse change in variances along few coordinates.
    \item Central Gaussian distributions with long-range correlations decaying as a power law.
    \item Central Gaussian distributions with equicorrelation structure, differing only by a small change in the correlation parameter.
    \item Uniform distributions on hypercubes, with the alternative having ``smaller support" on a sparse subset of coordinates.
    \item Uniform distributions on concentric hyperspheres with differing radii.
\end{enumerate}

\medskip

In each case, we have a different form of perturbation going from the null to the alternative, in each case under controlled sparsity. For instance, in Models 1--3, differences between $\bbP$ and $\bbQ$ arise from shifts in the mean vector, confined to a small number of coordinates. In Models 4--6, the distinction lies in changes to the covariance structure, either via variance inflation in a subset of coordinates, power-law decay of correlations, or equicorrelation. Model 7 explores changes in the support of a uniform distribution over $[0,1]^d$, with the support being compressed in the first $P$ coordinates under the alternative. Similarly, Model 8 compares distributions supported on hyperspheres of different radii (support change in all the coordinates).

\medskip

In Models 1-5 and 7, the signal distinguishing $\bbP$ from $\bbQ$ is confined to a small and fixed number of coordinates, relative to the ambient dimension $d$. As a result, the distinction between the two distributions becomes increasingly subtle in high dimensions -- a needle in a haystack problem-- posing a challenge for high-dimensional two-sample testing procedures \cite{ramdas2015decreasing}. Thus, we expect power to decay as dimension grows, regardless of which test is employed -- the slower the decay, the better. To the contrary, in Models 5,6 the difference is controlled by a parameter $\varepsilon$, where larger values of $\varepsilon$ magnify the difference between  $\bbP$ from $\bbQ$ , which reflects in increasing power for increasing $\varepsilon$. In these scenarios, we hope to see fast gains in power as $\varepsilon$ grows.

   \medskip
 We consider 200 replications within each scenario, and report the proportion of rejections: 
Figures~\ref{fig:GaussianSparseMeanShift}-\ref{fig:ConcentricSpheres} show the empirically observed percentage of rejections in the 8 models under the alternative, with different sample sizes, dimensions, and hyperparameters. 
All tests are calibrated by permutation, drawing $B=300$ random permutations. 
 The average Type I error for the proposed tests \texttt{KLR} and \texttt{KLR0} is shown in Figure~\ref{fig:levelKLR} and Figure~\ref{fig:levelKLR0}, respectively.
We consider a Gaussian kernel in the first 4 models, and a laplacian kernel in the latter 4.

\medskip

\medskip

Across all eight synthetic models, both \texttt{KLR} and \texttt{KLR-0} perform at least as well as the best competing tests, and in many cases substantially outperform them. The gains are particularly striking in high-dimensional regimes. For instance, in Model 4 with $d = 1500$, our test achieves nearly perfect power ($\approx 1.0$) while the strongest competitor remains at less than half that ($\approx 0.4$). Similar findings appear in Models 3 and 6, with other methods degrading sharply in low-signal regimes, while \texttt{KLR}  and \texttt{KLR-0} attaining near perfect power.  The overall pattern emerging in our simulations is that our proposed test(s) maintains high power even as dimensionality and distributional complexity increase, when alternative methods exhibit rapidly collapsing power performance.

\noindent\begin{figure}[h!]
\begin{minipage}{.39\textwidth}
\begin{center}
\fbox{\begin{minipage}{.95\textwidth}
\textbf{Model 1 $(\Delta, P)$}
\begin{align*}
    &\bbP \equiv \cN(0, I_d), \qquad \textnormal{and}\\
    &\bbQ \equiv \cN(m_{\Delta, P}, I_d)
\end{align*} 
where $m_{\Delta,P}
= (m_1,\dots,m_d)$ is given by:
\begin{equation}
\label{sims_mean}  
\begin{split}
m_j = 
\begin{cases} 
\Delta, & 1\leq j \leq P \\
0, & P < j \leq d.
\end{cases}
\end{split}
\end{equation}
\end{minipage}}
\end{center}
\end{minipage}
\begin{minipage}{.6\textwidth}
    \centering
    \includegraphics[width=0.9\linewidth]{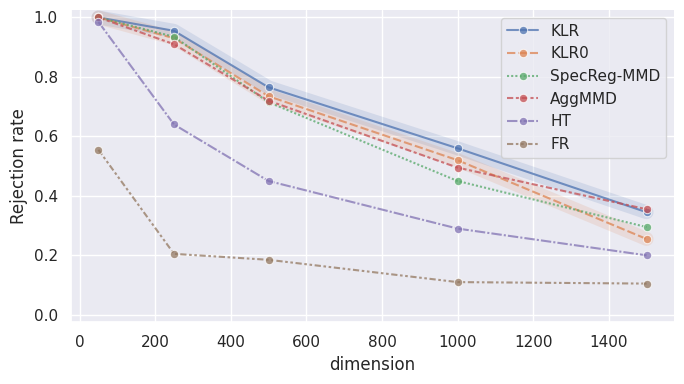}
\end{minipage}
    \noindent
    \includegraphics[width=.97\linewidth]{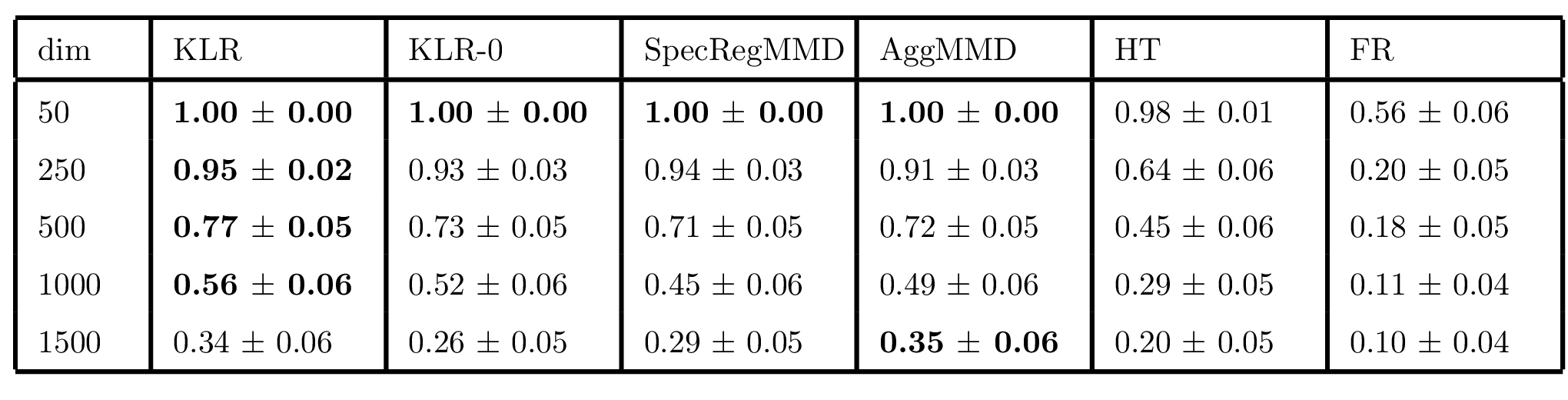}
    \caption{Average rejection rate in Model 1 $(\Delta =1, P =2)$  over 250 simulated experiments for increasing dimension and sample sizes $n,m = 100$. 
    }
    \label{fig:GaussianSparseMeanShift}
\end{figure}
\noindent\begin{figure}[h!]
\begin{minipage}{.39\textwidth}
\begin{center}
\fbox{\begin{minipage}{.95\textwidth}
\textbf{Model 2 $(\Delta,P)$}
\begin{align*}
    &\bbP \equiv \textnormal{Laplace}(0, 1)^{\otimes d}, \qquad \textnormal{and}\\
    &\bbQ \equiv \textnormal{Laplace}(m_{j}, 1)^{\otimes d})
\end{align*} 
where $m_{j}$ is given by:
\begin{align*}
m_{j} = 
\begin{cases} 
\Delta, & 1\leq j \leq P\\
0, & P < j \leq d.
\end{cases}
\end{align*}
\end{minipage}}
\end{center}
\end{minipage}
\begin{minipage}{.6\textwidth}
    \centering
    \includegraphics[width=0.9\linewidth]{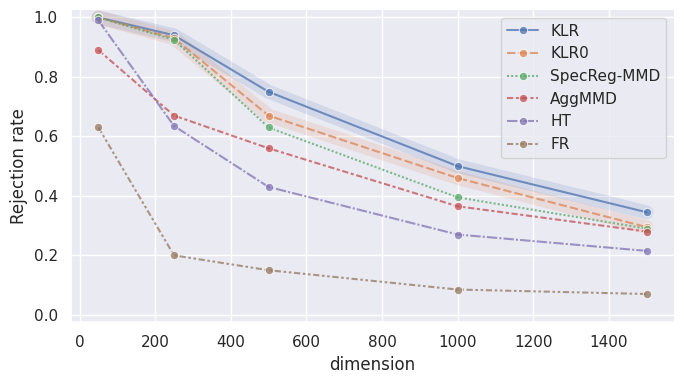}
\end{minipage}
    \noindent
    \includegraphics[width=.97\linewidth]{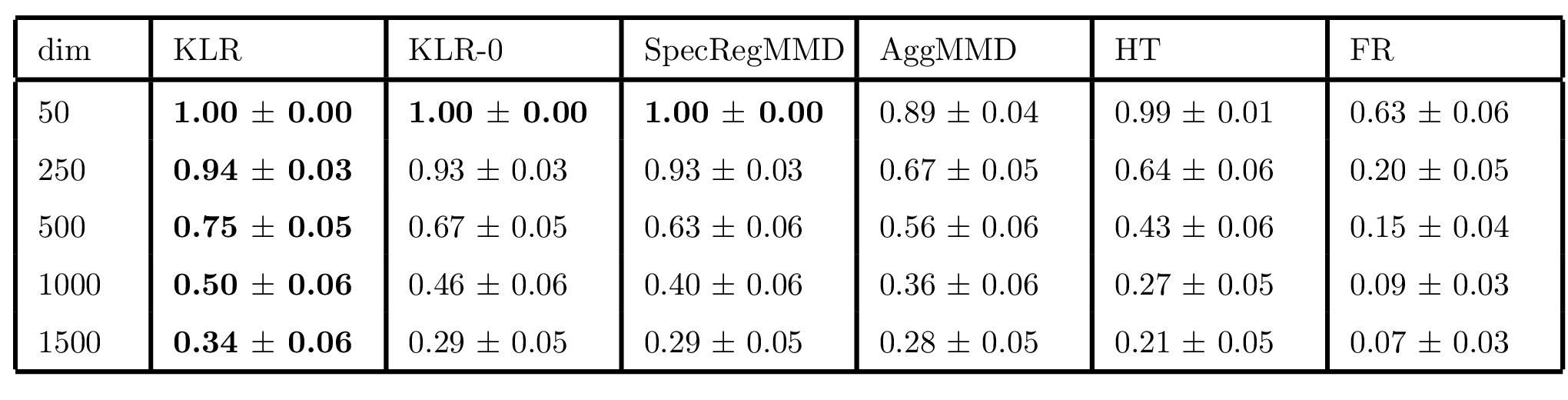}
    \caption{Average rejection rate in Model 2 $(\Delta = 1, P = 4)$  over 250 simulated experiments for increasing dimension and sample sizes $n,m = 100$. }    \label{fig:LaplaceSparseMeanShift}
\end{figure}
\noindent\begin{figure}[h!]
\begin{minipage}{.39\textwidth}
\begin{center}
\fbox{\begin{minipage}{.95\textwidth}
\textbf{Model 3 $(\Delta,P)$}
\begin{align*}
    &\bbP \equiv \cN(0, I_d), \qquad\textnormal{and}
    \\ &\bbQ \equiv \frac{1}{2}\cN(-m_{\Delta,P}, I_d) + \frac{1}{2}\cN(m_{\Delta,P}, I_d)
\end{align*}
where $m_{\Delta,P} = (m_1, \dots, m_d)$ is:
\[
m_j = 
\begin{cases}
\Delta, & 1\leq j \leq P\\
0, & P < j \leq d.
\end{cases}
\]
\end{minipage}}
\end{center}
\end{minipage}
\begin{minipage}{.6\textwidth}
    \centering
    \includegraphics[width=0.9\linewidth]{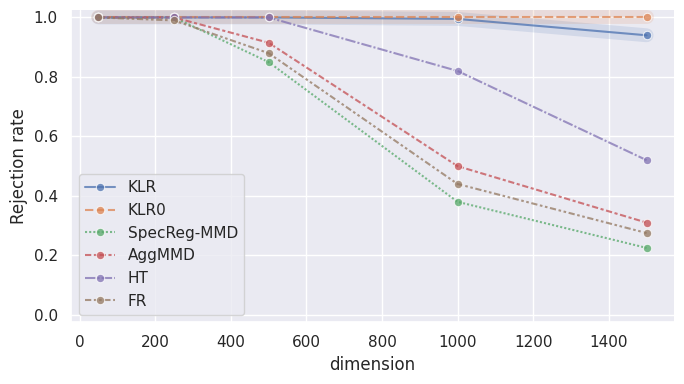}
\end{minipage}
\noindent\includegraphics[width=.97\linewidth]{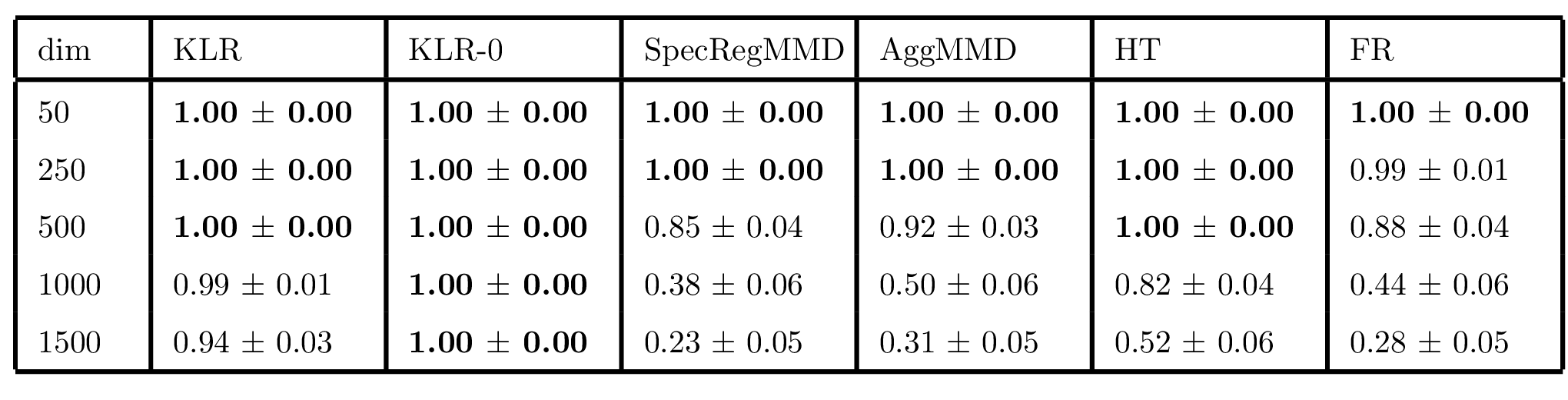}
    \caption{Average rejection rate in Model 3 $(\Delta = 4, P = 1)$  over 250 simulated experiments for increasing dimension and sample sizes $n,m = 100$. }
    \label{fig:enter-label}
\end{figure}
\noindent\begin{figure}[h!]
\begin{minipage}{.39\textwidth}
\begin{center}
\fbox{\begin{minipage}{.95\textwidth}
\textbf{Model 4} $(\lambda,P)$
$$
    \bbP \equiv \cN(0, I_d), 
    \qquad 
    \bbQ \equiv \cN(0, \Sigma_{P, \lambda})\\
$$
where $\Sigma_{P, \lambda} = \textnormal{diag}\left(s_1, s_2, \dots, s_d\right)$:
\begin{equation}
\label{sims:cov}
\begin{split}
s_j = 
\begin{cases} 
\lambda, & 1\leq j \leq P\\
1, & P < j \leq d.
\end{cases}
\end{split}
\end{equation}
\end{minipage}}
\end{center}
\end{minipage}
\begin{minipage}{.6\textwidth}
    \centering
    \includegraphics[width=0.9\linewidth]{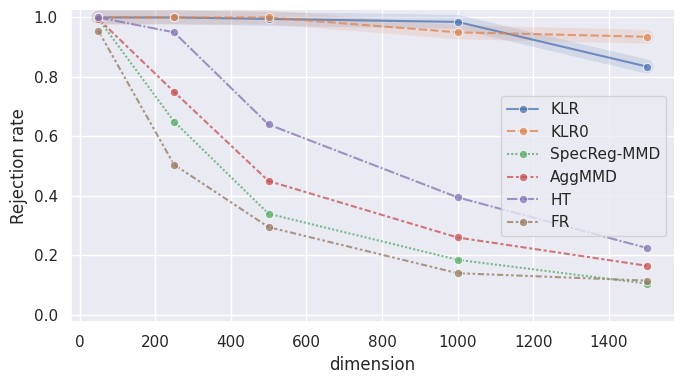}
\end{minipage}
    \noindent
    \includegraphics[width=.97\linewidth]{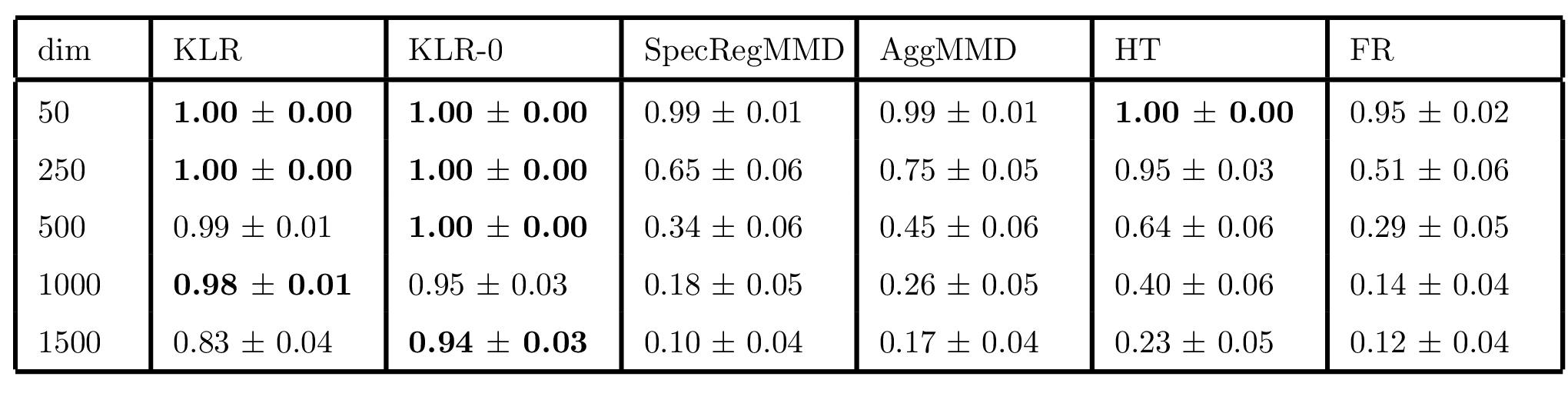}
    \caption{Average rejection rate in Model 4 $(\lambda = 3, P = 5)$  over 250 simulated experiments for increasing dimension and sample sizes $n,m = 100$. }    \label{fig:GaussianSpikedCovariance}    
\end{figure}
\noindent\begin{figure}[h!]
\begin{minipage}{.39\textwidth}
\begin{center}
\fbox{\begin{minipage}{.95\textwidth}
\textbf{Model 5 $(\alpha,\varepsilon)$}.
\begin{align*}
    &\bbP \equiv \cN(0, \Sigma_{\alpha}), \qquad \textnormal{and}\\
    &\bbQ \equiv \cN(0, \Sigma_{\alpha + \varepsilon})
\end{align*} 
where $\{\Sigma_{\alpha}\}_{ij} = (i-j)^{\alpha}$.
\end{minipage}}
\end{center}
\end{minipage}
\begin{minipage}{.6\textwidth}
    \centering
    \includegraphics[width=0.88\linewidth]{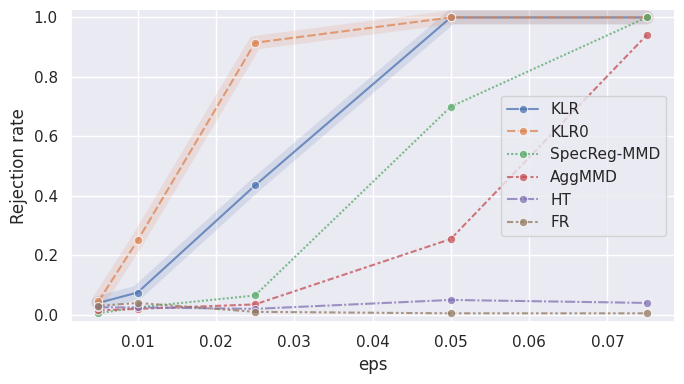}
\end{minipage}
    \noindent\includegraphics[width=.94
    \linewidth]{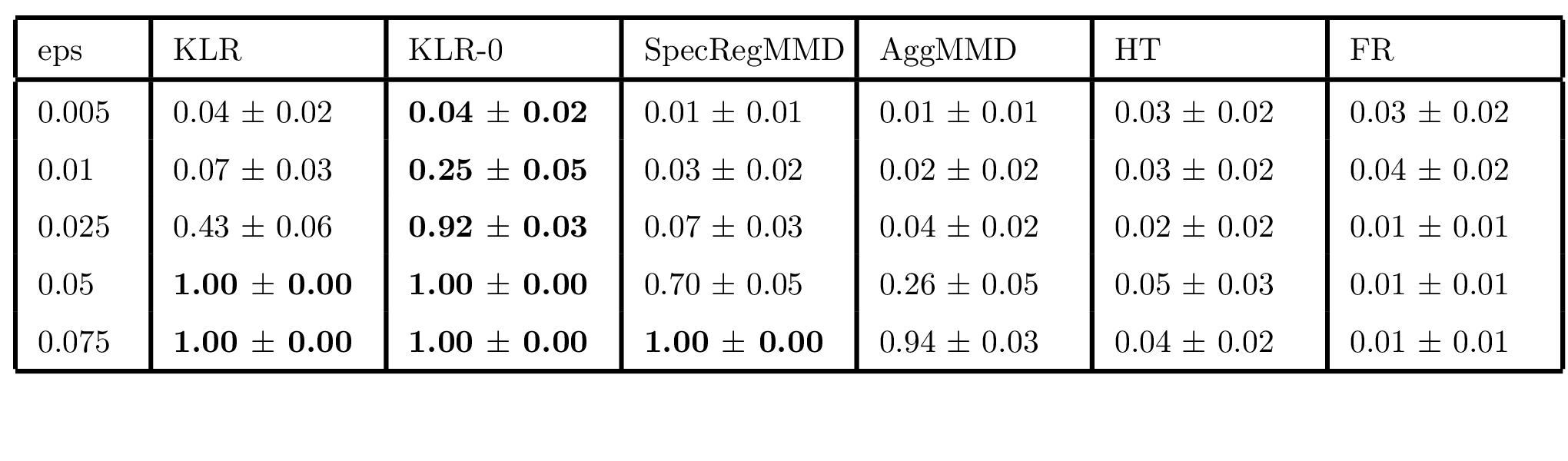}
    \caption{Average rejection rate in Model 5 $(\alpha= 0.5,\varepsilon)$, for increasing values of $\varepsilon$, $n,m = 100$ and $d=500$.
}
    \label{fig:EquiCorrelationGaussian}
\end{figure}
\noindent\begin{figure}[h!]
\begin{minipage}{.39\textwidth}
\begin{center}
\fbox{\begin{minipage}{.95\textwidth}
\textbf{Model 6 $(\alpha,\varepsilon)$}.
\begin{align*}
    &\bbP \equiv \cN(0, \Sigma_{\alpha}), \qquad \textnormal{and}\\
    &\bbQ \equiv \cN(0, \Sigma_{\alpha + \varepsilon})
\end{align*} 
where $\Sigma_{\alpha} = (1-\alpha)I_d + \alpha 1_d1_d^\top$.
\end{minipage}}
\end{center}
\end{minipage}
\begin{minipage}{.6\textwidth}
    \centering
    \includegraphics[width=0.88\linewidth]{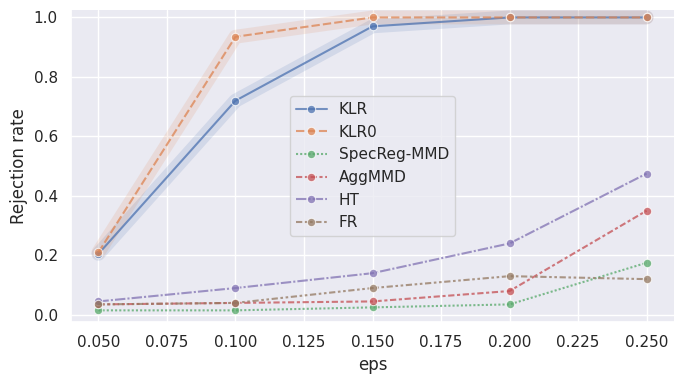}
\end{minipage}
\noindent\includegraphics[width=.94\linewidth]{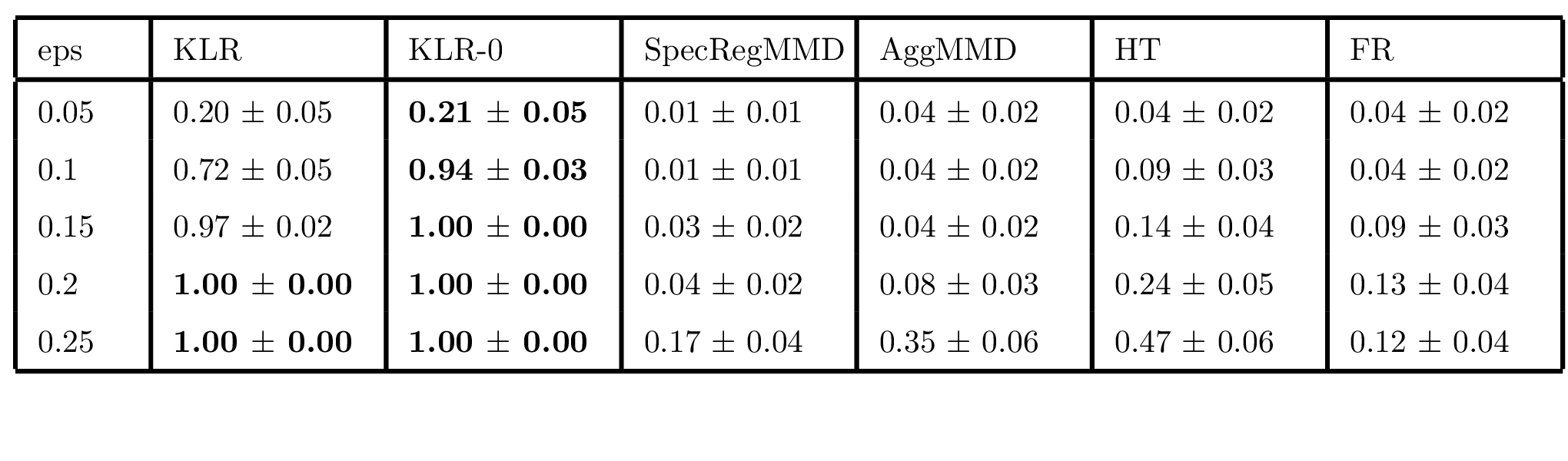}
\caption{Average rejection rate in Model 6 $(\alpha = 0.5,\varepsilon)$, for increasing values of $\varepsilon$, $n,m = 100$ and $d=500$.
}    \label{fig:DecreasingCorrelationGaussian}
\end{figure}
\noindent\begin{figure}[h!]
\begin{minipage}{.39\textwidth}
\begin{center}
\fbox{\begin{minipage}{.95\textwidth}
\textbf{Model 7} $(\varepsilon,P)$
\begin{align*}
    &\bbP \equiv \textnormal{Unif}([0,1]^d),, \qquad \textnormal{and}\\
    &\bbQ \equiv \textnormal{Unif}([0, 1-\varepsilon]^\Delta \times [0,1]^{d-P}
\end{align*} 
\end{minipage}}
\end{center}
\end{minipage}
\begin{minipage}{.6\textwidth}
    \centering
    \includegraphics[width=0.9\linewidth]{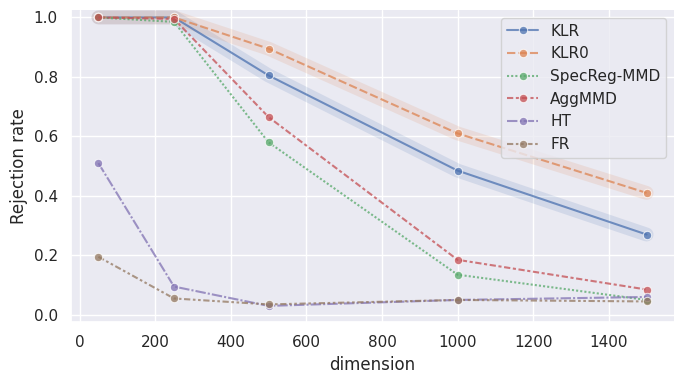}
\end{minipage}
    \noindent    \includegraphics[width=.97\linewidth]{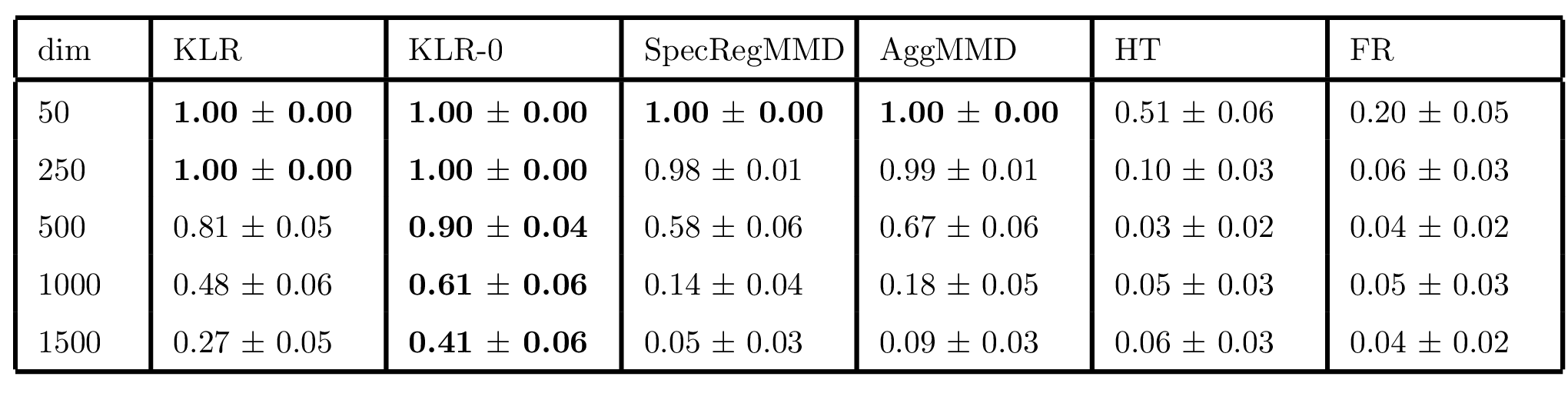}
    \caption{Average rejection rate in Model 6 $(\varepsilon = 0.02, P = 30)$, over 250 simulated experiments, for increasing dimension, and sample sizes $n,m = 100$.}
    \label{fig:UniformThinHypercube}
\end{figure}
\noindent\begin{figure}[h!]
\begin{minipage}{.39\textwidth}
\begin{center}
\fbox{\begin{minipage}{.95\textwidth}
\textbf{Model 8} $(\varepsilon)$
\begin{align*}
    &\bbP \equiv \textnormal{Unif}(S^{d-1}(1)), \qquad \textnormal{and}\\
    &\bbQ \equiv \textnormal{Unif}(S^{d-1}(1 + \varepsilon))
\end{align*} 
where $S^{d-1}(r) = \{x\in\bbR^d\::\: \|x\|=r\}$.
\end{minipage}}
\end{center}
\end{minipage}
\begin{minipage}{.6\textwidth}
    \centering
    \includegraphics[width=0.9\linewidth]{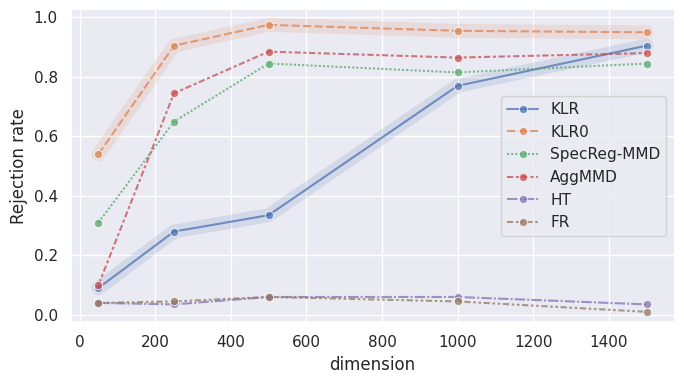}
\end{minipage}
    \noindent    \includegraphics[width=.97\linewidth]{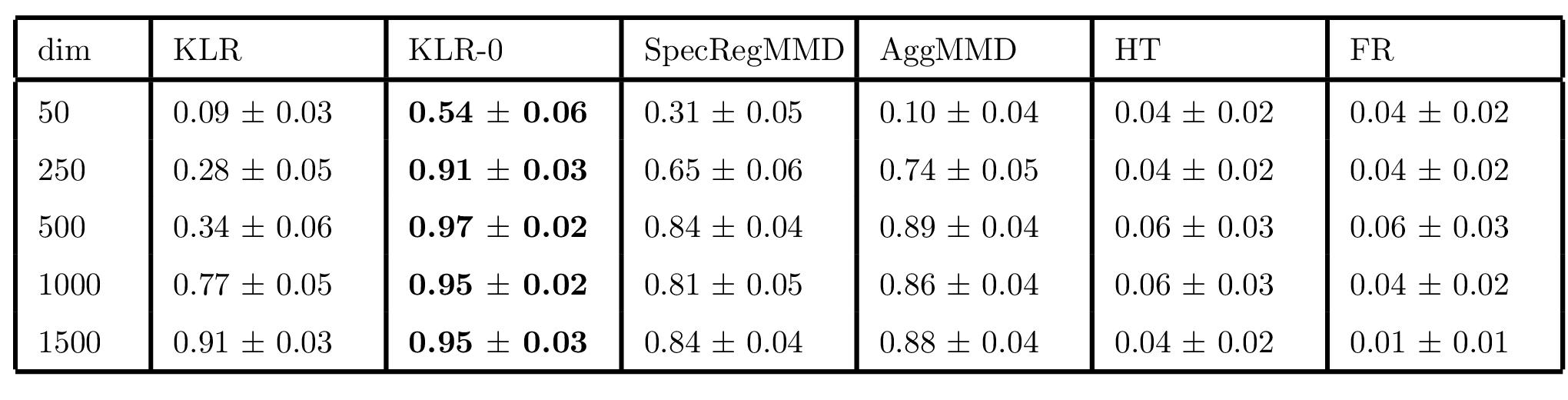}
    \caption{Average rejection rate in Model 8 $(\varepsilon = 0.02)$ over 250 simulated experiments, for increasing dimension, and sample sizes $n,m = 100$.}    \label{fig:ConcentricSpheres}
\end{figure}


\begin{figure}
\centering
\includegraphics[width=.95\textwidth]{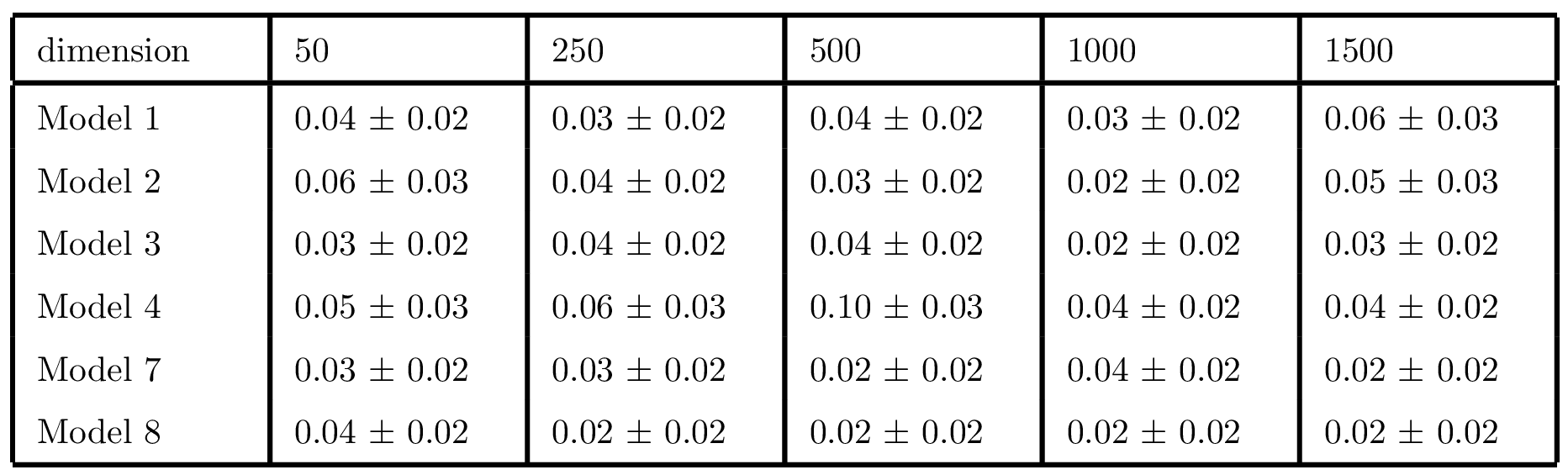}
\includegraphics[width=.95\textwidth]{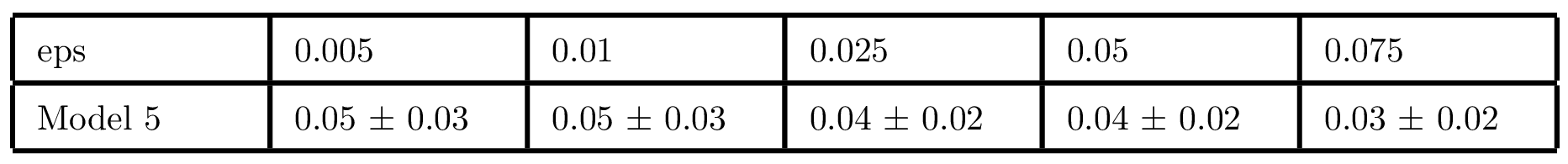}
\includegraphics[width=.95\textwidth]{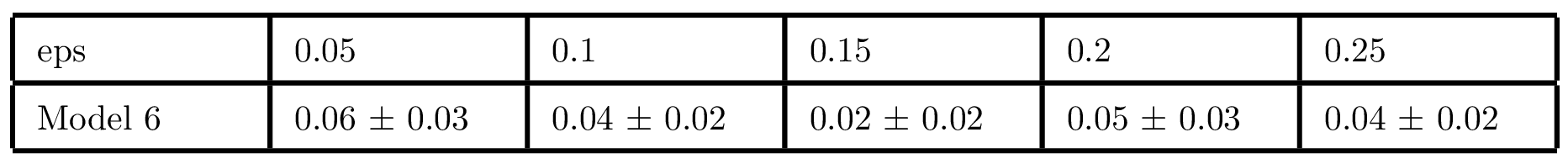} 
\caption{Average level values (rejection under the null hypothesis) drawing samples from the null across the models and setting considered, for the test statistic $\texttt{KLR}$.}
\label{fig:levelKLR}
\end{figure}

\begin{figure}
\centering
\includegraphics[width=.95\textwidth]{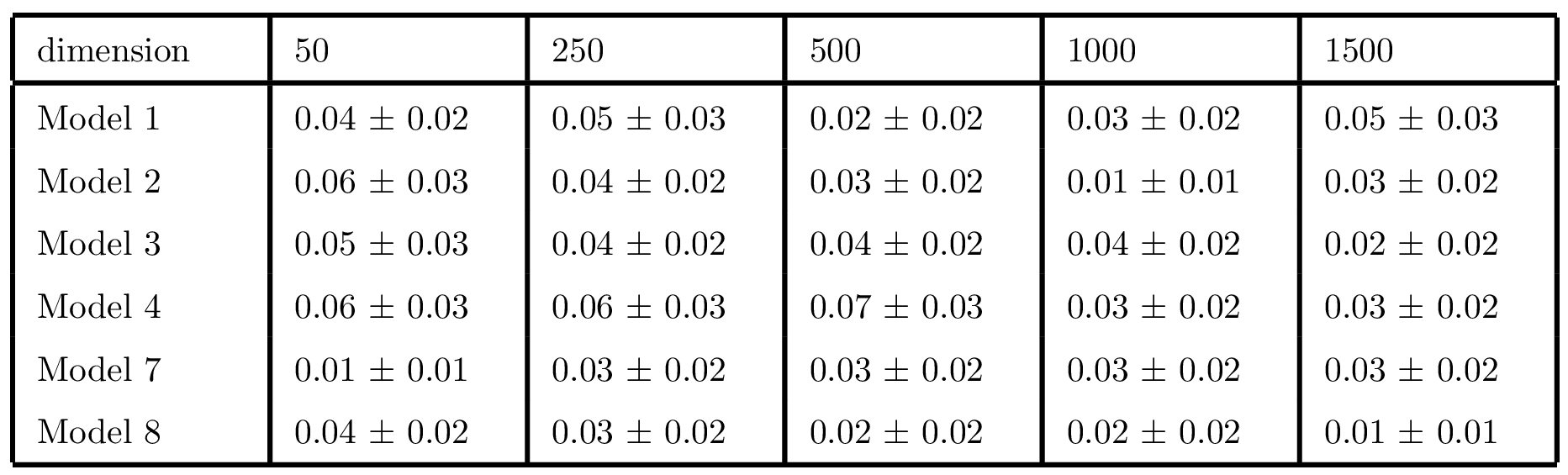}
\includegraphics[width=.95\textwidth]{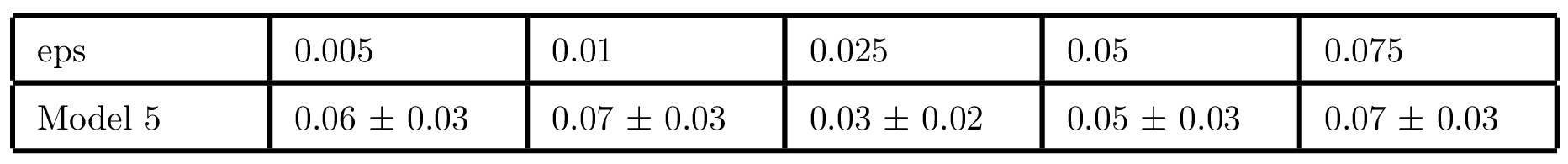}
\includegraphics[width=.95\textwidth]{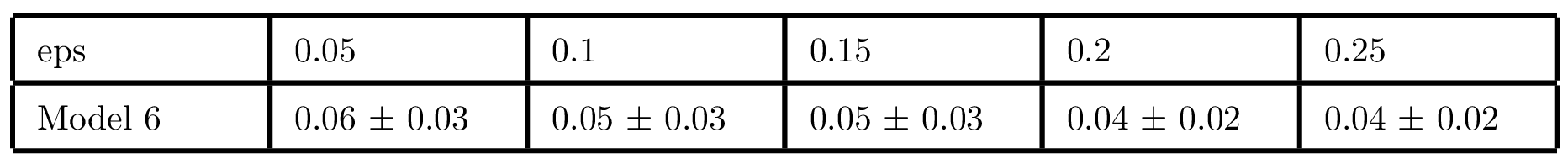} 
\caption{Average level values (rejection under the null hypothesis) drawing samples from the null, across the models and setting considered, for the test statistic $\texttt{KLR0}$.}
\label{fig:levelKLR0}
\end{figure}

\subsection{Real Data}


Further to the simulation presented in the previous section, we now illustrate the practical relevance of our testing procedure in the context of transfer learning and covariate shift detection. Specifically, we consider two benchmark image datasets commonly used in machine learning. The first dataset includes collections of real and synthetic images from the MNIST dataset \citep{lecun1998mnist,radford2015unsupervised}, and the problem is to assess the quality of generative models that produced the synthetic data -- whether we can distinguish the real from the synthetic generative mechanism. The second tests for distributional shifts between two test sets in the CIFAR-10 classification task \cite{krizhevsky2009learning,recht2019imagenet}. These examples represent qualitatively different scenarios. In both cases, our test exhibits strong empirical performance.

\paragraph{MNIST vs. DCGAN-generated MNIST.}
The MNIST dataset contains 70,000 handwritten digit images  \citep{lecun1998mnist}. We compare the distribution $\bbP$ of samples of true MNIST images to a distribution $\bbQ$ of samples by a pretrained DCGAN \citep{radford2015unsupervised}, trained to mimic the generation of handwritten digit samples in MNIST. Samples from both distributions are shown in Figure \ref{fig:mnist}.
This setting mimics a transfer learning scenario where a model trained on real data is exposed to synthetic inputs at deployment, and is thus relevant to generative model evaluation under distribution shift. 

\medskip

For each sample size $n \in \{100,200,300\}$, we draw $N_{\text{iters}} = 75$ independent pairs of samples $(X_1,\dots,X_n) \sim \mathbb{P}$ and $(Y_1,\dots,Y_n) \sim \mathbb{Q}$. Each run produces a $p$-value computed from $200$ permutations, and the reported rejection rate is the proportion of the 75 repetitions in which the null hypothesis was rejected. To assess type I error control, we repeat the same procedure with both samples drawn from the pooled data, verifying that the empirical rejection rates are close to the nominal level. Samples from both distributions are shown in Figure \ref{fig:mnist}, along with the average rejection rates for the three sample sizes, under the alternative (drawing samples from MNIST and fake MNIST datasets, respectively) and null distribution  (drawing both samples from MNIST) respectively.

\paragraph{CIFAR-10 vs. CIFAR-10.1.} 

The CIFAR-10 dataset \cite{krizhevsky2009learning} consists of 32x32 colour labeled images in 10 classes (airplane, automobile, bird, cat, deer, dog, frog, horse, ship, truck) and has served as benchmark for various modern classification pipelines. The
CIFAR-10.1 \cite{recht2019imagenet} is a new test set for the CIFAR-10 classification tasks constructed to better approximate an independent sample from the underlying distribution. Samples from both
distributions are shown in Figure~\ref{fig:cifar}.

Let $\bbP$ denote the standard CIFAR-10 test set and $\bbQ$ the CIFAR-10.1 sample. Despite being drawn from nominally the same distribution, in many cases the classification model performance consistently drops on $\bbQ$, suggesting the presence of covariate shift. This was suggested by \cite{recht2019imagenet}, and later assessed by \cite{liu2020learning}. This makes the pair $(\bbP, \bbQ)$ a natural testbed for evaluating two-sample tests in domain adaptation contexts. We test for equality $\bbP = \bbQ$ and report results in Figure\ref{fig:cifar}.

\medskip

For each sample size $n \in \{250, 500, 750, 1000\}$, we draw $N_{\text{iters}} = 75$ independent pairs of samples $(X_1,\dots,X_n) \sim \mathbb{P}$ and $(Y_1,\dots,Y_n) \sim \mathbb{Q}$. Each run produces a p-value computed from $B=200$ permutations, and the reported rejection rate is the proportion of the 75 repetitions in which the null hypothesis was rejected. To assess type I error control, we repeat the same procedure with both samples drawn from the pooled data, verifying that the empirical rejection rates are close to the nominal level. Samples from both distributions are shown in Figure \ref{fig:cifar}, along with the average rejection rates for the three sample sizes, under the alternative (drawing samples from CIFAR-10 and CIFAR-10.1 datasets, respectively) and null (drawing both samples from the CIFAR-10).

\begin{figure}[]

    \vspace{1cm}

    \centering
    \begin{subfigure}[b]{0.45\textwidth}
        \centering
        \includegraphics[width=1\textwidth]{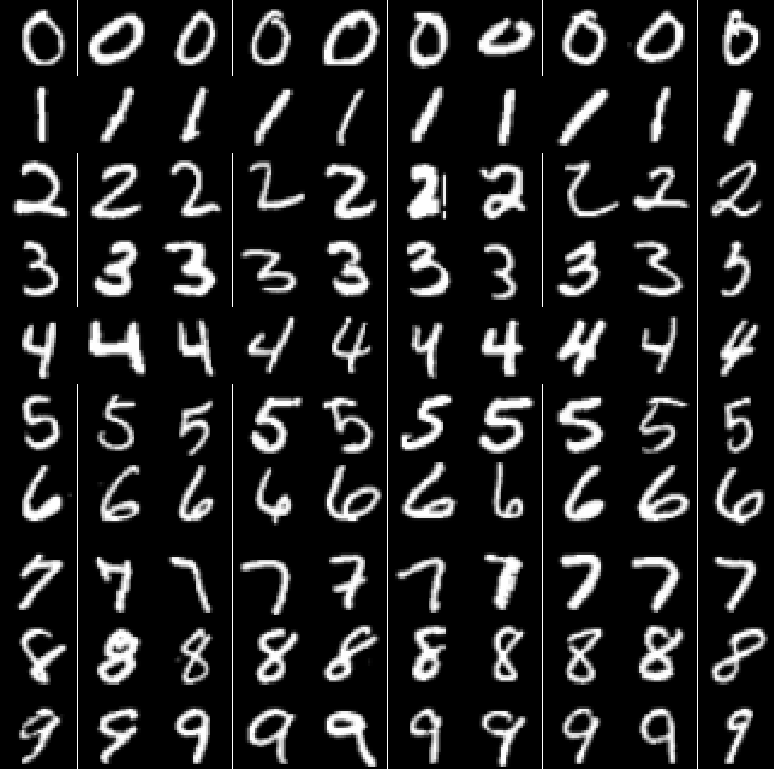}
        \caption{“Real” MNIST dataset.}
    \end{subfigure}
    \hfill
    \begin{subfigure}[b]{0.45\textwidth}
        \centering
        \includegraphics[width=1\textwidth]{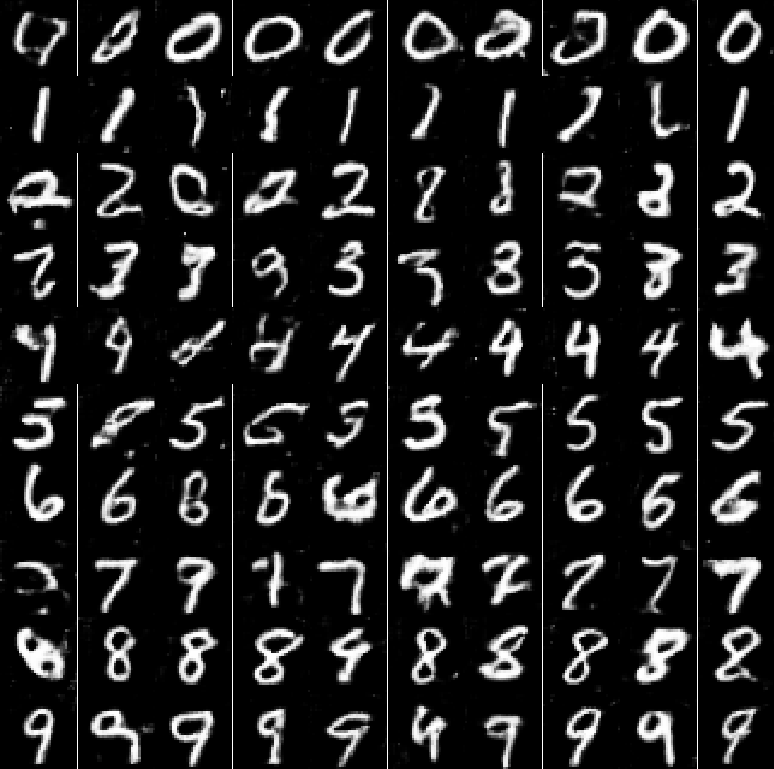} 
        \caption{``Fake” MNIST dataset.}
    \end{subfigure}

    \vspace{1em}

    \begin{subfigure}[b]{\textwidth}
        \centering
        \includegraphics[width=.75\textwidth]{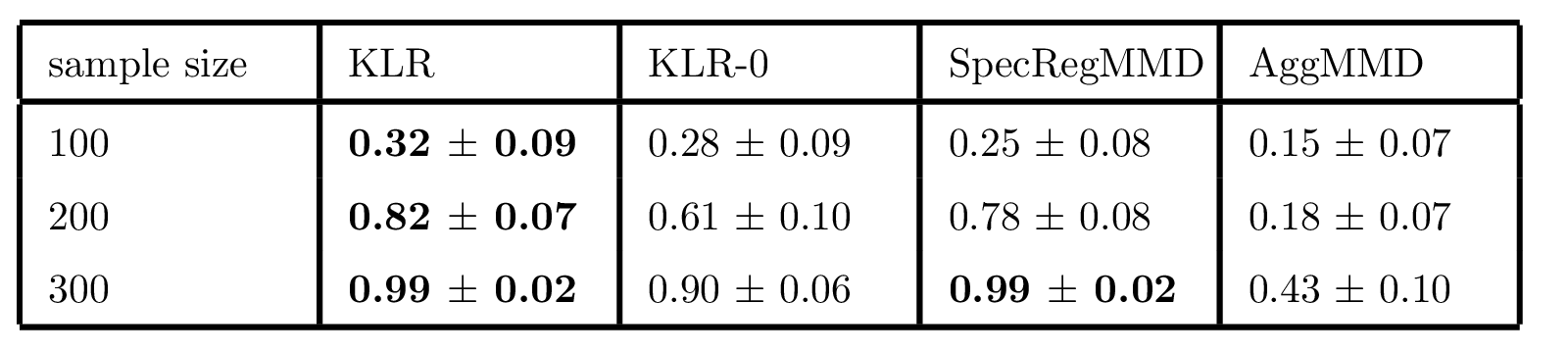}
        \caption{Rejection rates of tests for increasing sample sizes.}
    \end{subfigure}

    \vspace{1em}

    \begin{subfigure}[b]{\textwidth}
        \centering
        \includegraphics[width=.75\textwidth]{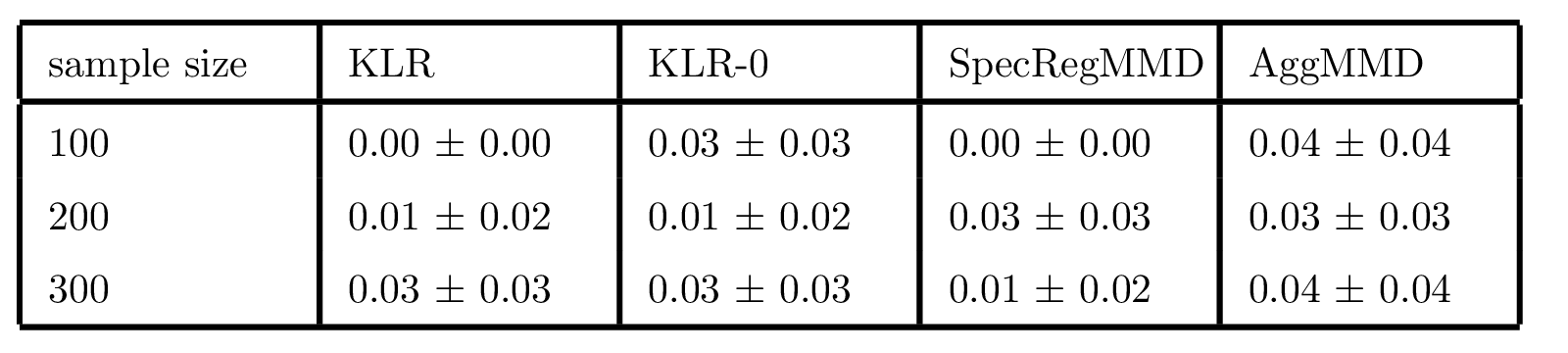}
        \caption{Average false rejections with increasing sample size, when drawing both samples from the null.}
    \end{subfigure}
    
    \caption{Top: samples from the real and “fake” MNIST datasets of handwritten digits. Middle: empirical rejection rates when testing for distributional equality between the two datasets. “Fake”-MNIST is generated by a DCGAN.
    Bottom: rejection rates under the null hypothesis, drawing both samples from the real MNIST dataset.}
    \label{fig:mnist}
\end{figure}

\begin{figure}[]
    \centering
    \begin{subfigure}[b]{0.45\textwidth}
        \centering
        \includegraphics[width=1\textwidth]{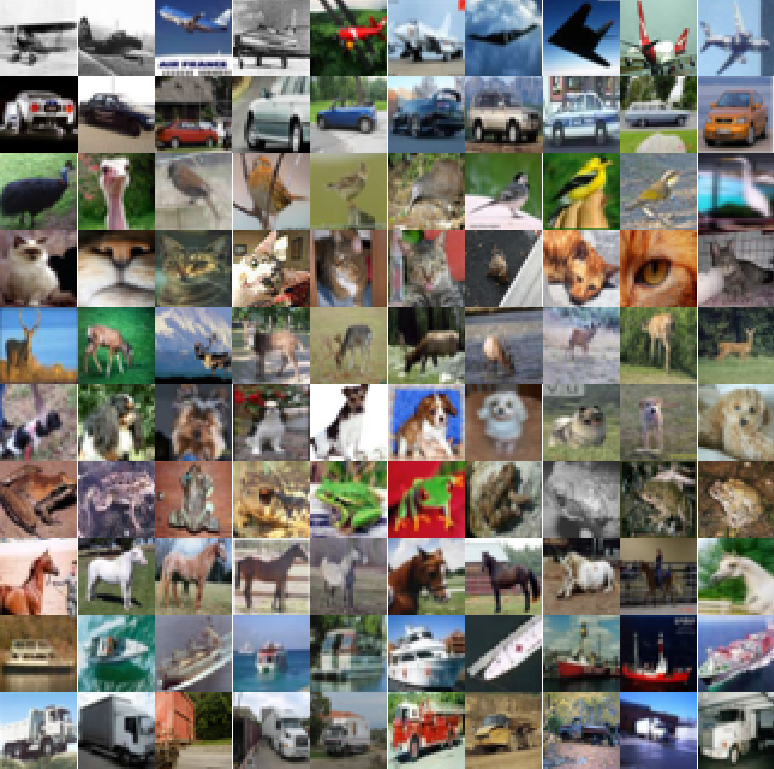}
        \caption{ CIFAR-10 dataset.}
    \end{subfigure}
    \hfill
    \begin{subfigure}[b]{0.45\textwidth}
        \centering
        \includegraphics[width=1\textwidth]{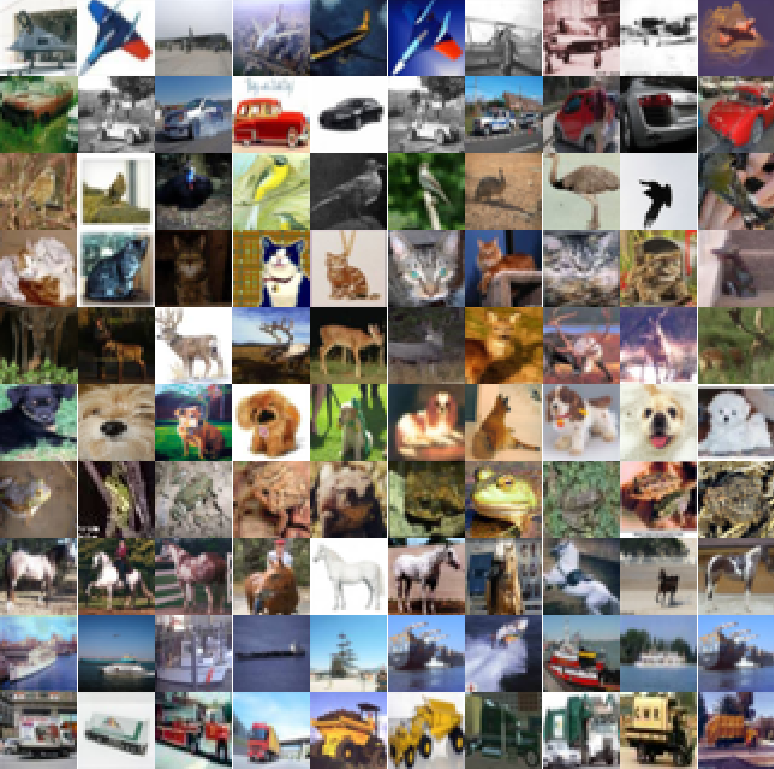} 
        \caption{ CIFAR-10.1 dataset.}
    \end{subfigure}

    \vspace{1em}

    \begin{subfigure}[b]{\textwidth}
        \centering
        \includegraphics[width=.75\textwidth]{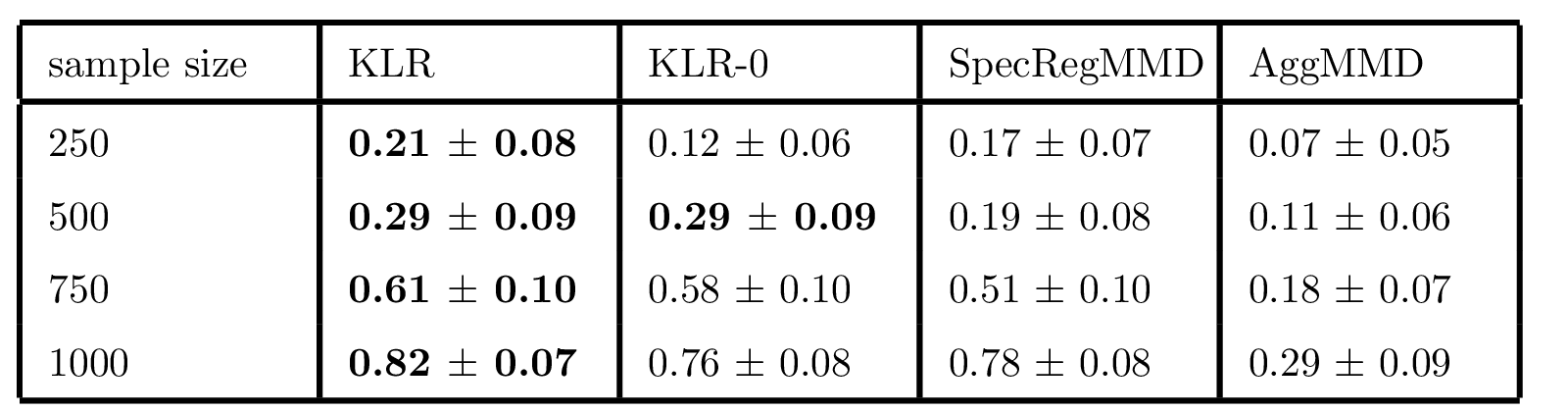}
        \caption{Rejection rates of tests for increasing sample sizes.}
    \end{subfigure}

    \vspace{1em}

    \begin{subfigure}[b]{\textwidth}
        \centering
        \includegraphics[width=.75\textwidth]{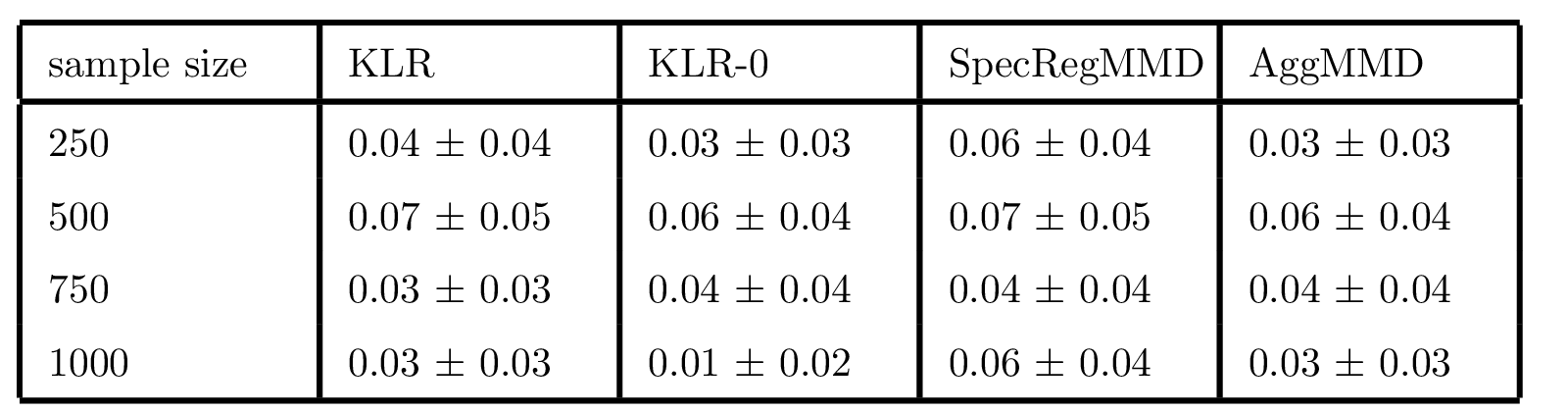}
        \caption{Average level values of tests for increasing sample sizes, drawing both samples from the null.}
    \end{subfigure}

    \caption{Top: samples from the  CIFAR-10 \cite{krizhevsky2009learning} and  CIFAR-10.1  \cite{recht2019imagenet} dataset. Middle:  empirical rejection rates when testing for distributional equality between the two datasets. Bottom: empirical level rates under the null, drawing both samples from the CIFAR-10 dataset.}
    \label{fig:cifar}
\end{figure}

\newpage

\appendix

\section{Proofs}
 
\subsection*{Proofs of main results}
\begin{proof}[Proof of Theorem~\ref{thm:main}]
    The statement for $\bbP=\bbQ$ is clear by \cite[][Theorem 3]{minh2021regularized}.
    It remains to show the claimed divergence for $\bbP\neq \bbQ$.
    Note that, if $f\in\range(\S_{\bbP}^{\sfrac{1}{2}})$, then:
    $
    \|(\S_\bbP + \gamma\Id)^{-\sfrac{1}{2}}f\|\to \|\S_{\bbP}^{-\sfrac{1}{2}}f\|
    $. To see this, let $\{e_j\}_{j\geq 1}$ be a CONS for $\cH$ of eigenfunctions of $\S_{\bbP}$, with corresponding (real, positive) eigenvalues $\{\alpha_j\}_{j\geq 1}$. If $f\in\range(\S_{\bbP}^{\sfrac{1}{2}})$, then $\exists~ g = \sum_{j\geq 1} g_j e_j$ such that $f = \S_{\bbP}^{\sfrac{1}{2}}g = \sum_{j\geq 1}g_j\alpha_j^{\sfrac{1}{2}} e_j$, with $\sum g_j^2 <\infty$. Hence, we have:
    $$
    \|(\S_\bbP + \gamma\Id)^{-\sfrac{1}{2}}f\|^2
     = \sum_{j\geq 1} g_j^2\frac{\alpha_j}{\alpha_j + \gamma}
    < \sum_{j\geq 1} g_j^2 < \infty 
    $$
    and the convergence follows by the dominated convergence theorem. The divergence in \eqref{eq:ker_regularized_KL}, hence, does not occur because of the Mahalonobis distance term, but can only be caused by the $\log\dettwo$ term. 
    That is, to prove the claim, we need to show that:
    $$
    - \log\dettwo \left( \Id + \bH_{\gamma} \right)\to \infty,\qquad \text{as } \quad \gamma\to 0
    $$
    where $\bH_{\gamma}= (\S_{\bbQ} + \gamma \Id)^{-\sfrac{1}{2}}(\S_{\bbP} - \S_{\bbQ})(\S_{\bbQ} + \gamma \Id)^{-\sfrac{1}{2}}$.
    We argue by contradiction. Assume, to the contrary, that it the quantity remains bounded as $\gamma\to 0$. 
    By the coercivity of the Carleman-Fredholm determinant (see the Proof of Theorem 3 in \cite{waghmare2023functionalgraphicallasso}), if $\|\bH_{\gamma}\|_{2} \to \infty$, then $|\log\dettwo \left( \Id + \bH_{\gamma} \right)| \to \infty$. Therefore, $\bH_{\gamma}$ must be bounded in Hilbert-Schmidt norm as $\gamma \to 0$. 
    However, $\bH_{\gamma}= (\S_{\bbQ} + \gamma \Id)^{-\sfrac{1}{2}}(\S_{\bbP} - \S_{\bbQ})(\S_{\bbQ} + \gamma \Id)^{-\sfrac{1}{2}}$ is bounded as $\gamma \to 0$ precisely when $\S_{\bbP}-\S_{\bbQ}$ is of the form $\smash{\S_{\bbQ}^{\sfrac{1}{2}}\bH_{0}\S_{\bbQ}^{\sfrac{1}{2}}}$ or equivalently, $\smash{\S_{\bbP} = \S_{\bbQ}^{\sfrac{1}{2}}(\bI + \bH_{0})\S_{\bbQ}^{\sfrac{1}{2}}}$ for some Hilbert-Schmidt $\bH_{0}$. This follows by Proposition 2.4 of \cite{hanke2017} when applied to the linear maps $\bH \mapsto (\S_{\bbQ} + \gamma \Id)^{-\sfrac{1}{2}}\bH(\S_{\bbQ} + \gamma \Id)^{-\sfrac{1}{2}}$ mapping Hilbert-Schmidt operators to Hilbert-Schmidt operators. Furthermore, $\bI + \bH_{0} \succeq \bzero$ since $\bI + \bH_{\gamma} = (\S_{\bbQ} + \gamma \Id)^{-\sfrac{1}{2}}(\S_{\bbP} + \gamma \bI)(\S_{\bbQ} + \gamma \Id)^{-\sfrac{1}{2}} \succeq \bzero$. It follows that $\cN(\bzero, \S_{\bbP})$ and $\cN(\bzero, \S_{\bbQ})$ are equivalent, implying $\bbP = \bbQ$  which contradicts our original assertion.
         
\end{proof}

\begin{proof}[Proof of Proposition~\ref{prop:prob_bound:KL}]
    To alleviate notation, we will employ the notation $\bA_{\gamma} := (\bA+\gamma\Id)$ for a general operator $\bA$ and positive $\gamma>0$. Define the operators:
    $$
    \bH_{\gamma} : \S_{\bbP,\gamma}^{-\sfrac{1}{2}}(\S_{\bbQ} - \S_{\bbP})\S_{\bbP,\gamma} ^{-\sfrac{1}{2}}
    \qquad \text{and}\qquad
    \hat\bH_{\gamma}:= \S_{\bbP_n,\gamma}^{-\sfrac{1}{2}}(\S_{\bbQ_m}  - \S_{\bbP_n})\S_{\bbP_n,\gamma}^{-\sfrac{1}{2}}.
    $$
    
    With this notation, we bound the difference:
    \begin{align*}
         \test^{\textnormal{KL}}_{\gamma}(\bbP_n,\bbQ_m) -  \test^{\textnormal{KL}}_{\gamma}(\bbP,\bbQ) &  \leq  \frac{1}{2}\left\lVert \S_{\bbP,\gamma}^{-\sfrac{1}{2}} (\m_{\bbP} - \m_{\bbQ}) - \S_{\bbP_n,\gamma} ^{-\sfrac{1}{2}} (\m_{\bbP_n} - \m_{\bbQ_m}) \right\rVert _{\cH}^2 
           \\ &\quad  + \frac{1}{2}\left\lvert \log\dettwo(\Id + \bH_\gamma) - \log\dettwo(\Id + \hat\bH_\gamma) \right\rvert
            \end{align*}

    For first term on the right-hand side, employing Lemma~\ref{lem:mahl:seq}, we have the upper bound:
    \begin{align*}
        &\frac{1}{2}\left\lVert \S_{\bbP,\gamma}^{-\sfrac{1}{2}} (\m_{\bbP} - \m_{\bbQ}) - \S_{\bbP_n,\gamma} ^{-\sfrac{1}{2}} (\m_{\bbP_n} - \m_{\bbQ_m}) \right\rVert _{\cH}^2 
        \\ 
         &\qquad \leq \frac{1}{2}
          \left( \gamma^{-1/2} \left\lVert (\m_{\bbP} +  \m_{\bbQ}) -(\m_{\bbP_n} + \m_{\bbQ_m} )\right\rVert _{\cH}
          + {\sqrt{K}}\gamma^{-\sfrac{1}{2}} \left\lVert\S_{\bbP} - \S_{\bbP_n}\right\rVert _{\textnormal{Tr}}^{\sfrac{1}{2}} \right)^2      
          \\
        &\qquad \leq
          2\gamma^{-1} \left\lVert \m_{\bbP} - \m_{\bbP_n} \right\rVert_{\cH}^2 +2\gamma^{-1}\left\lVert \m_{\bbQ} - \m_{\bbQ_m} \right\rVert_{\cH}^2 
          + K\gamma^{-1} \left\lVert\S_{\bbP} - \S_{\bbP_n}\right\rVert _{\textnormal{Tr}}.
    \end{align*}

    Then, moving to the second term on the right-hand side, observe that $\bH_{\gamma}, \hat\bH_{\gamma} \succ -\Id$, and are trace-class for any strictly positive regularisation parameter $\gamma>0$. Indeed:
    \begin{align*}
        \langle \bH_{\gamma}\bf, \bf\rangle_{\cH}
            &= \langle  \S_{\bbP,\gamma}^{-\sfrac{1}{2}}(\S_{\bbQ} - \S_{\bbP})\S_{\bbP,\gamma} ^{-\sfrac{1}{2}}\bf,\bf\rangle_{\cH}
            \\& = \langle  (\S_{\bbP,\gamma}^{-\sfrac{1}{2}}\S_{\bbQ,\gamma}\S_{\bbP,\gamma} ^{-\sfrac{1}{2}} - \Id)\bf,\bf\rangle_{\cH}
            \\& = \| \S_{\bbQ,\gamma}^{\sfrac{1}{2}}\S_{\bbP,\gamma}^{-\sfrac{1}{2}}\bf\|_{\cH} - \|\bf\|_{\cH} \geq -\|\bf\|_{\cH}
    \end{align*}
    and 
        $\| \bH_{\gamma}\|_{\textnormal{Tr}} \leq \gamma^{-2}(\| \S_{\bbP} - \S_{\bbQ}\|_{\textnormal{Tr}}$,
    and similarly for $\hat\bH_{\gamma}$.
    Hence, by \ref{lgb_continuity} in Lemma~\ref{lem:logdet_bounds} we can bound:
    \begin{align*}
        \log\dettwo(\Id + \bH_\gamma) - \log\dettwo(\Id + \hat\bH_\gamma) 
        & \leq \|(\Id + \bH_\gamma)^{-1}\bH_\gamma^2 - (\Id + \hat\bH_\gamma)^{-1}\hat \bH_\gamma^2\|_{\textnormal{Tr}}
        \\& \leq \left\lVert\S_{\bbP,\gamma}^{\sfrac{1}{2}}\S_{\bbQ,\gamma}^{-1}\S_{\bbP,\gamma}^{\sfrac{1}{2}} - \S_{\bbP_n,\gamma}^{\sfrac{1}{2}}\S_{\bbP_m,\gamma}^{-1}\S_{\bbP_n,\gamma}^{\sfrac{1}{2}}\right\rVert_{\textnormal{Tr}}
        \\ &\leq T_1 + T_2 + T_3
   \end{align*}
    where we have plugged-in the definitions of $\bH_\gamma, \hat\bH_\gamma$ and simplified the expression to obtain the last term, and introduced the notation:
$$
    T_1 =  \left\lVert\left(\S_{\bbP,\gamma}^{\sfrac{1}{2}} - \S_{\bbP_n,\gamma}^{\sfrac{1}{2}}\right) \S_{\bbQ,\gamma}^{-1}\S_{\bbP,\gamma}^{\sfrac{1}{2}}\right\rVert_{\textnormal{Tr}},
    \qquad\qquad 
    T_2 = \left\lVert\S_{\bbP_n,\gamma}^{\sfrac{1}{2}} \S_{\bbQ,\gamma}^{-1}\left(\S_{\bbP,\gamma}^{\sfrac{1}{2}} - \S_{\bbP_n,\gamma}^{\sfrac{1}
    {2}}\right)\right\rVert_{\textnormal{Tr}},
$$
$$
     T_3 = \left\lVert\S_{\bbP_n,\gamma}^{\sfrac{1}{2}} \left(\S_{\bbQ,\gamma}^{-1} - \S_{\bbQ_m,\gamma}^{-1}\right)\S_{\bbP_n,\gamma}^{\sfrac{1}{2}}\right\rVert_{\textnormal{Tr}}.
$$
For the first term, we see that:
\begin{align*}
T_1 
    &= 
   \left\lVert\left(\S_{\bbP,\gamma}^{\sfrac{1}{2}} - \S_{\bbP_n,\gamma}^{\sfrac{1}{2}}\right) \S_{\bbQ,\gamma}^{-1}\S_{\bbP,\gamma}^{\sfrac{1}{2}}\right\rVert_{\textnormal{Tr}} 
\\  & \leq 
   \left\lVert\S_{\bbP,\gamma}^{\sfrac{1}{2}} - \S_{\bbP_n,\gamma}^{\sfrac{1}{2}}\right\rVert_{\textnormal{HS}} 
    \cdot \left\lVert\S_{\bbQ,\gamma}^{-1}\S_{\bbP,\gamma}^{\sfrac{1}{2}}\right\rVert_{\textnormal{HS}}
\\  & \leq           \left\lVert\S_{\bbP,\gamma}^{\sfrac{1}{2}} - \S_{\bbP_n,\gamma}^{\sfrac{1}{2}}\right\rVert_{\textnormal{HS}} \cdot \left\lVert
          \S_{\bbQ,\gamma}^{-1}
          \right\rVert_{\textnormal{op}} \cdot
          \left\lVert
          \S_{\bbP,\gamma}^{\sfrac{1}{2}}
          \right\rVert_{\textnormal{HS}} 
\\  & \leq       \left\lVert\S_{\bbP,\gamma}^{\sfrac{1}{2}} - \S_{\bbP_n,\gamma}^{\sfrac{1}{2}}\right\rVert_{\textnormal{HS}} \cdot \left\lVert
      \S_{\bbQ,\gamma}^{-1}
      \right\rVert_{\textnormal{op}} \cdot
      \left\lVert
      \S_{\bbP,\gamma}
      \right\rVert_{\textnormal{Tr}} ^{\sfrac{1}{2}}\qquad
\\  & \leq       \sqrt{2K}\gamma^{-3/2}  \left\lVert\S_{\bbP} - \S_{\bbP_n}\right\rVert_{\textnormal{HS}}    
\end{align*}
For the second term, mutatis-mutandis we see that:
\begin{align*}
T_2 = \left\lVert\S_{\bbP_n,\gamma}^{\sfrac{1}{2}} \S_{\bbQ,\gamma}^{-1}\left(\S_{\bbP,\gamma}^{\sfrac{1}{2}} - \S_{\bbP_n,\gamma}^{\sfrac{1}
    {2}}\right)\right\rVert_{\textnormal{Tr}}
 \leq       \sqrt{2K}\gamma^{-3/2}  \left\lVert\S_{\bbP} - \S_{\bbP_n}\right\rVert_{\textnormal{HS}}.
\end{align*}
Similarly, for the third term we have the bound:
\begin{align*}
     T_3 
     &= \left\lVert\S_{\bbP_n,\gamma}^{\sfrac{1}{2}} \left(\S_{\bbQ,\gamma}^{-1} - \S_{\bbQ_m,\gamma}^{-1}\right)\S_{\bbP_n,\gamma}^{\sfrac{1}{2}}\right\rVert_{\textnormal{Tr}} 
     \leq 
     2K\gamma^{-2}\left\lVert\S_{\bbQ} - \S_{\bbQ_m}\right\rVert_{\textnormal{HS}}.
\end{align*}

\noindent Putting things together gives:

    \begin{align*}
    \left\lvert\test^{\textnormal{KL}}_{\gamma}(\bbP_n,\bbQ_m) -  \test^{\textnormal{KL}}_{\gamma}(\bbP,\bbQ) \right\rvert 
       &\leq
       4\gamma^{-2}K \cdot \Big(
         {\left\lVert\S_\bbQ  - \S_{\bbQ_m}\right\rVert _{\textnormal{Tr}}}
       + {\left\lVert \S_{\bbP_n} -  \S_\bbP\right\rVert _{\textnormal{Tr}}}
       \Big)
       +
       2\gamma^{-1}K \cdot \Big(
         {\left\lVert\m_\bbQ  - \m_{\bbQ_m}\right\rVert^2}
       + {\left\lVert \m_{\bbP_n} -  \m_\bbP\right\rVert^2 }
       \Big)
       \\
       &\leq 6\gamma^{-2}K \Big(
         \underbrace{\left\lVert\S_\bbQ  - \S_{\bbQ_m}\right\rVert _{\textnormal{Tr}}}_{\Delta_{Y}}
       + \underbrace{\left\lVert \S_{\bbP_n} -  \S_\bbP\right\rVert _{\textnormal{Tr}}}_{\Delta_{X}}
       \Big).
\end{align*}

It is then easy to see that changing either of $X_i$ or $Y_i$ in $\Delta := {\Delta_{X} + \Delta_{Y}}$ results in changes in magnitude of at most $K/\min(n,m)$.
Hence, McDiarmid's theorem yields that:
\begin{align*}
    P\left( \Delta - \bbE[\Delta] > \varepsilon \right) \leq  2\exp\left( - \frac{2\varepsilon^2}{K^2(n^{-1} + m^{-1})}\right)
    = 2\exp\left( - \frac{2\varepsilon^2 n m }{K^2(n + m)}\right)
\end{align*}

To conclude, we need to bound the decay of $\bbE[\Delta]$ with $n,m$. This requires a bit of caution, since it requires us to deal with Banach-valued random variables \cite{ledoux2013probability}. First, by symmetrisation:
\begin{align*}
    \bbE\left[ \left\lVert\S_\bbQ  - \S_{\bbQ_m}\right\rVert _{\textnormal{Tr}}\right]
    & = \bbE_X \left\lVert \frac{1}{n}\sum_{i = 1}^n k_{X_i}\otimes k_{X_i} - \bbE_{\tilde X}[k_{\tilde X}\otimes k_{\tilde X}] \right\rVert _{\textnormal{Tr}}
    \\
    \textnormal{(Jensen)}\quad 
    & \leq \bbE_{X,\tilde X} \left\lVert \frac{1}{n}\sum_{i = 1}^n (k_{X_i}\otimes k_{X_i} - k_{\tilde X_i}\otimes k_{\tilde X_i}) \right\rVert _{\textnormal{Tr}}
    \\
    \textnormal{(symmetrisation)}\quad 
    & \leq 2\bbE_{X,\varepsilon} \left\lVert \frac{1}{n}\sum_{i = 1}^n \varepsilon_i k_{X_i}\otimes k_{X_i} \right\rVert _{\textnormal{Tr}}
\end{align*}
Now, observing that $\|A\|_{\textnormal{Tr}} = \sup_{C \textnormal{ compact}, \|C\|_{\textnormal{op}} \leq 1} |\trace(A C)|$, we can write the trace norm as:
\begin{align*}
   &2\bbE_{X,\varepsilon} \sup_{\substack{C \textnormal{ compact},\\ \|C\|_{\textnormal{op}}\leq 1} } \left\lvert\trace\left( \frac{1}{n}\sum_{i = 1}^n \varepsilon_i C( k_{X_i}\otimes k_{X_i}) \right)\right\rvert
    = 2\bbE_{X,\varepsilon} \sup_{\substack{C \textnormal{ compact},\\ \|C\|_{\textnormal{op}}\leq 1}} \left\lvert \sum_{i = 1}^n  \varepsilon_i \trace\left( \frac{C( k_{X_i}\otimes k_{X_i}) }{n} \right)\right\rvert.
\end{align*}
Since the unit ball is compact in the weak topology, we can employ Khintchine's inequality to upper bound the last displayed equation by:
\begin{align*}
    &2\bbE_{X} \sup_{\substack{C \textnormal{ compact},\\ \|C\|_{\textnormal{op}}\leq 1}} \left( \sum_{i = 1}^n  \trace\left( \frac{C( k_{X_i}\otimes k_{X_i}) }{n} \right)^2\right)^{1/2}
    \\
    &= 
    2\bbE_{X} \sup_{\substack{C \textnormal{ compact},\\ \|C\|_{\textnormal{op}}\leq 1}} \frac{1}{\sqrt{n}}\left( \sum_{i = 1}^n  \frac{\trace\left( C (k_{X_i}\otimes k_{X_i})\right)^2 }{n}\right)^{1/2}
   \leq  
    2K^{1/2}\frac{1}{\sqrt{n}}.
\end{align*}
establishing that:
$$
\bbE[  \Delta] = \bbE\left[ \left\lVert\S_\bbQ  - \S_{\bbQ_m}\right\rVert _{\textnormal{Tr}} + \left\lVert \S_{\bbP_n} -  \S_\bbP\right\rVert _{\textnormal{Tr}} \right] \leq 2K^{1/2}\left( \frac{1}{\sqrt{n}} + \frac{1}{\sqrt{m}}\right).
$$
For the first order term, it is standard that $\bbE[\Delta_1]\leq 2 K/n$. In particular:
$$
\bbE[  \Delta_1] = \bbE\left[ \left\lVert\m_\bbQ  - \m_{\bbQ_m}\right\rVert + \left\lVert \m_{\bbP_n} -  \m_\bbP\right\rVert \right] \leq K\left( \frac{1}{n} + \frac{1}{m}\right).
$$
Hence, we obtain that:
\begin{align*}
    &P\left( \left\lvert\test^{\textnormal{KL}}_{\gamma}(\bbP_n,\bbQ_m) -  \test^{\textnormal{KL}}_{\gamma}(\bbP,\bbQ) \right\rvert  > 12\gamma^{-2}K^{3/2}\left( \frac{1}{\sqrt{n}} + \frac{1}{\sqrt{m}}\right) + \varepsilon  \right)\\
    &\leq 
    P\left(6K\gamma^{-2}(\Delta - \bbE[\Delta])> 12\gamma^{-2}K^{3/2}\left( \frac{1}{\sqrt{n}} + \frac{1}{\sqrt{m}}\right) + \varepsilon  \right)
    \\
    &\leq 
    P\left(\Delta - \bbE[\Delta] > 2K^{1/2}\left( \frac{1}{\sqrt{n}} + \frac{1}{\sqrt{m}}\right) - \bbE[\Delta] + \frac{\gamma^2}{6K}\varepsilon  \right)
    \\
    &\leq 
    P\left(\Delta - \bbE[\Delta] >  \frac{\gamma^2}{6K}\varepsilon   \right)
    \\
    & \leq  \exp\left( - \frac{\gamma^4\varepsilon^2 n m }{18K^4(n + m)}\right)
\end{align*}
which completes the proof.
\end{proof}

\begin{proof}[Proof of Corollary~\ref{cor:01law}]
 Under $H_0$, the population statistic is identically null by \eqref{eq:ker_regularized_KL}, i.e., $\test^{\textnormal{KL}}_{\gamma}(\bbP, \bbQ) = 0$ for all $\gamma$. Hence, (1) in Proposition~\ref{prop:prob_bound:KL}:
\begin{align*}
    \lim_{n,m \to \infty} P\left(\lvert \test^{\textnormal{KL}}_{\gamma_{nm}}(\bbP_n, \bbQ_m)\rvert  > 
    \varepsilon  \right) 
    & \leq \lim_{n,m \to \infty} P\left(\lvert \hat\test^{\textnormal{KL}}_{\gamma_{nm}}(\bbP_n, \bbQ_m)\rvert  > 
    \frac{\sqrt{2}K(\gamma_{nm} + 4K)}{\gamma_{nm}^2}\left(\frac{1}{\sqrt{n}} + \frac{1}{\sqrt{m}}\right)
    + \varepsilon/2 \right) 
    \\& \leq 2 \cdot \lim_{n,m \to \infty}\exp\left( -\frac{\gamma_{nm}^4 \min\{n,m\}}{72K^4}\varepsilon^2 \right) = 0 .
\end{align*}
where we have used that  $\gamma_{nm}^4\min\{n,m\}\to\infty$ as $(n,m)\to\infty$.

\medskip

Similarly, under $H_1$ we have that the population statistic diverges as regularization decays.
In particular, we have that
$$
P(|\test^{\textnormal{KL}}_{\gamma_{nm}}(\bbP_n,\bbQ_m) - \test^{\textnormal{KL}}_{\gamma_{nm}}(\bbP,\bbQ)| > \varepsilon) \to 0 \quad \text{as } n,m \to \infty,
$$
by (1) in Proposition~\ref{prop:prob_bound:KL}, provided $\gamma_{nm}^4\min\{n,m\}\to\infty$. In particular, for any $\delta>0$ and sufficiently large  $m,n$,
$$
P(|\test^{\textnormal{KL}}_{\gamma_{nm}}(\bbP_n,\bbQ_m) - \test^{\textnormal{KL}}_{\gamma_{nm}}(\bbP,\bbQ)| > \varepsilon/2) < \delta.
$$
Now, note that:

$$
P(|\test^{\textnormal{KL}}_{\gamma_{nm}}(\bbP_n,\bbQ_m)| > \varepsilon) \geq P(|\test^{\textnormal{KL}}_{\gamma_{nm}}(\bbP,\bbQ)| > \varepsilon/2) - P(|\test^{\textnormal{KL}}_{\gamma_{nm}}(\bbP_n,\bbQ_m) - \test^{\textnormal{KL}}_{\gamma_{nm}}(\bbP,\bbQ)| > \varepsilon/2),
$$
but the first term on the right hand side is deterministic and diverging as regularization decays.
Hence, 
$$
P(|\test^{\textnormal{KL}}_{\gamma_{nm}}(\bbP_n,\bbQ_m)| > \varepsilon) \geq  1 - \delta
$$
for any small $\delta>0$, provided $n,m$ are sufficiently large.
Since $\delta$ was arbitrary, this proves the claim.
\end{proof}

\begin{lemma}\label{lem:boundonthresh}
    Let $\{X_j\}_{j\geq 1},\{Y_j\}_{j\geq 1}$ be i.i.d.\ sequences drawn from $\bbP, \bbQ \in \cP(\cX)$, respectively. Take $\gamma_n \geq C^2n^{\sfrac{(\beta-1)}{4}}$ for some $C>0$. 
    Define:
    \begin{equation}\label{eq:u}
    u_{\alpha, n} := n^{-\sfrac{\beta}{2}} 6K\left(  {4K^{1/2}} +  KC\sqrt{\log \alpha^{-1} } \right).
\end{equation}
Then:
    \begin{equation*}\label{eq:boundonthresh}
        P_{H_0}\left( \left\lvert\test^{\textnormal{KL}}_{\gamma}(\bbP_n,\bbQ_n) \right\rvert  > u_{\alpha, n}\right) \leq \alpha
    \end{equation*} 
    for any $\alpha\in(0,1)$.
\end{lemma}

\begin{proof}[Proof of Lemma~\ref{lem:boundonthresh}]
Note that, for $\alpha>0$:
\begin{align*}
\exp\left( - \frac{\gamma^4 n }{36K^4}\varepsilon^2\right) \leq \alpha  
    &\quad \Longleftrightarrow \quad
    -  \frac{\gamma^4 n }{36K^4}\varepsilon^2 \leq \log\alpha 
\\
    &\quad \Longleftrightarrow \quad
       \frac{\gamma^4 n }{36K^4}\varepsilon^2<\log\frac{1}{\alpha} 
\\
    &\quad \Longleftrightarrow \quad
       \varepsilon \leq \sqrt{36K^4 \gamma^{-4} n^{-1} \log\frac{1}{\alpha} }
\end{align*}
In particular, if $\gamma_n \geq  C^2 n^{\frac{\beta-1}{4}}$ for some $C>0$, this holds for:
$$
\varepsilon \leq 6 C K^2 n^{-\sfrac{\beta}{2}}\sqrt{\log\frac{1}{\alpha} }
$$
By Proposition~\ref{prop:prob_bound:KL}, taking $\gamma_n \geq C^2 n^{\frac{\beta-1}{4}}$, balanced samples $n=m$, and setting $\varepsilon = 6 CK^2 n^{-\sfrac{\beta}{2}}\sqrt{\log\frac{1}{\alpha} }$ we have under the null hypothesis $H_0$ that:
$$
        P_{H_0}\left( \left\lvert\test^{\textnormal{KL}}_{\gamma_n}(\bbP_n,\bbQ_n) \right\rvert  > n^{-\sfrac{\beta}{2}} 6K\left(  {4K^{1/2}} +  KC\sqrt{\log\frac{1}{\alpha} } \right) \right)
       \leq  \alpha.
$$
where we have used that, under the null, $\test_\gamma(\bbP,\bbQ) = 0$ for all $\gamma>0.$
\end{proof}

\begin{proof}[Proof of Theorem~\ref{thm:consistent}]
For simplicity, we consider the case where $n=m$. The proof can be easily adapted to the case where $n,m$ differ, both diverge to infinity but with a finite limiting ratio. Let us first observe that, given observed samples $\{X_1,\dots,X_n\}$ and $\{Y_1,\dots,Y_n\}$ drawn from $\bbP,\bbQ$ respectively, then we have that:
$$
\frac{1}{2n}\left( \sum_{j=1}^n\delta_{X_j} + \sum_{j=1}^n\delta_{Y_j}\right) \to \frac{\bbP +\bbQ}{2}, \quad \textnormal{if}\quad n\to \infty,
$$
where the convergence holds weakly almost surely (i.e., in the topology of weak convergence of probability measures), 
In particular, by the continuous mapping theorem, if we let $Z_1,\dots,Z_{2n}$ denote i.i.d.\ random draws from the pooled sample $\{X_1, \dots , X_n, Y_1,\dots, Y_n\}$ then we have that, for any fixed $\gamma>0$:
$$
\test^{\textnormal{KL}}_{\gamma}\left(\frac{1}{n}\sum_{j=1}^n\delta_{Z_j}, \frac{1}{m}\sum_{j=1}^m\delta_{Z_{n+j}}\right)
\to 
\test^{\textnormal{KL}}_{\gamma}\left(\frac{\bbP}{2} +  \frac{\bbQ}{2}, \frac{\bbP}{2} +  \frac{\bbQ}{2}\right) = 0, \quad n\to \infty, 
$$
almost surely, and therefore, by  Corollary~\ref{cor:01law}, we have that for any $\varepsilon>0$:
$$
P\left( 
\test^{\textnormal{KL}}_{\gamma_n}\left(\frac{1}{n}\sum_{j=1}^n\delta_{Z_j}, \frac{1}{m}\sum_{j=1}^m\delta_{Z_{n+j}}\right) > \varepsilon 
\right)\to 0,
$$
provided $\gamma_n^4 n\to\infty$ as $n\to\infty$.This shows that for any $\varepsilon>0$, the $1-\alpha$ quantile $\tau^*_{\gamma_n,n}<\varepsilon$ as $n\to\infty.$ So the proof will be complete if we can show that the observed value $\test^{\textnormal{KL}}_{\gamma_{n}}(\bbP_n,\bbQ_m)$ exceeds $\varepsilon$ with probability converging to $1$, as $n$ diverges. But this is precisely the statement of Corollary~\ref{cor:01law}.

\end{proof}

\begin{lemma}\label{lem:lowerbound}
    Let $\bbP,\bbQ \in \cP(\cX)$. Then,     for all $\gamma>0$:
    \begin{enumerate}
        \item 
        $
        \test_\gamma(\bbP,\bbQ) \geq  \frac{1}{4K^2}\|\S_{\bbQ} - \S_{\bbP}\|^2_{\textnormal{HS}}
        $

        \item       $
        \test_\gamma(\bbP,\bbQ) \leq \gamma^{-2}\|\S_{\bbQ} - \S_{\bbP}\|^2_{\textnormal{HS}} \leq \gamma^{-2}
        $
    \end{enumerate}

\end{lemma}

\begin{proof}[Proof of Lemma~\ref{lem:lowerbound}]
    By \ref{lgb_lower} in Lemma~\ref{lem:logdet_bounds}, we have that:
\begin{align*}
    \test_{\gamma}(\bbP,\bbQ) 
        & \geq \frac{1}{2}\trace\left(\left(\Id + \S_{\bbP, \gamma}^{-\sfrac{1}{2}}(\S_{\bbQ} - \S_{\bbP})\S_{\bbP, \gamma}^{-\sfrac{1}{2}}\right)^{-1}\left(\S_{\bbP, \gamma}^{-\sfrac{1}{2}}(\S_{\bbQ} - \S_{\bbP})\S_{\bbP, \gamma}^{-\sfrac{1}{2}}\right)^2\right)
    \\
        & = \frac{1}{2}\trace\left(\S_{\bbP, \gamma}^{\sfrac{1}{2}}\S_{\hat\bbQ}^{-1}\S_{\bbP, \gamma}^{\sfrac{1}{2}}\left(\S_{\bbP, \gamma}^{-\sfrac{1}{2}}(\S_{\hat\bbQ} - \S_{\bbP})\S_{\bbP, \gamma}^{-\sfrac{1}{2}}\right)^2\right)
    \\
        & = \frac{1}{2}\trace\left(\S_{\bbQ}^{-1}(\S_{\bbQ} - \S_{\bbP})\S_{\bbP, \gamma}^{-1}(\S_{\bbQ} - \S_{\bbP})\right)
    \\
        & \geq \frac{1}{4K^2}\trace\left((\S_{\bbQ_n} - \S_{\bbP})^2\right)
    \\
        & = \frac{1}{4K^2}\|\S_{\bbQ} - \S_{\bbP}\|^2_{\textnormal{HS}}
\end{align*}
where we have repeatedly used the cyclic property of the trace, and that $\trace(AB) \geq \lambda_{\textnormal{min}}(A)\trace(B)$ for non-negative operators $A,B\succeq 0$.
The upper bound is proven similarly.
\end{proof}

\begin{proof}[Proof of Theorem~\ref{thm:sepa_bound}]
Note that we have $k^2(\cdot,\cdot)  = \Psi(\cdot - \cdot)^2$, so that by \cite[][Corollary 4]{sriperumbudur2010hilbert}
$$
\operatorname{MMD}_{k^2}(\bbP,\bbQ) = \left\lVert\phi(\bbP) - \phi(\bbQ)\right\rVert_{L^2(\cX,\Lambda)}
$$
where $\Lambda = \cF(\Psi^2)$ denotes the Fourier transform of $\Psi^2$, and $\phi(\bbP),\phi(\bbQ)$ the characteristic functions of $\bbP,\bbQ\in\cP(\cX)$, respectively. Furthermore, it is straightforward to see (by Lemma~\ref{lem:HSnorm2Emb_diff}) that:
$$
\|\S_\bbP - \S_\bbQ\|_{\textnormal{HS}} = \operatorname{MMD}_{k^2}(\bbP,\bbQ)
$$
so that if $\bbP,\bbQ\in \cP_{\Delta}^{\Lambda}$, then $\|\S_\bbP - \S_\bbQ\|_{\textnormal{HS}} \geq \Delta$. It follows that, if 
$
\left\| \S_\bbP - \S_\bbQ \right\|_{\textnormal{HS}} > \Delta,
$
then for any \( \varepsilon > 0 \),
\begin{equation}\label{eq:bound:diffPQ}
{P} \left( \left\| \S_{\bbQ_m} - \S_{\bbP_n} \right\|_{\textnormal{HS}} < \Delta^2/3 - \varepsilon \right)
\leq 2 \exp\left( -\frac{m\varepsilon^2}{8K^2} \right) + 2 \exp\left( -\frac{n\varepsilon^2}{8K^2} \right).
\end{equation}
Indeed, by the triangle inequality
\begin{align*}
\left\| \S_{\bbQ} - \S_{\bbP} \right\|_{\textnormal{HS}}^2 
&= 
\left\| (\S_{\bbQ} - \S_{\bbQ_m}) + (\S_{\bbQ_m} - \S_{\bbP_n}) + (\S_{\bbP_n}-  \S_{\bbP}) \right\|_{\textnormal{HS}}^2 
\\
&\leq\left(
\left\| \S_{\bbQ} - \S_{\bbQ_m})\right\|_{\textnormal{HS}}  + \left\|\S_{\bbQ_m} - \S_{\bbP_n}\right\|_{\textnormal{HS}} + \left\|\S_{\bbP_n}-  \S_{\bbP}) \right\|_{\textnormal{HS}}\right)^2 
\\
&\leq 
3\left\| \S_{\bbQ_m} - \S_{\bbP_n} \right\|_{\textnormal{HS}}^2
+ \left\| \S_{\bbQ_m} - \S_\bbQ \right\|_{\textnormal{HS}}^2
+ \left\| \S_{\bbP_n} - \S_\bbP \right\|_{\textnormal{HS}}^2
\end{align*}
where we have used that $(a+b+c)^2\leq 3a^2+3b^2+3c^2$.
In particular, we have that
\[
\left\| \S_{\bbQ_m} - \S_{\bbP_n} \right\|_{\textnormal{HS}}^2
\geq \frac{1}{3}\left\| \S_\bbQ - \S_\bbP \right\|_{\textnormal{HS}}^2
- \left\| \S_{\bbQ_m} - \S_\bbQ \right\|_{\textnormal{HS}}^2
- \left\| \S_{\bbP_n} - \S_\bbP \right\|_{\textnormal{HS}}^2.
\]
So if \( \left\| \S_{\bbQ_m} - \S_{\bbP_n} \right\|_{\textnormal{HS}}^2 < \Delta^2/3 - \varepsilon \) while $\left\| \S_\bbP - \S_\bbQ \right\|_{\textnormal{HS}}^2 > \Delta^2$, then
$
\left\| \S_{\bbQ_m} - \S_\bbQ \right\|_{\textnormal{HS}}^2
+ \left\| \S_{\bbP_n} - \S_\bbP \right\|_{\textnormal{HS}}^2 > \varepsilon.
$
Using a union bound and standard concentration of Hilbert-valued random variables we arrive at
\[ 
{P}\left( \left\| \S_{\bbP_n} - \S_\bbP \right\|_{\textnormal{HS}} > {\varepsilon} \right)
\leq 2 \exp\left( -\frac{n\varepsilon^2}{2K^2} \right),
\]
and the same holds when swapping $\mathbb{Q}$ for \( \mathbb{P} \).  Together with Lemma~\ref{lem:lowerbound}, this gives:
\begin{align*}
P_{H_1}\left(\test_{\gamma}(\bbP_n,\bbQ_m)  < \frac{\Delta^2}{12K^2} - \varepsilon \right) 
    &   \leq 
 P_{H_1}\left( \|\S_{\bbQ_m} - \S_{\bbP_n}\|^2_{\textnormal{HS}}  < \frac{\Delta^2}{3} - 4K^2\varepsilon \right) 
    \\&  \leq 
     P_{H_1}\left( \left\| \S_{\bbQ_m} - \S_\bbQ \right\|_{\textnormal{HS}}^2
+ \left\| \S_{\bbP_n} - \S_\bbP \right\|_{\textnormal{HS}}^2 > 4K^2\varepsilon\right)
    \\&  \leq 
     P_{H_1}\left( \left\| \S_{\bbQ_m} - \S_\bbQ \right\|_{\textnormal{HS}}^2
 > 2K\sqrt{\varepsilon/2}\right)
 +
  P_{H_1}\left( \left\| \S_{\bbP_n} - \S_\bbP \right\|_{\textnormal{HS}}^2
 > 2K\sqrt{\varepsilon/2}\right)
    \\&  \leq 
 2 \exp\left( -{m\varepsilon^2} \right) + 2 \exp\left( -{n\varepsilon^2} \right).
\end{align*}
In particular, for balanced sample sizes:
$$
P_{H_1}\left(\test_{\gamma}(\bbP_n,\bbQ_n)  < \frac{\Delta^2}{12K^2} - \varepsilon \right) 
\leq
 4 \exp\left( -{n\varepsilon^2} \right).
$$
which is bounded above by $\delta$ whenever $\varepsilon\geq n^{-\sfrac{1}{2}}\sqrt{\log 4/\delta}$.

Next, note that for $u(\alpha,n)$ as in \eqref{eq:u}, if we have that
\begin{equation}\label{eq:keyseparation}
n^{-\sfrac{\beta}{2}} 6K\left(  {4K^{1/2}} +  KC\sqrt{\log\frac{1}{\alpha} } \right)  \leq \frac{\Delta^2}{12K^2} -n^{-\sfrac{1}{2}}\sqrt{\log 4/\delta}
\end{equation}
then, by monotonicity we have that, for $\gamma_n = Cn^{\sfrac{(\beta - 1)}{4}}$
$$
P_{H_1}\left(\test_{\gamma_n}(\bbP_n,\bbQ_n)  < u_{\alpha, n} \right) 
 \leq 
 P_{H_1}\left(\test_{\gamma_n}(\bbP_n,\bbQ_n)  < \frac{\Delta^2}{12K^2} - n^{-\sfrac{1}{2}}\sqrt{\log 4/\delta} \right) \leq \delta,
$$
and we observe that \eqref{eq:keyseparation} holds whenever:
$$
n \geq  \max \left\lbrace\left( \frac{6K\left(4K^{1/2} + KC\sqrt{\log\frac{1}{\alpha}} \right)}{\frac{\Delta^2}{12K^2} - \sqrt{\log 4/\delta}} \right)^{2/\beta}, \:
 \frac{144K^4}{\Delta^4}\log (4/\delta) \right\rbrace
$$

\end{proof}

\subsection*{Auxiliary results}
\begin{lemma}\label{lem:HSnorm2Emb_diff}
 For $\bbP\in\cP(\cX)$, the covariance embedding $\S_\bbP$ is an integral operator on $\cH$ with kernel $k(\cdot,\cdot)$ and measure $\bbP$. Furthermore, for $\bbP,\bbQ\in\cP(\cX)$. Then:
        \begin{align}
        &\|\S_{\bbP} - \S_{\bbQ}\|_{\text{op}} =
        \sup_{\|f\|_{\cH} \leq 1} 
        \left\lbrace \int_{\cX}f^2(\bu)~d\bbP(\bu) -  \int_{\cX}f^2(\bu)~d\bbQ(\bu) \right\rbrace
        \label{eqn:OPnorm2Emb_diff}\\
        \text{and} \qquad &\|\S_{\bbP} - \S_{\bbQ}\|_{\textnormal{HS}} =
         \int_{\cX} k^2(\bu, \bu') ~d\bbP(\bu)d\bbP(\bu') + \int_{\cX} k^2(\bv,\bv') d\bbQ(\bv)d\bbQ(\bv') - 2\int_{\cX} k^2(\bu, \bv) ~d\bbP(\bu)d\bbQ(\bv)
        .\label{eqn:HSnorm2Emb_diff}
        \end{align}   
    
\end{lemma}

\begin{proof}[Proof of Lemma~\ref{lem:HSnorm2Emb_diff}]
First note that, for $f\in\cH$:
\begin{align*}
    \S_\bbP f(\cdot) 
            & = \int_{\cX} \langle k_{\bu}, f\rangle k_{\bu}(\cdot)\, ~d\bbP(\bu) = \int_{\cX}  f(\bu) k_{\bu}(\cdot)\, ~d\bbP(\bu)
    = \int_{\cX} f(\bu) k(\bu,\cdot)\, ~d\bbP(\bu).
    \end{align*}
    Then, observe that 
    $$
    \langle \S_{\bbP} f, f \rangle_{\cH} = 
    \int_{\cX}\langle  (k_{\bu}\otimes k_{\bu}) f, f \rangle_{\cH} ~d\bbP(\bu)= 
    \int_{\cX} \langle k_{\bu}, f\rangle^2 ~d\bbP(\bu) =  
    \int_{\cX} f^2(\bu) ~d\bbP(\bu),
    $$
    and that hence we can write:
        \begin{align*}
    \|\S_{\bbP} - \S_{\bbQ}\|_{\text{op}}
            =   \sup_{\|f\|_{\cH} \leq 1} \langle (\S_{\bbP} - \S_{\bbQ}) f, f \rangle_{\cH}
            =    \sup_{\|f\|_{\cH} \leq 1} 
        \left\lbrace \int_{\cX}f^2(\bu)~d\bbP(\bu) -  \int_{\cX}f^2(\bu)~d\bbQ(\bu) \right\rbrace
        \end{align*}
        proving \eqref{eqn:OPnorm2Emb_diff}. To show \eqref{eqn:HSnorm2Emb_diff}, observe that, given  some orthonormal basis $ \{ e_i \} $ of $ \cH $, we may write:
    \begin{align*}
    \|\S_{\bbP} - \S_{\bbQ}\|_{\textnormal{HS}}^2 
            &   =    \sum_{i\geq 1} \langle (\S_{\bbP} - \S_{\bbQ}) e_i, (\S_{\bbP} - \S_{\bbQ}) e_i \rangle_{\cH}
        \\  &   =    \sum_{i\geq 1} \langle \left(\int (k_\bu\otimes k_\bu) d(\bbP-\bbQ)(\bu)\right) e_i, \left(\int (k_\bv\otimes k_\bv) d(\bbP-\bbQ)(\bv)\right) e_i \rangle_{\cH}
        \\  &   =    \sum_{i\geq 1} \int\int \langle(k_\bu\otimes k_\bu) e_i,   (k_\bv\otimes k_\bv) e_i \rangle_{\cH} ~d^2(\bbP-\bbQ)(\bu,\bv) 
        \\  &   =    \sum_{i\geq 1} \int\int \langle k_\bu, e_i\rangle_{\cH}  \langle k_\bu, k_\bv \rangle_{\cH}\langle k_\bv, e_i\rangle_{\cH}   ~d^2(\bbP-\bbQ)(\bu,\bv) 
        \\  &   =    \int\int \langle k_\bu, k_\bv \rangle_{\cH} \sum_{i\geq 1}\left(  \langle k_\bu, e_i\rangle_{\cH}  \langle k_\bv, e_i\rangle_{\cH} \right)  ~d^2(\bbP-\bbQ)(\bu,\bv) 
        \\  &   =    \int\int k^2(\bu,\bv) d^2(\bbP-\bbQ)(\bu,\bv) 
    \end{align*}
    where we have used bilinearity of the inner product, dominated convergence, Parseval's identity and the reproducing property.
\end{proof}

\begin{lemma}\label{lem:easybound}
    Let $\gamma > 0$ and non-negative, Hilbert-Schmidt operators $\bA,\bB\succeq 0 $.
    \begin{enumerate}
        \item\label{eb:op} $
        \|(\gamma\bI + \bA)^{-\sfrac{1}{2}}\|_{\textnormal{op}}  \leq  \gamma^{-\sfrac{1}{2}}.
        $

        \item\label{eb:-1}
        $\left\lVert [\gamma \Id + \bA]^{-1} - [\gamma \Id + \bB]^{-1} \right\rVert _{\textnormal{HS}} \leq  \gamma^{-2}\left\lVert \bA - \bB \right\rVert _{\textnormal{HS}}$.

         \item\label{eb:1/2} $\left\lVert [\gamma \Id + \bA]^{\sfrac{1}{2}} \:-\: [\gamma \Id + \bB]^{\sfrac{1}{2}} \:\right\rVert _{\textnormal{HS}} \leq  \frac{\sqrt{2}}{2}\gamma^{-\sfrac{1}{2}}\left\lVert \bA - \bB \right\rVert _{\textnormal{HS}}$.

        \item\label{eb:-1/2:HS} $\left\lVert [\gamma \Id + \bA]^{-\sfrac{1}{2}} - [\gamma \Id + \bB]^{-\sfrac{1}{2}} \right\rVert _{\textnormal{HS}} \leq  \gamma^{-\sfrac{3}{2}}\left\lVert \bA - \bB \right\rVert _{\textnormal{HS}}$.
    \end{enumerate}
    If $\bA,\bB$ are additionally trace class:
    \begin{enumerate}\setcounter{enumi}{3}
        \item\label{eb:-1/2:tr} $\left\lVert [\gamma \Id + \bA]^{-\sfrac{1}{2}} - [\gamma \Id + \bB]^{-\sfrac{1}{2}} \right\rVert _{\textnormal{HS}} \leq \gamma^{-1} \left\lVert \bA - \bB\right\rVert _{\textnormal{Tr}}^{\sfrac{1}{2}}$.
    \end{enumerate}
\end{lemma}

\begin{proof}[Proof of Lemma~\ref{lem:easybound}]
    Let us write $\bA_{\gamma}:= \bA + \gamma\Id$ and $\bB_{\gamma}:= \bB+\gamma\Id$.
        If $\{\lambda_i\}_{i\geq 1}$ comprise an eigenvalue sequence for $\bA$, then necessarily $\lambda_i\to 0$ as $i\to \infty$ and, since they are all non-negative:
        $$
        \|\bA_{\gamma}^{-\sfrac{1}{2}}\|_{\textnormal{op}} = \sup_{i\geq 1} \frac{1}{(\gamma + \lambda_i)^{1/2}} \leq \gamma^{-1/2}.
        $$
    proving \ref{eb:op}. Next, let us write:
$$
\bA_{\gamma}^{-1} - \bB_{\gamma}^{-1} = \bA_{\gamma}^{-1}(\bB -\bA)\bB_{\gamma}^{-1}
$$
so that taking norms, and applying \ref{eb:op} twice gives:
$$
\|\bA_{\gamma}^{-1} - \bB_{\gamma}^{-1}\|_{\textnormal{HS}} \leq \gamma^{-2}\|\bA - \bB\|_{\textnormal{HS}}
$$
proving \ref{eb:-1}.

\medskip
To show \ref{eb:1/2} observe that $\bH_\gamma ,\bB_\gamma \Id \succeq \gamma \Id$ and since the square root is operator monotone, we obtain that $(\bH_\gamma  + \bB_\gamma)^{\sfrac{1}{2}} \succeq \sqrt{2 \gamma}\Id$. Therefore, by \cite[][Proposition 3.2]{van1980inequality}:
$$
\left\lVert \bA_{\gamma}^{\sfrac{1}{2}} - \bB_{\gamma}^{\sfrac{1}{2}}\right\rVert _{\textnormal{HS}} 
\leq 
\frac{\sqrt{2}}{2}\gamma^{-\sfrac{1}{2}}\left\lVert \bA - \bB\right\rVert _{\textnormal{HS}}
$$
proving \ref{eb:1/2}. Then, observe that we may write:
    \begin{align*}
        \bA_{\gamma}^{-\sfrac{1}{2}} - \bB_{\gamma}^{-\sfrac{1}{2}}
            &= \bA_{\gamma}^{-\sfrac{1}{2}}\left(\bB_{\gamma}^{\sfrac{1}{2}} - \bA_{\gamma}^{\sfrac{1}{2}}\right) \bB_{\gamma}^{-\sfrac{1}{2}}
    \end{align*}
    and that by taking norms, and using \ref{eb:op} and \ref{eb:1/2}, we arrive at
    \begin{align*}
            \left\lVert \bA_{\gamma}^{-\sfrac{1}{2}} - \bB_{\gamma}^{-\sfrac{1}{2}} \right\rVert _{\textnormal{HS}}        
            \leq 
            \gamma^{-1}  \left\lVert \bA_{\gamma}^{\sfrac{1}{2}} - \bB_{\gamma}^{\sfrac{1}{2}}\right\rVert _{\textnormal{HS}} 
            \leq
            \frac{\sqrt{2}}{2}\gamma^{-\sfrac{3}{2}}  \left\lVert \bA - \bB\right\rVert _{\textnormal{HS}},
    \end{align*}
proving \ref{eb:-1/2:HS}. Finally, to prove \ref{eb:-1/2:tr}, we simply employ  the Powers–Størmer's inequality \cite{powers1970free} to change the Hilbert-Schmidt norm to a trace norm:
    $$
    \left\lVert \bA_{\gamma}^{\sfrac{1}{2}} - \bB_{\gamma}^{\sfrac{1}{2}}\right\rVert _{\textnormal{HS}}  
    \leq 
 \left\lVert \bA - \bB\right\rVert^{\sfrac{1}{2}} _{\textnormal{Tr}} 
    $$
completing the proof.
\end{proof}

\begin{lemma}\label{lem:mahl:seq}
    Let $\{\bA_n\}_{n\geq 1}, \bA$ be Hilbert-Schmidt operators on $\cH$. Let $\{m_n\}_{n\geq 1}, m \in \cH$. Then:
    $$
    \left\lVert (\bA_n+\gamma \Id)^{-\sfrac{1}{2}}\bm_n 
    \right\rVert_{\cH} - \left\lVert (\bA+\gamma \Id)^{-\sfrac{1}{2}}\bm 
    \right\rVert_{\cH}
    \leq
    \gamma^{-\sfrac{1}{2}}\left(\|\bm - \bm_n\|_{\cH} + \|\bm\|_{\cH}\|\bA-\bA_n\|^{\sfrac{1}{2}}_{\textnormal{Tr}}\right).
    $$
\end{lemma}
\begin{proof}
Using the notation $\bA_{\gamma} = \bA+\gamma\Id$ and $\bA_{n, \gamma} = \bA_n+\gamma\Id$
    \begin{align*}
        \left\lVert \bA_{\gamma, n}^{-\sfrac{1}{2}}\bm_n 
    \right\rVert_{\cH} - \left\lVert \bA_{\gamma}^{-\sfrac{1}{2}}\bm 
    \right\rVert_{\cH}
        & \leq   \left\lVert \bA_{\gamma, n}^{-\sfrac{1}{2}}\bm_n -
     \bA_{\gamma}^{-\sfrac{1}{2}}\bm 
    \right\rVert_{\cH}
    \\  & = \left\lVert \bA_{\gamma, n}^{-\sfrac{1}{2}}(\bm_n - \bm) +
     (\bA_{\gamma, n}^{-\sfrac{1}{2}} - \bA_{\gamma}^{-\sfrac{1}{2}})\bm 
    \right\rVert_{\cH}
    \\  & \leq \left\lVert \bA_{\gamma, n}^{-\sfrac{1}{2}}(\bm_n - \bm) +
     (\bA_{\gamma, n}^{-\sfrac{1}{2}} - \bA_{\gamma}^{-\sfrac{1}{2}})\bm 
    \right\rVert_{\cH}
    \\ \textnormal{(By Lemma~\ref{lem:easybound}})\quad &\leq
    \gamma^{-\sfrac{1}{2}}\|\bm_n - \bm\|_{\cH} + \gamma^{-\sfrac{1}{2}}\|\bm\|_{\cH}\cdot \|\bA_n-\bA\|^{\sfrac{1}{2}}_{\textnormal{Tr}} 
    \end{align*}
completing the proof.
\end{proof}

\begin{lemma}
The functional $\mathbf{A} \mapsto \log \dettwo(\mathbf{I}+\mathbf{A})$ is twice differentiable in the Gâteux sense, with the first and second Gâteux derivatives at $\mathbf{A}$ given by $(\mathbf{I}+\mathbf{A})^{-1}-\mathbf{I}$ and $[(\mathbf{I}+\mathbf{A}) \otimes(\mathbf{I}+\mathbf{A})]^{-1}$, respectively.
\end{lemma}
For completeness, we reproduce the proof from \cite{waghmare2023functionalgraphicallasso}, which relies on some basic results in convex analysis in Hilbert spaces.
\begin{proof}
 We simply evaluate the derivative of $f$ at $t=0$ by looking at its Taylor expansion. For
$\mathbf{A}, \mathbf{B}$ trace-class
$$
\begin{aligned}
f(t)-f(0) & =\log \dettwo(\mathbf{I}+\mathbf{A}+t \mathbf{B})-\log \dettwo(\mathbf{I}+\mathbf{A}) \\
& =\log \operatorname{det}(\mathbf{I}+\mathbf{A}+t \mathbf{B})-\log \operatorname{det}(\mathbf{I}+\mathbf{A})-\operatorname{tr}(\mathbf{A}+t \mathbf{B})+\operatorname{tr} \mathbf{A} \\
& \left.=\log \operatorname{det}\left[\mathbf{I}+t(\mathbf{I}+\mathbf{A})^{-1} \mathbf{B}\right)\right]-t \operatorname{tr}(\mathbf{B}) \\
& =\left[t \operatorname{tr}\left[(\mathbf{I}+\mathbf{A})^{-1} \mathbf{B}\right]-\frac{1}{2} t^2 \operatorname{tr}\left\{\left[(\mathbf{I}+\mathbf{A})^{-1} \mathbf{B}\right]^2\right\}+o\left(t^3\right)\right]-t \operatorname{tr}(\mathbf{B}) \\
& =t \operatorname{tr}\left\{\left[(\mathbf{I}+\mathbf{A})^{-1}-\mathbf{I}\right] \mathbf{B}\right\}-\frac{1}{2!} t^2 \operatorname{tr}\left\{\left[(\mathbf{I}+\mathbf{A})^{-1} \mathbf{B}\right]^2\right\}+o\left(t^3\right)
\end{aligned}
$$
The result follows from the continuity of expressions in $\|\cdot\|_2$ norm and the fact that trace-class operators are dense in the space of Hilbert-Schmidt operators.
\end{proof}

\begin{lemma}\label{lem:logdet_bounds}
    For $\gamma>0$ and $\bH$ Hilbert-Schmidt with $\bH\succ -\Id$, the following bound hold:
    \begin{enumerate}
        \item \label{lgb_upper}  $-\log\dettwo(\Id + \bH) \leq \trace((\Id+\bH)^{-1}\bH^2).$
        \item\label{lgb_lower} $-\log\dettwo(\Id + \bH) \geq \frac{1}{2}\trace((\Id+\bH)^{-1}\bH^2) $
    \end{enumerate}
    Furthermore, for $\bH,\widetilde\bH$ Hilbert-Schmidt with $\bH, \widetilde\bH\succ -\Id$:

    \begin{enumerate}\setcounter{enumi}{2}
        \item\label{lgb_continuitySpectral}
        $\left|\log \dettwo(\Id+\bH)-\log \dettwo(\Id+\widetilde\bH)\right| 
        \leq \sum_{i\geq 1}\left\lvert \frac{\lambda_i^2}{1 +\lambda_i} - \frac{\widetilde\lambda_i^2}{1 +\widetilde\lambda_i}\right\rvert, 
        $ 
    \end{enumerate}
     where $\{\lambda_i\}_{i\geq 1}$ and $\{\widetilde\lambda_i\}_{i\geq 1}$ denote the eigenvalues of $\bH,\widetilde\bH$, respectively. In particular:
    \begin{enumerate}\setcounter{enumi}{3}
        \item\label{lgb_continuity}
        $\left|\log \dettwo(\Id+\bH)-\log \dettwo(\Id+\widetilde\bH)\right| 
        \leq \left\lVert(\Id +\bH)^{-1}\bH^2 
        - (\Id +\widetilde\bH)^{-1}\widetilde\bH^2\right\rVert_{\textnormal{Tr}},$
    \end{enumerate}
\end{lemma}
\begin{proof}
    Note that,  for all $x>-1, x - \log(1 +x) \leq \frac{x^2}{1+x}$. Indeed, as in \cite[Lemma 68][]{minh2021regularized}, if we define $
g(x) = \frac{x^2}{1+x} - x + \log(1+x)$ for $x > -1$, then $g'(x) = \frac{x}{(1+x)^2} \geq 0$,
so \( g \) is decreasing on \( (-1,0) \), increasing on \( (0,\infty) \), and \( g(0) = 0 \). Hence $
x - \log(1+x) \leq \frac{x^2}{1+x} \quad \text{for all } x > -1.$
Therefore, if $\{\lambda_i\}_{i\geq 1}$ denote the eigenvalues of $\bH$, with $\lambda_i>-1$ for all $i\geq 1$ by assumption, we have that:
    \begin{align*}
    -\log\dettwo(\Id + \bH)
        &= \trace(\bH - \log(\Id+ \bH)) 
    \\  &= \sum_{i\geq 1} \left(\lambda_i - \log(1 + \lambda_i)\right)
     \leq \sum_{i\geq 1} \frac{\lambda_i^2}{1+ \lambda_i} \
    = \trace((\Id+\bH)^{-1}\bH^2).
    \end{align*}
    which proves \ref{lgb_upper}.

    \medskip

    Observe that for all $x>-1$, $x - \log(1+x) \geq \frac{x^2}{2(1+x)}$. Indeed:
    if we let
\(
f(x) = x - \log(1+x) - \frac{x^2}{2(1+x)}
\) for $x>1$, 
then
\(
f'(x) = \frac{x^2}{2(1+x)^2} \geq 0.
\)
Hence, \( f \) is non-decreasing and \( f(0) = 0 \), so
\(
x - \log(1+x) \geq \frac{x^2}{2(1+x)}\) \text{for all } \(x > -1\). This proves \ref{lgb_lower} by the same argument.

\medskip Similar arguments show that for all $x,y>-1$ we have:
$$
(x - \log(1+x)) - (y - \log(1+y) \leq \left\lvert
\frac{x^2}{1+x} - \frac{y^2}{1+y}\right\rvert
$$
Hence, if $\{\lambda_i\}_{i\geq 1}$ and $\{\widetilde\lambda_i\}_{i\geq 1}$ denote the eigenvalues of $\bH,\widetilde\bH$, respectively, we see that:
\begin{align*}
-\log \dettwo(\Id+\bH)+\log \dettwo(\Id+\widetilde\bH) 
&=
\sum_{i\geq 1} \left( (\lambda_i + \log(1+\lambda_i)) - (\widetilde\lambda_i + \log(1+\widetilde\lambda_i))\right) 
\\&\leq \sum_{i\geq 1}\left\lvert \frac{\lambda_i^2}{1 +\lambda_i} - \frac{\widetilde\lambda_i^2}{1 +\widetilde\lambda_i}\right\rvert.
\end{align*}
proving \ref{lgb_continuitySpectral}
Then, by \citet[][Theorem II]{kato1987variation}
(see also \cite[][Theorem 48]{quang2022kullback})) we obtain that:
$$
-\log \dettwo(\Id+\bH)+\log \dettwo(\Id+\widetilde\bH) \leq \| (\Id+\bH)^{-1}\bH - (\Id+\widetilde\bH)^{-1}\widetilde\bH^2\|_{\textnormal{Tr}}
$$
proving \ref{lgb_continuity}.
\end{proof}

\addcontentsline{toc}{section}{References}
\bibliographystyle{plainnat}
\bibliography{bib}

\end{document}